\documentclass{article} 

\usepackage[main, final, nonatbib]{neurips_2026}

\usepackage[utf8]{inputenc} 
\usepackage[T1]{fontenc}    
\usepackage{url}            
\usepackage{booktabs}       
\usepackage{amsfonts}       
\usepackage{nicefrac}       
\usepackage{microtype}      
\usepackage[dvipsnames]{xcolor}         

\usepackage{float}
\usepackage{amsmath, amsthm, amssymb}
\usepackage{xspace}
\usepackage{multirow}
\usepackage{graphicx}
\usepackage{svg}
\definecolor{LightCyan}{rgb}{0.9,0.9,1}
\definecolor{LightRed}{rgb}{0.85,0.27,0.34}
\definecolor{iclrblue}{rgb}{0.21,0.49,0.70}
\usepackage[colorlinks=true,linkcolor=iclrblue,citecolor=iclrblue,urlcolor=NavyBlue]{hyperref}  
\usepackage[capitalize,nameinlink]{cleveref}   

\usepackage[compact]{titlesec}
\titlespacing{\section}{0pt}{1ex}{0.15ex}
\titlespacing{\subsection}{0pt}{0.2ex}{0.13ex}
\titlespacing{\paragraph}{0pt}{0ex}{1em}
\usepackage{enumitem}

\setlist{
    itemsep=0.35ex,      
    topsep=0.15ex,       
    parsep=0pt,          
    partopsep=0pt        
}

\usepackage{pifont}
\usepackage{subcaption}
\usepackage[para]{threeparttable}
\usepackage[hang,flushmargin,symbol*]{footmisc}
\usepackage{etoc}   
\usepackage{wrapfig}
\usepackage{enumitem}
\usepackage{color, colortbl}
\usepackage[export]{adjustbox}
\usepackage{algorithm}
\usepackage{algorithmic}

\newtheorem{theorem}{Theorem}[section]
\newtheorem{proposition}[theorem]{Proposition}

\newtheorem{definition}[theorem]{Definition}
\newtheorem*{remark}{Remark}

\usepackage[square,numbers]{natbib}
\newcommand{\sys}{TopoPrune\xspace}

\title{TopoPrune: Robust Data Pruning via \\ Unified Latent Space Topology}
\author{
  Arjun Roy \\
  Purdue University \\
  \small \texttt{roy208@purdue.edu} \\
  \And
  Prajna G. Malettira \\
  Purdue University \\
  \small \texttt{pmaletti@purdue.edu} \\
  \And
  Manish Nagaraj \\
  Purdue University \\
  \small \texttt{mnagara@purdue.edu} \\
  \And
  Kaushik Roy \\
  Purdue University \\
  \small \texttt{kaushik@purdue.edu} \\
}

\begin{document}
\etocdepthtag.toc{mtoc}  

\maketitle
\label{Sec:Abstract}
\begin{abstract}

Geometric data pruning methods, while practical for leveraging pretrained models, are fundamentally unstable. Their reliance on extrinsic geometry renders them highly sensitive to latent space perturbations, causing performance to degrade during cross-architecture transfer or in the presence of feature noise. We introduce \sys, a framework which resolves this challenge by leveraging topology to capture the stable, intrinsic structure of data. \sys operates at two scales, (1) utilizing a \textit{topology-aware manifold approximation} to establish a global low-dimensional embedding of the dataset. Subsequently, (2) it employs \textit{differentiable persistent homology} to perform a local topological optimization on the manifold embeddings, ranking samples by their structural complexity. We demonstrate that our \textit{unified dual-scale topological approach} ensures high accuracy and precision, particularly at significant dataset pruning rates (e.g., 90\%). Furthermore, through the inherent stability properties of topology, \sys is (a) exceptionally robust to noise perturbations of latent feature embeddings and (b) demonstrates superior transferability across diverse network architectures. This study demonstrates a promising avenue towards stable and principled topology-based frameworks for robust data-efficient learning.\footnote{Code is available at \url{https://github.com/ar-roy/topoprune}}

\end{abstract}
\vspace{-\baselineskip}
\section{Introduction}
\label{Sec:Introduction}

\begin{wrapfigure}{r}{0.6\textwidth}
  \vspace{-3em}
  \begin{center}
    \includegraphics[width=\linewidth]{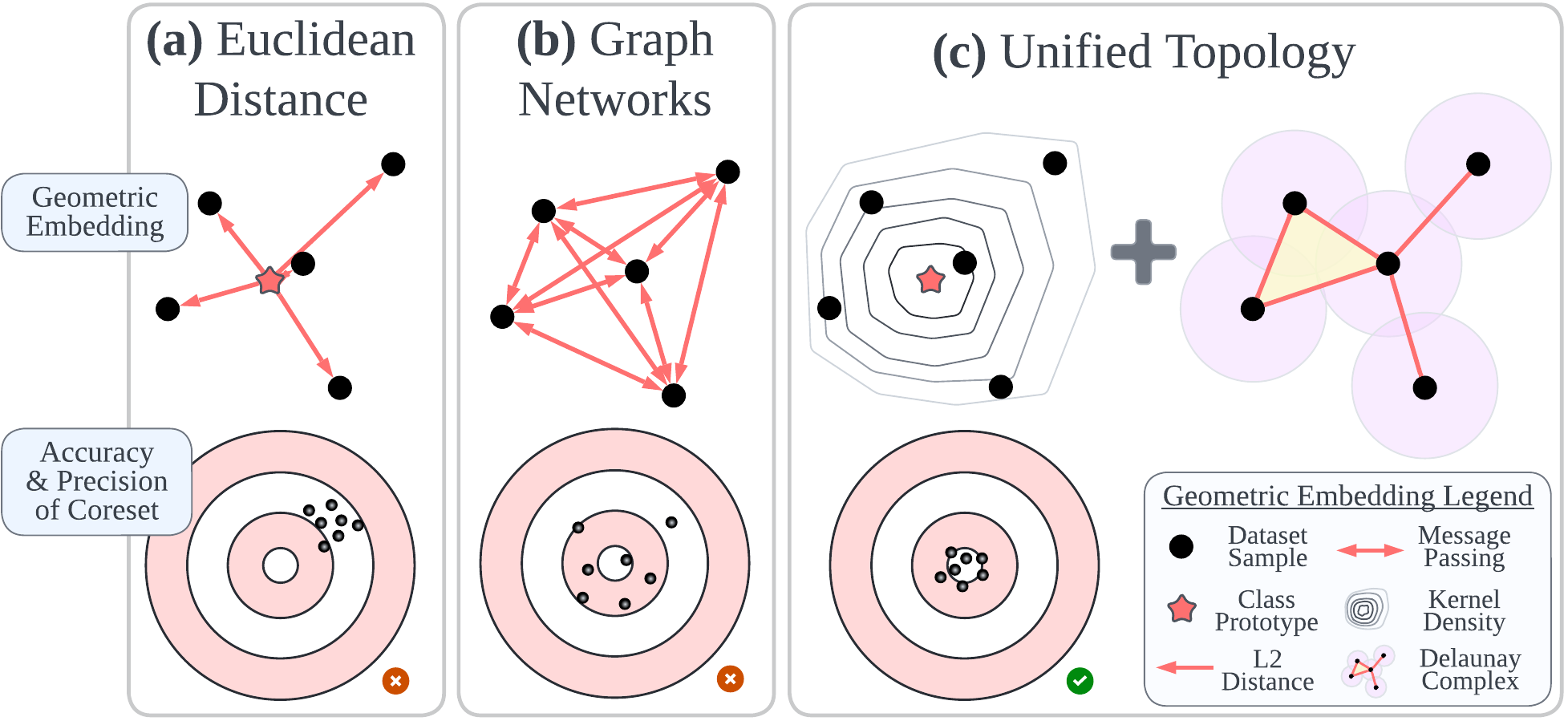}
  \end{center}
  \caption{\textbf{Topological data selection yields higher-performing and more stable coresets.} \textbf{(a)} Euclidean-based selection is precise but achieves lower accuracy. \textbf{(b)} Graph-based methods achieve higher accuracy but are highly variable. \textbf{(c)} Our topological approach achieves both high accuracy and precision.}
  \label{fig:topo_better}
  \vspace{-\baselineskip}
\end{wrapfigure}

The computational demands of training modern deep learning systems have escalated dramatically due to the scale of contemporary models and datasets. This growth has made training and fine-tuning computationally prohibitive, creating a need for data-efficient learning strategies. Data pruning is one such strategy which sub-samples a large dataset into a smaller, representative subset (or coreset) that preserves the learning characteristics of the full dataset. Thereby enabling rapid model training, efficient fine-tuning, and reduced storage costs, all while minimizing degradation in final model performance.

Broadly, coreset selection methods fall into three major categories. \textit{Optimization-based} methods select a coreset whose loss landscape \citep{glister, grad5} or gradient dynamics \citep{craig, gradmatch, grad2} align with the full dataset. While effective, such approaches are hampered by computationally intensive second-order \citep{grad4} or bilevel optimization \citep{grad1}. \textit{Score-based} methods rank samples using training dynamics \citep{forgetting, slocurve, elfs} or uncertainty estimations \citep{el2n, score1, score2, dual}. However, these scores are intrinsically biased by a models inductive prior, capacity bottleneck, and unique learning trajectory. Furthermore, requiring access to training dynamics makes both optimization and score-based methods incompatible with the growing ecosystem of pretrained models, where only final, static representations are accessible.

To overcome this constraint, \textit{geometry-based} coreset selection methods can operate on static embeddings from pretrained models. Approaches range from representations based on the penultimate-layer distances \citep{moderateds}, distributional similarity via optimal transport \citep{fdmat}, Wasserstein distance \citep{fairwass}, or geometric reconstruction error \citep{mindboundary}. While avoiding costly training analysis, these methods rely on metrics that are sensitive to the extrinsic geometry of the latent space, a vulnerability we term ``geometric brittleness'' \citep{beyondeuclid}. This brittleness leads to two shortcomings: (1) over-prioritization of dense regions at the expense of informative sparse distribution tails \citep{ccs}, and (2) an instability across network architectures or under direct embedding perturbation. This is most apparent in Euclidean-distance metrics \citep{moderateds} and message-passing or spectral graph methods \citep{d2, ses, cui2025fast}, which are sensitive to latent-space changes (see \cref{fig:topo_better}).

In this work, we introduce \sys, a novel framework that resolves the challenge of geometric brittleness by leveraging topology \citep{topo0, beyondeuclid}, which studies properties of space that are preserved under continuous deformations such as stretching and bending. As a canonical illustration, a coffee mug and a torus (donut) are topologically homeomorphic, one can be continuously deformed into the other because they share the same fundamental invariant (a single hole). By focusing on these stable, intrinsic invariants rather than transient, extrinsic geometric measurements (such as distance or curvature), we can analyze the latent space of datasets with a stable, topological metric. This structural focus allows \sys to achieve exceptional robustness against embedding perturbations; whether caused by isotropic noise, anisotropic distortion, or shifts across varying network architectures \citep{ph1, ph4}. This enables the use of proxy models \citep{proxy0} or off-the-shelf pretrained models for coreset generation without retraining.

Our framework first establishes a \textbf{global structure} by using topology-aware manifold approximation \citep{umap, pacmap} to project high-dimensional features into standardized low-dimensional embeddings. While this global structure can group similar samples, it fails to distinguish which samples to prioritize within a localized region. Existing methods often resort to random sampling within localized regions \citep{ccs} or use geometric heuristics like message-passing \citep{d2}. To better complement this global view with \textbf{local structure}, we then employ differentiable persistent homology \citep{dif_ph1, dif_ph2, dif_ph3} to assess a sample's structural relevancy, relative to its immediate neighbors. Persistent homology tracks the ``birth'' and ``death'' (persistence) of topological structures at multiple scales (derived from a filtration of simplicial complexes). For our application, we perform an optimization that \textit{maximizes the persistence (the lifespan of the features) of local topological structures} constructed from a multiparameter filtration of the manifold projected Delaunay complex \citep{ph5, ph17}. This process iteratively repositions samples to an optimal configuration that enhances topological stability, directly measuring a sample's contribution to the structural complexity of its local neighborhood.

Our approach makes the following contributions:
\begin{itemize}[itemsep=2mm, parsep=0pt, topsep=0mm]
    \item We introduce \sys, a novel coreset selection framework that defines sample importance through a dual-scale topological analysis. It combines a \textit{global manifold projection} with a \textit{local persistence} score derived from a differentiable persistent homology optimization to identify structurally critical samples with higher accuracy and precision compared to previous coreset methods, particularly in extreme compression regimes (e.g., 90\% pruning).
    \item We show that \sys is resilient to representation degradation, maintaining superior robustness and coreset quality across three distinct perturbation paradigms: isotropic latent noise, pixel-level input noise, and structured image corruptions.
    \item We demonstrate that \sys establishes robust cross-architecture transferability, consistently yielding high-quality coresets regardless of the transfer direction, whether utilizing diverse proxy embeddings (e.g., from ResNet to ViT) for a fixed target model or a single proxy to train a diverse set of target models. This flexibility permits the use of computationally inexpensive or off-the-shelf pretrained models for coreset generation without costly retraining.
\end{itemize}
\section{Background and Related Work}
\label{Sec:Background and Related Work}

\subsection{Topology at Two Scales}
While many modern topological algorithms inherently model both the global and local structure of data simultaneously, our work decouples these concepts into two distinct stages. For the purposes of this paper, we define global topology as the manifold structure of the entire dataset, which we capture as a low-dimensional embedding. We then define local topology as the fine-grained structure arising from the interactions between samples and their immediate neighbors, which we analyze using persistent homology.

\paragraph{Low-Dimensional Manifold Approximations.}
A critical step in high-dimensional analysis is creating a low-dimensional data representation. Linear methods like PCA \citep{pca} are efficient but preserve only global variance while missing non-linear structures. t-SNE \citep{tsne} excels at preserving fine-grained local neighborhoods but distorts global structure. Topology-aware manifold approximation and projection (MAP) methods such as UMAP \citep{umap}, PaCMAP \citep{pacmap}, and DensMAP \citep{densmap} model high-dimensional data as a fuzzy topological structure to preserve both fine-grained local connectivity and large-scale global relationships.

Learning-based alternatives such as Topological Autoencoders (TopoAE) \citep{topoae} and Representation Topological Divergence (RTD) \citep{rtd} offer powerful global topology preservation. However, they require computationally expensive training and often yield less distinct class separation on complex datasets. Topology-preserving objectives have similarly been applied to neural network embedding spaces in knowledge distillation \citep{kim2024topological}, semi-supervised continual learning \citep{fan2024persistence}, and vision-language models \citep{zhang2024homology, you2026topological}, each of which also requires training. Therefore, we prioritize algorithmic solutions like UMAP for efficiency and alignment with our training-free objective. These MAP-based methods have proven highly effective for interpreting the complex representations learned by deep models across numerous domains, from single-cell genomics \citep{man1} to clustering in dictionary learning \citep{man3, man4}. As shown by \cite{topo2}, these techniques produce embeddings with compact and well-defined clusters, enhancing the impact of downstream analysis. 

By preserving nearest-neighbor structure from the high-dimensional space, MAP-based methods capture the underlying topological connectivity of the data manifold. This preserves notions of ``prototypicality'' (samples in highly connected regions) and ``atypicality'' (samples in sparse regions) after projection (see \cref{Sec:Appendix_topology_memorization} for a qualitative explanation on prototypicality and \cref{Sec:Appendix_manifold_projection_techniques} for an investigation across projection techniques).

\begin{figure*}[!tp]
\centering
\includegraphics[width=\textwidth]{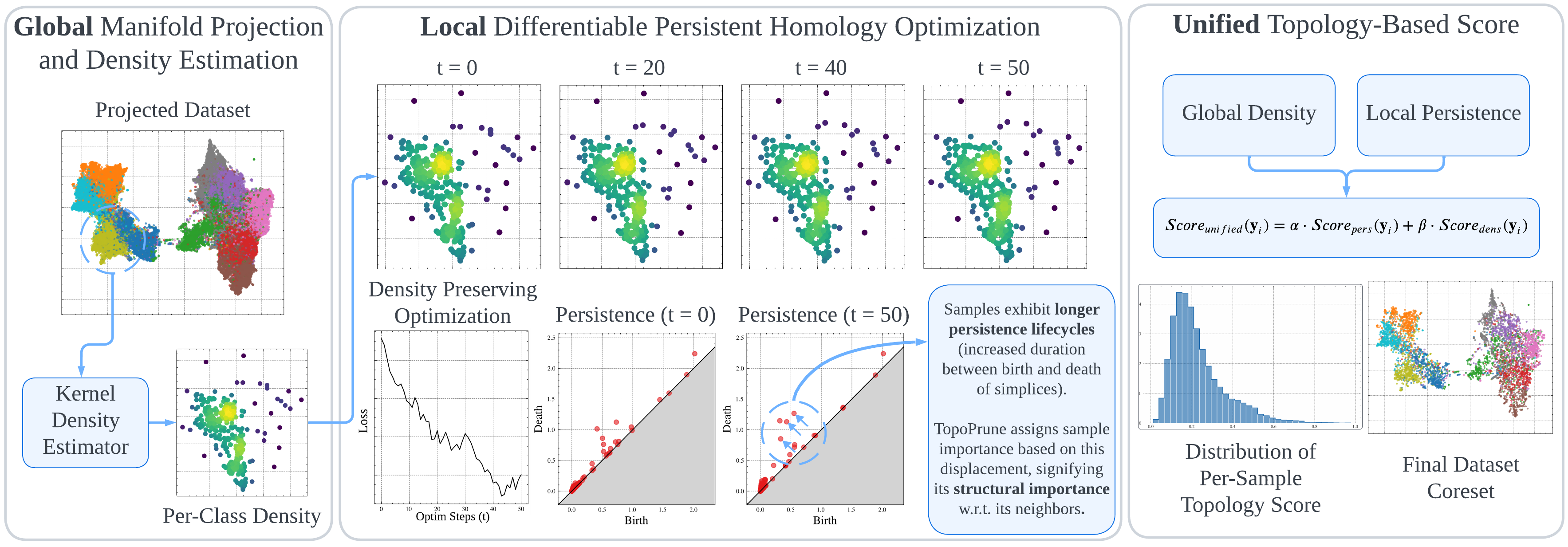}
\caption{\textbf{An overview of \sys}. \textbf{(Left)} A topology-aware projection visualizes the \textit{global} data manifold. \textbf{(Middle)} Within each class, a density-preserving persistent homology optimization derives a \textit{local} persistence score per sample. The color map indicates high (yellow) to low (blue) density. \textbf{(Right)} The final coreset is constructed via stratified sampling on a unified score combining \textit{global density} and \textit{local persistence}. This not only prioritizes the most topologically informative samples but also faithfully represents the density distribution of the original dataset.}
\vspace{-\baselineskip}
\label{fig:topocore_overview}
\end{figure*}

\paragraph{Interactions of Samples and their Nearest Neighbors.}
Understanding the local interactions between a sample and its neighbors is crucial for determining its importance. A prevalent approach is the use of Graph Neural Networks (GNNs), which propagate information between nodes on a graph typically defined by nearest-neighbor relationships. In GNNs, a sample's importance is quantified through message-passing that aggregates features from its local neighborhood \citep{d2}. Other methods use graph-level structural entropy combined with Shapley values and blue noise sampling \citep{bluenoise} to select a diverse coreset \citep{ses}. FAST \citep{cui2025fast} likewise describes topology through the spectrum (symmetric normalized Laplacian) of a nearest-neighbor graph, which it uses to regularize distribution matching. However, these approaches operate on a single, fixed graph and can be sensitive to the geometric hyperparameters used in its construction.

In contrast, persistent homology \citep{ph18, ph19} offers a fundamentally different and more robust framework. Instead of analyzing a single graph, it studies the evolution of higher-order topological structures (e.g., connected components, loops, voids) across a multiparameter filtration of simplicial complexes \citep{ph21, ph22, ph1.2}. This provides a complete summary of the data's shape at all scales simultaneously. A key advantage of persistent homology is its proven stability \citep{ph1}. The persistence diagram of a dataset is guaranteed to change only slightly in response to small perturbations of the input data, making it a robust descriptor of local structure \citep{ph3, ph7}. Inclusion of these robust geometric descriptors has been widely used for understanding feature embeddings in machine learning applications such as monitoring generalization in networks over training \citep{ph2} and exploring the topology of latent embeddings throughout network layers \citep{ph13, ph23}. For a detailed overview on the underlying construction of simplicial complexes and persistent homology, please see \cref{Sec:appendix_persistent_homology}.

While traditional persistent homology provides a powerful descriptive tool, its integration into modern deep learning pipelines has been limited as it is not inherently differentiable. Recent advances in differentiable persistent homology have overcome this barrier by enabling the backpropagation of gradients from the persistence diagram back to the coordinates of the input data points \citep{dif_ph2, dif_ph3}. This allows for the direct optimization of the data's topological features within a gradient-based framework. The work of \cite{dif_ph1} provides a fast and stable computational framework for these gradients, even for the more expressive case of multiparameter persistent homology. By leveraging this, rather than simply describing static local topology, we can perform an optimization to actively enhance it.

\section{Methodology}
\label{Sec:Methodology}

Our proposed method, \sys, constructs a coreset by analyzing the data's topological structure at two distinct scales. (1) \textit{Global Manifold Embedding}, projects the original high-dimensional embeddings into a standardized low-dimensional space. This ensures a stable, global view of the data's overall structure. (2) \textit{Local Topological Interaction}, which employs differentiable multi-parameter persistent homology to probe the local structure formed by samples and their closest neighbors. Together, these two topological scales are used to derive an importance score for each sample based on global density and local persistent homology, contributing to a unified topological measurement for selecting individual samples (see \cref{fig:topocore_overview}).

\subsection{Global Structure: Dataset Representation with Topological Manifold Embedding}

Given a well-trained deep model, denoted by $f(\cdot)$, we can express it as a composition of a feature extractor $h(\cdot)$ and a classifier $g(\cdot)$, such that $f(\cdot) = g(h(\cdot))$. Here, $h(\cdot)$ represents the network up to the \textit{penultimate layer}, which maps an input data point $\mathbf{x}$ to a high-dimensional feature embedding $\mathbf{z} = h(\mathbf{x}) \in \mathbb{R}^D$. The full dataset $\mathcal{D} = \{(\mathbf{x}_1, y_1), \dots, (\mathbf{x}_N, y_N)\}$ can thus be transformed into a high-dimensional feature set $Z = \{\mathbf{z}_1, \dots, \mathbf{z}_N\}$. While this high-dimensional space $Z$ contains rich semantic information, its extrinsic geometry is often complex and architecture-dependent. To obtain a stable and standardized representation, we project $Z$ onto a low-dimensional manifold using topology-based manifold approximation and projection techniques \citep{umap, pacmap, densmap}.

This process involves two main stages. First, a topological representation of the high-dimensional data is constructed as a fuzzy simplicial set. This structure captures the data's shape by assigning a membership strength ($p_{ij}$), to the potential connections between each point and its neighbors, where the ``fuzzy'' aspect represents the belief that a certain simplex exists in the true underlying manifold. Subsequently, a low-dimensional embedding is learned $Y = \{\mathbf{y}_1, \dots, \mathbf{y}_N\}$, where $\mathbf{y}_i \in \mathbb{R}^d$ and $d \ll D$, whose own fuzzy simplicial set ($q_{ij}$) is similarly defined. The final low-dimensional representation $Y$ is found by optimizing the positions of the points $\{\mathbf{y}_i\}$ to minimize a cross-entropy loss between the high-dimensional ($p_{ij}$) and low-dimensional ($q_{ij}$) pairwise similarities:
\begin{equation}\label{eq:umap_loss}
\mathcal{L}_{\text{proj}}(Y) = \sum_{i \neq j} \left[ p_{ij} \log\left(\frac{p_{ij}}{q_{ij}}\right) + (1 - p_{ij}) \log\left(\frac{1 - p_{ij}}{1 - q_{ij}}\right) \right]
\end{equation}

This process yields a standardized manifold embedding that preserves the data's intrinsic shape. Through a detailed investigation into different manifold approximation and projection techniques presented in \cref{Sec:Appendix_manifold_projection_techniques} we use UMAP \citep{umap} as it creates uniform manifold embeddings across network architectures (see \cref{Sec:appendix_umap_sensitivity} exploring UMAP hyperparameters). On this low-dimensional manifold we compute \textbf{Density Score}, a class-conditional rarity metric for each sample using a Kernel Density Estimator (KDE). Rather than using density directly, we adopt its negative log-transform, the self-information (or surprisal) of the sample under the estimated density \citep{cover2006elements, sun2022information}. For an event $x$ with probability $p(x)$, the self-information $-\log p(x)$ quantifies how unexpected, or informative, that event is. High-probability events convey little information, while rare events convey much. Applied to a sample's class-conditional density, this transforms the score into a monotonic measure of structural rarity that increases with informativeness. For a sample $\mathbf{y}_i$ belonging to class $c$:
\begin{equation}
\text{Score}_{\text{dens}}(\mathbf{y}_i) = -\log \left( \frac{1}{N_c h} \sum_{\mathbf{y}_j \in Y_c} K\!\left(\frac{\mathbf{y}_i - \mathbf{y}_j}{h}\right) \right)
\end{equation}
where $N_c = |Y_c|$ is the total number of samples in class $c$, $K$ is a Gaussian kernel, and $h$ is the bandwidth. This score allows us to distinguish samples in low-density (atypical) regions of the manifold from those in high-density (prototypical) regions.

\subsection{Local Structure: Sample Neighborhoods with Persistence-Based Optimizer}

The global manifold embedding provides a low-dimensional representation that faithfully approximates the global structure of the data manifold. While this ensures a stable, high-level representation, a purely global perspective is insufficient for identifying the most informative samples, particularly in extreme compression regimes (e.g., 90\% pruning). Samples co-located in the embedding space provide highly redundant information during model optimization, yielding diminishing marginal returns for gradient updates \citep{sener2018active, scaling2}. Because these tightly clustered samples inherently exhibit nearly identical density scores, relying exclusively on density-based selection within a local neighborhood fails to resolve structural importance, effectively degenerating into local uniform sampling \citep{d2, ccs}.

To resolve this local ambiguity and capture fine-grained structure, we leverage persistent homology not as a static descriptor, but as a dynamic topological optimization process. The objective of this process is to iteratively adjust the position of each point within its class manifold to maximize persistence life-cycles (increasing the duration between birth and death of topological features within the simplicial complex). This is performed independently for each class $c \in \{1, \dots, C\}$ to analyze the specific intra-class structure. For each class, we begin with its point cloud from the global manifold embedding $Y_c = \{\mathbf{y}_i \mid \text{label}(\mathbf{y}_i) = c\}$ and construct a Delaunay filtration \citep{ph17} on $Y_c$, which circumvents prohibitive combinatorial explosions typically associated with the standard Vietoris-Rips (VR) complex \citep{ph9, ph7}. By restricting VR-style weights to the Delaunay skeleton, we accelerate computation while maintaining stability guarantees in our low-dimensional setting.

\begin{figure*}[!tp]
    \centering
    \begin{subfigure}[b]{0.32\textwidth}
        \centering
        \includegraphics[width=\textwidth]{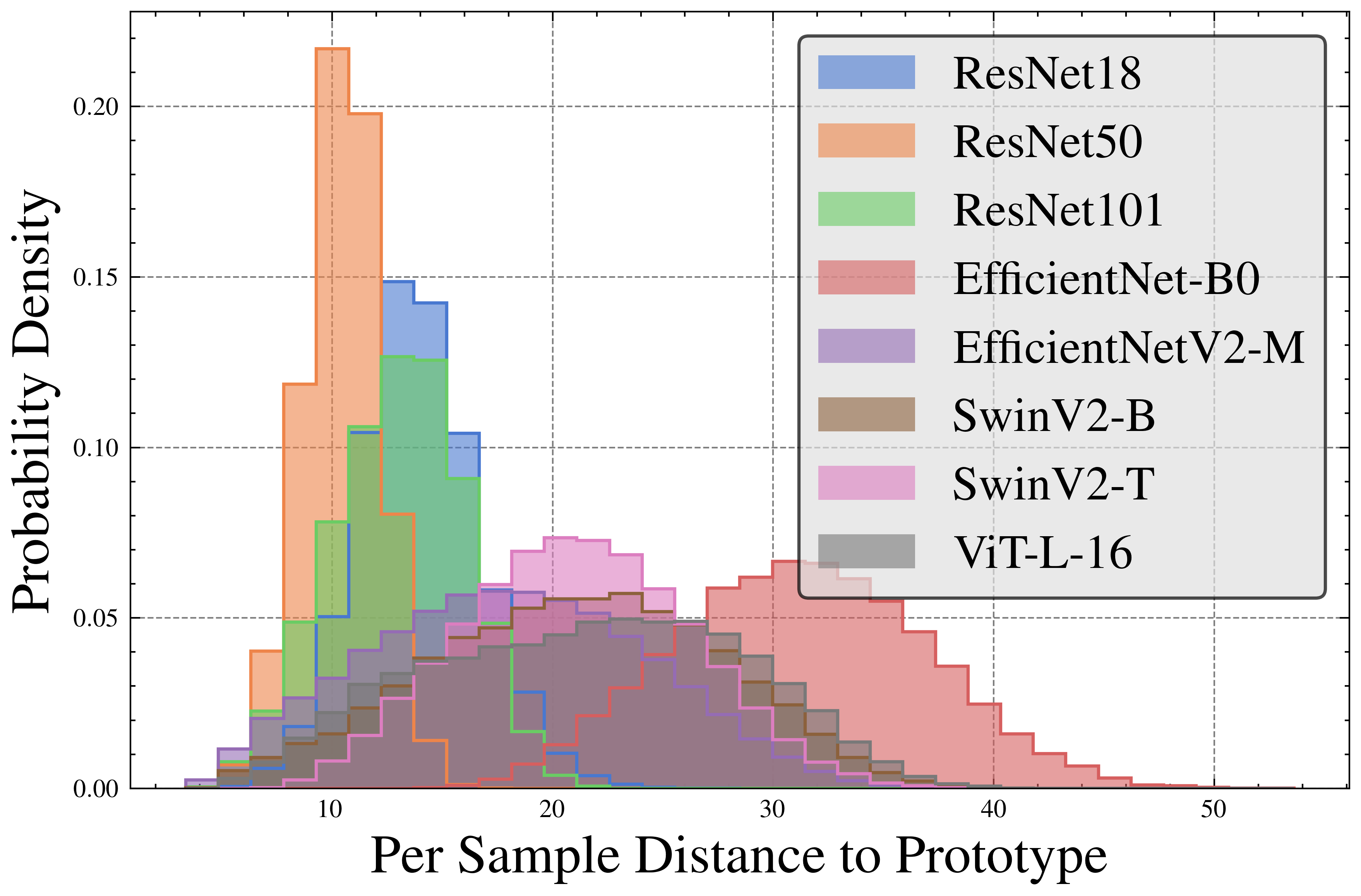}
        \caption{\footnotesize Euclidean Distance to Prototype}
        \label{fig:proxy_networks_distance}
    \end{subfigure}
    \begin{subfigure}[b]{0.32\textwidth}
        \centering
        \includegraphics[width=\textwidth]{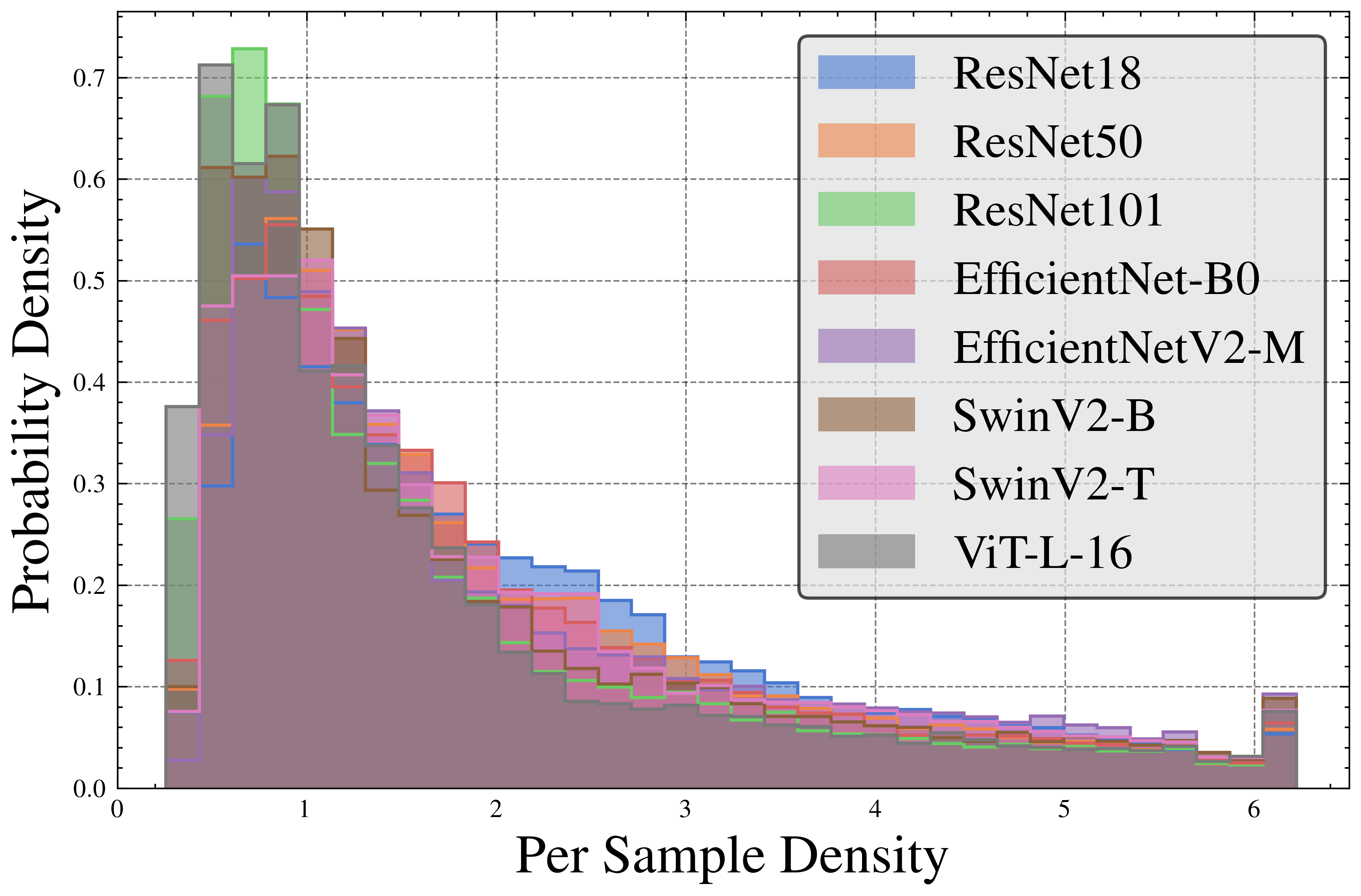}
        \caption{\footnotesize Global Manifold Kernel Density}
        \label{fig:proxy_networks_density}
    \end{subfigure}
    \begin{subfigure}[b]{0.32\textwidth}
        \centering
        \includegraphics[width=\textwidth]{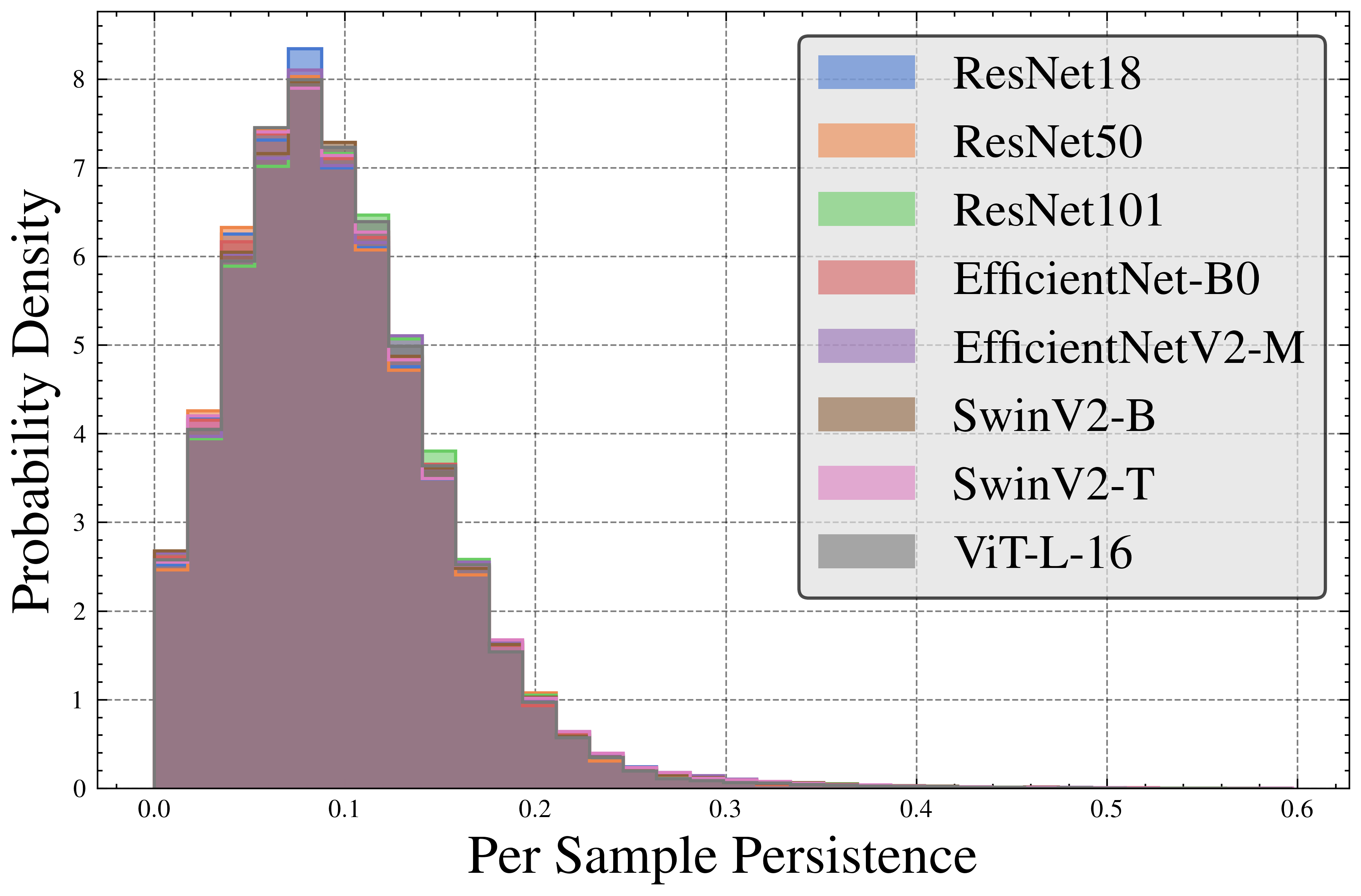}
        \caption{\footnotesize Local Persistence}
    \label{fig:proxy_networks_persistence}
    \end{subfigure}
    
    \caption{\textbf{Topological metrics are more consistent across networks}. Which translates directly to better coreset performance. Metric distributions become progressively more uniform as we move from \textbf{(a)} unstable Euclidean distances, to \textbf{(b)} density estimation from global topological projection, and finally to \textbf{(c)} local persistence.}
    \label{fig:proxy_networks_all}
\end{figure*}

Similar to work from \cite{dif_ph1} we define a differentiable objective, $\mathcal{L}_{\text{pers}}(Y_c)$, whose gradient, $-\nabla_{Y_c} \mathcal{L}_{\text{pers}}$, points in the direction of steepest ascent, \textit{maximally increasing the persistence life-cycle of samples}. This objective is formulated using a multi-parameter filtration considering two parameters: (1) the class-manifold Delaunay filtration ($Del_{Y_c}$) and (2) the class-manifold Kernel Density Estimator ($\hat{f}= KDE_{Y_c}$). The persistence of this two-parameter filtration is summarized using the Hilbert decomposition signed measure, of homology degree 1 ($H_1$), denoted $ \scriptstyle \mu_{H_1(Del_{Y_c}, \hat{f})}^{Hil}$ \citep{ph14}. We ablate the choice of filtration and homology degree in \cref{Sec:appendix_filtration_homology}. This descriptor represents the persistence diagram as a finite collection of positive point masses (representing feature births) and negative point masses (representing feature deaths) in the parameter space of (distance, density). Our objective is to maximize persistence life-cycles by maximizing the Optimal Transport (OT) distance between this signed measure and the zero measure, $\mathbf{0}$ \citep{ph11}. The zero measure acts as a trivial baseline with no mass. OT distances reduce to the internal matching cost between positive (birth) and negative (death) point masses, capturing the total topological persistence. The differentiable objective for a given class $c$ is therefore defined as:
\begin{equation}
    \mathcal{L}_{\text{pers}}(Y_c) = \text{OT}(\mu_{H_1(Del_{Y_c}, \hat{f})}^{Hil}, \mathbf{0})\label{eq:lpers}
\end{equation}

The optimization seeks a new point configuration $Y'_c$ that maximizes this objective, solved iteratively via gradient ascent (see \cref{Sec:appendix_tda_steps_ablation} exploring optimization steps). This formulation \textit{resolves topological ambiguity by separating features from the diagonal, while preserving the original density of the class manifold}, as the density is recomputed at each epoch and is an integral part of the objective. We then define the \textbf{Persistence Score} for each sample $\mathbf{y}_i$ belonging to class $c$ as the magnitude of its total displacement during its class-specific optimization, where $\mathbf{y}_i$ is the initial position and $\mathbf{y'}_i$ is the final, optimized position.:
\begin{equation}
    \text{Score}_{\text{pers}}(\mathbf{y}_i) = \| \mathbf{y}_i - \mathbf{y'}_i \|_2, \quad \text{for } \mathbf{y}_i \in Y_c, \mathbf{y'}_i \in Y'_c
\end{equation}

\paragraph{Interpreting this notion of local dataset structure.} 
A high persistence score quantifies the topological ambiguity a sample introduces in its class manifold. Density-preservation during optimization is vital, ensuring our search for structurally important samples does not distort global representativeness. The magnitude of corrective displacement therefore serves as a direct, dynamic measure of a sample's contribution to the topological complexity of its class.

\subsection{Comprehensive Score with Global and Local Dataset Structures}
\label{sec:comprehensive_score}

\begin{wrapfigure}{r}{0.35\linewidth}
    \vspace{-1em}
    \begin{center}
    \includegraphics[width=\linewidth]{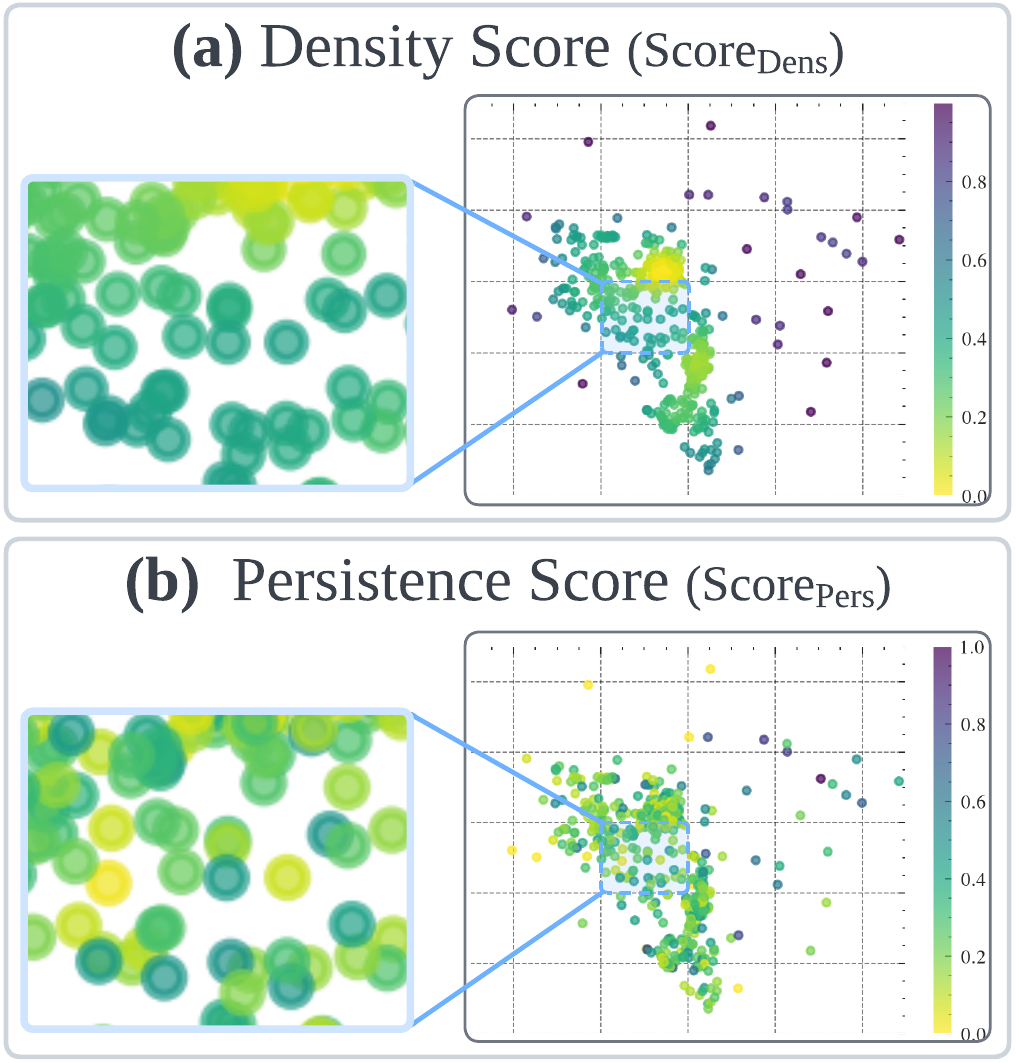}
    \end{center}
    \caption{\textbf{Dual-scale scoring.} \textbf{(a)} Global Density Score provides structural prior but assigns saturated values within localized regions. \textbf{(b)} Local Persistence Score provides fine-grained neighborhood structure, isolating topological anchors.}
    \label{fig:dens_pers_example}
\end{wrapfigure}

To construct a highly effective coreset, we formulate a comprehensive sample importance metric that unifies global density and local persistence. As $\text{Score}_{\text{dens}}$ measures density surprisal and $\text{Score}_{\text{pers}}$ measures topological complexity, both signals quantify a sample's structural informativeness. To ensure neither metric implicitly dominates the formulation due to scale differences, both signals are independently normalized to the $[0, 1]$ range. This unified score is computed as a weighted combination, where $\mathcal{N}(\cdot)$ denotes min-max normalization, and hyperparameters $\alpha, \beta \in [0, 1]$ modulate the influence of local topological complexity versus global distributional rarity (see \cref{fig:dens_pers_example} for visualization):
\begin{equation}
    \text{Score}_{\text{unified}}(\mathbf{y}_i) = \alpha \cdot \mathcal{N}(\text{Score}_{\text{pers}}(\mathbf{y}_i)) + \beta \cdot \mathcal{N}(\text{Score}_{\text{dens}}(\mathbf{y}_i))\label{eq:toposcore}
\end{equation}

We empirically validate the complementary nature of these metrics through three analyses provided in \cref{Sec:appendix_density_persistence_ablation}. First, a joint distribution analysis across 10 network architectures reveals a consistently low linear correlation between persistence and density (average Pearson $r=0.102$), confirming that the two scores capture independent structural properties. Second, qualitative analysis of target class embeddings illustrates that proximate samples with identical global densities frequently exhibit divergent local persistence scores. Finally, an ablation on the weighted combination demonstrates that integrating both signals outperforms either metric in isolation (e.g., by up to 5.4\% at high pruning rates on CIFAR-100). Together, these analyses confirm that $\text{Score}_{\text{unified}}$ measures structural informativeness from complementary views: global rarity and local topological complexity.

\subsection{Mislabel Filtering and Coreset Construction}

Following CCS \citep{ccs}, we ensure our coreset is not corrupted by noisy or mislabeled data, which can receive high importance scores yet degrade model performance \citep{score3}. We incorporate a filtering step to create a clean dataset $\mathcal{D}_{clean} = \mathcal{D} \setminus \mathcal{I}_{mis}$ where $\mathcal{I}_{mis}$ are the mislabeled sample indices. Most prior methods identify such samples using training-dynamic metrics like Area Under the Margin (AUM) \citep{aum}, which conflicts with our training-free objective. Therefore, we utilize a Neighborhood Label Purity Score (NLPS). Drawing on the training-free local voting introduced by \citet{Zhu2022DetectingCL}, NLPS calculates the fraction of a sample's k-nearest latent-space neighbors with a different class label. A high NLPS indicates a mixed-label neighborhood, effectively serving as a training-free analog to the ``flip-flop'' candidates identified by AUM. We validate NLPS against alternative proxies in \cref{Sec:appendix_mislabel_proxy_ablation} and adopt the mislabel ratios from \citet{ccs} (Appendix \cref{table:hyperparameters}).

From the resulting clean dataset $\mathcal{D}_{clean}$, we construct the coreset via stratified sampling on $\text{Score}_\text{unified}$, preserving the original class distribution as in \citet{ccs}. TopoPrune produces an unbalanced coreset, respecting the dataset's intrinsic class imbalance rather than enforcing uniform per-class counts. The full pipeline therefore comprises three phases: (1) dual-scale topological scoring, (2) NLPS-based mislabel filtering, and (3) stratified topological selection. Pseudocode appears in \cref{Sec:Appendix_pseudocode} and an illustrative walkthrough in \cref{Sec:appendix_illustrative}.

\section{Results}
\label{Sec:Results}

\subsection{Experimental Setup}
\label{sec:experimental_setup}

\begin{wraptable}{r}{0.61\textwidth}
\vspace{-\baselineskip}
\vspace{-\baselineskip}
\vspace{-\baselineskip}
\vspace{-\baselineskip}
\caption{\textbf{Accuracy across coreset selection methods} on CIFAR-10, CIFAR-100 and ImageNet-1K. \sys demonstrates a scaling advantage: while competitive on simpler datasets, performance improvements over baselines widen as dataset complexity and pruning severity increase.}
\setlength{\tabcolsep}{3.1pt}
\centering
\resizebox{\linewidth}{!}
{  
\begin{tabular}{lcccccc}
    \toprule
    & Pruning Rate ($\rightarrow$) & 30\% & 50\% & 70\% & 80\% & 90\% \\
    \cmidrule(lr){3-7}
    & & \multicolumn{5}{c}{\textbf{CIFAR-10 \footnotesize(ResNet-18)}} \\
    \midrule
    \multirow{4}{*}{
    \begin{tabular}
        [c]{@{}c@{}}\small \textit{No}\\ \small \textit{Training}\\ \small \textit{Dynamics}
    \end{tabular}
    } & Random & 94.5\footnotesize $\pm$0.1 & 93.5\footnotesize $\pm$0.1 & \underline{90.8\footnotesize $\pm$0.2} & 86.6\footnotesize $\pm$0.3 & 76.7\footnotesize $\pm$0.9  \\
    & Moderate & 94.2\footnotesize $\pm$0.1 & 93.1\footnotesize $\pm$0.1 & 89.9\footnotesize $\pm$0.2 & 87.2\footnotesize $\pm$0.2 & 76.9\footnotesize $\pm$1.0  \\
    & FDMat  & 94.7\footnotesize $\pm$0.1 & \underline{93.6\footnotesize $\pm$0.2} & \textbf{90.8\footnotesize $\pm$0.2} & \underline{87.3\footnotesize $\pm$0.4} & 74.4\footnotesize $\pm$0.7   \\
    & \textbf{\sys (NLPS)}  & \textbf{94.8\footnotesize $\pm$0.1} & \textbf{93.6\footnotesize $\pm$0.2} & 90.3\footnotesize $\pm$0.2 & \textbf{87.3\footnotesize $\pm$0.3} & \textbf{77.1\footnotesize $\pm$0.6}  \\
    \midrule
    \multirow{10}{*}{
    \begin{tabular}
        [c]{@{}c@{}}\small \textit{With}\\ \small \textit{Training}\\ \small \textit{Dynamics}
    \end{tabular}
    } & Moderate (AUM) & 93.9\footnotesize $\pm$0.2 & 93.1\footnotesize $\pm$0.2 & 90.1\footnotesize $\pm$0.2 & 87.1\footnotesize $\pm$0.2 & 79.9\footnotesize $\pm$0.3  \\
    & Forgetting  & 94.5\footnotesize $\pm$0.2 & 92.6\footnotesize $\pm$0.1 & 89.8\footnotesize $\pm$0.2 & 85.6\footnotesize $\pm$0.3 & 67.6\footnotesize $\pm$0.4  \\
    & Glister & 94.4\footnotesize $\pm$0.2 & 93.8\footnotesize $\pm$0.2 & 90.8\footnotesize $\pm$0.4 & 85.1\footnotesize $\pm$0.6 & 66.8\footnotesize $\pm$1.3  \\
    & LCMat-S & 94.5\footnotesize $\pm$0.2 & 93.3\footnotesize $\pm$0.2 & 90.5\footnotesize $\pm$0.2 & 86.9\footnotesize $\pm$0.2 & 75.1\footnotesize $\pm$0.8 \\
    & ELFS & 94.7\footnotesize $\pm$0.2 & 93.7\footnotesize $\pm$0.2 & 90.6\footnotesize $\pm$0.3 & 87.5\footnotesize $\pm$0.4 & 79.8\footnotesize $\pm$0.9 \\
    & CCS & 95.5\footnotesize $\pm$0.1 & \underline{94.8\footnotesize $\pm$0.2} & 93.0\footnotesize $\pm$0.2 & \textbf{90.7\footnotesize $\pm$0.2} & 81.9\footnotesize $\pm$0.7  \\
    & FAST & 95.2\footnotesize $\pm$0.1 & 93.2\footnotesize $\pm$0.2 & 90.6\footnotesize $\pm$0.3 & 87.4\footnotesize $\pm$0.4 & 80.2\footnotesize $\pm$1.4 \\
    & D2 & \textbf{95.6\footnotesize $\pm$0.1} & \textbf{94.8\footnotesize $\pm$0.1} & \textbf{93.1\footnotesize $\pm$0.1} & 89.2\footnotesize $\pm$0.2 & 80.9\footnotesize $\pm$1.5   \\
    & \textbf{\sys}  & 94.9\footnotesize $\pm$0.2 & 94.2\footnotesize $\pm$0.2 & 92.1\footnotesize $\pm$0.2 & 89.4\footnotesize $\pm$0.2 & \textbf{82.0\footnotesize $\pm$0.2}  \\

    \midrule
    & & \multicolumn{5}{c}{\textbf{CIFAR-100 \footnotesize(ResNet-18)}} \\
    \midrule
    \multirow{4}{*}{
    \begin{tabular}
        [c]{@{}c@{}}\small \textit{No}\\ \small \textit{Training}\\ \small \textit{Dynamics}
    \end{tabular}
    } & Random & 75.3\footnotesize $\pm$0.2 & 71.6\footnotesize $\pm$0.1 & 63.7\footnotesize $\pm$0.5 & 55.9\footnotesize $\pm$1.0 & 34.0\footnotesize $\pm$1.1  \\
    & Moderate  & 74.9\footnotesize $\pm$0.3 & 70.1\footnotesize $\pm$0.3 & 63.7\footnotesize $\pm$0.2 & 56.1\footnotesize $\pm$0.5 & 34.9\footnotesize $\pm$2.1  \\
    & FDMat & 75.4\footnotesize $\pm$0.2 & \underline{71.9\footnotesize $\pm$0.3} & 64.0\footnotesize $\pm$0.6 & 56.1\footnotesize $\pm$1.5 & 37.5\footnotesize $\pm$1.6 \\
    & \textbf{\sys (NLPS)}  & \textbf{75.6\footnotesize $\pm$0.2} & \textbf{71.9\footnotesize $\pm$0.2} & \textbf{65.3\footnotesize $\pm$0.4} & \textbf{56.7\footnotesize $\pm$0.4} & \textbf{41.6\footnotesize $\pm$0.8} \\
    \midrule
    \multirow{10}{*}{
    \begin{tabular}
        [c]{@{}c@{}}\small \textit{With}\\ \small \textit{Training}\\ \small \textit{Dynamics}
    \end{tabular}
    } & Moderate (AUM) & 75.9\footnotesize $\pm$0.3 & 72.4\footnotesize $\pm$0.2 & 66.7\footnotesize $\pm$0.3 & 60.2\footnotesize $\pm$0.8 & 40.0\footnotesize $\pm$1.2  \\
    & Forgetting & 74.8\footnotesize $\pm$0.2 & 67.2\footnotesize $\pm$0.9 & 50.6\footnotesize $\pm$0.7 & 32.3\footnotesize $\pm$0.9 & 24.3\footnotesize $\pm$1.4 \\
    & Glister & 75.8\footnotesize $\pm$0.3 & 70.7\footnotesize $\pm$0.7 & 66.1\footnotesize $\pm$1.2 & 54.7\footnotesize $\pm$1.6 & 38.4\footnotesize $\pm$1.7 \\
    & LCMat-S & 75.3\footnotesize $\pm$0.2 & 71.1\footnotesize $\pm$0.2 & 62.5\footnotesize $\pm$0.8 & 52.1\footnotesize $\pm$2.0 & 36.1\footnotesize $\pm$1.7  \\
    & ELFS & 75.5\footnotesize $\pm$0.4 & 72.5\footnotesize $\pm$0.4 & 66.7\footnotesize $\pm$0.6 & 59.0\footnotesize $\pm$1.0 & 42.9\footnotesize $\pm$1.3 \\
    & CCS & \textbf{76.9\footnotesize $\pm$0.3} & \underline{73.8\footnotesize $\pm$0.3} & 67.8\footnotesize $\pm$0.7 & 60.7\footnotesize $\pm$0.6 & 45.2\footnotesize $\pm$2.4  \\
    & FAST & 76.4\footnotesize $\pm$0.2 & 72.8\footnotesize $\pm$0.2 & 66.9\footnotesize $\pm$0.8 & 60.3\footnotesize $\pm$2.6 & 42.6\footnotesize $\pm$2.4 \\
    & D2 & 75.1\footnotesize $\pm$0.5 & 71.2\footnotesize $\pm$0.2 & 67.8\footnotesize $\pm$0.9 & 61.1\footnotesize $\pm$1.4 & 44.3\footnotesize $\pm$2.6 \\
    & \textbf{\sys}  & 76.7\footnotesize $\pm$0.3 & \textbf{73.8\footnotesize $\pm$0.3} & \textbf{68.1\footnotesize $\pm$0.3} & \textbf{62.3\footnotesize $\pm$0.4} & \textbf{45.7\footnotesize $\pm$0.5} \\

    \midrule
    & & \multicolumn{5}{c}{\textbf{ImageNet-1K \footnotesize(ResNet-50)}} \\
    \midrule
    \multirow{4}{*}{
    \begin{tabular}
        [c]{@{}c@{}}\small \textit{No}\\ \small \textit{Training}\\ \small \textit{Dynamics}
    \end{tabular}
    } & Random  & 69.8\footnotesize $\pm$0.5 & 68.4\footnotesize $\pm$0.5 & 65.1\footnotesize $\pm$0.4 & 61.9\footnotesize $\pm$0.5 & 52.5\footnotesize $\pm$0.6 \\
    & Moderate & 69.5\footnotesize $\pm$0.2 & 65.8\footnotesize $\pm$0.4 & 60.5\footnotesize $\pm$0.1 & 57.7\footnotesize $\pm$0.2 & 50.0\footnotesize $\pm$0.4 \\
    & FDMat & \textbf{70.8\footnotesize $\pm$0.3} & 68.7\footnotesize $\pm$0.5 & 65.5\footnotesize $\pm$0.7 & 62.0\footnotesize $\pm$0.3 & 51.9\footnotesize $\pm$0.3 \\
    & \textbf{\sys (NLPS)} & 70.7\footnotesize $\pm$0.4 & \textbf{69.9\footnotesize $\pm$0.2} & \textbf{66.4\footnotesize $\pm$0.2} & \textbf{63.2\footnotesize $\pm$0.3} & \textbf{53.9\footnotesize $\pm$0.2} \\
    \midrule
    \multirow{10}{*}{
    \begin{tabular}
        [c]{@{}c@{}}\small \textit{With}\\ \small \textit{Training}\\ \small \textit{Dynamics}
    \end{tabular}
    } & Moderate (AUM)  & 69.6\footnotesize $\pm$0.4 & 67.2\footnotesize $\pm$0.6 & 63.9\footnotesize $\pm$0.8 & 60.4\footnotesize $\pm$0.6 & 52.7\footnotesize $\pm$0.3 \\
    & Forgetting  & 69.9\footnotesize $\pm$0.2 & 66.8\footnotesize $\pm$0.6 & 60.2\footnotesize $\pm$0.5 & 59.1\footnotesize $\pm$0.4 & 50.0\footnotesize $\pm$0.5 \\
    & Glister & 66.3\footnotesize $\pm$0.4 & 63.5\footnotesize $\pm$0.3 & 59.3\footnotesize $\pm$0.5 &  56.5\footnotesize $\pm$0.3 & 49.3\footnotesize $\pm$0.8 \\
    & LCMat-S & 69.8\footnotesize $\pm$0.4 & 67.5\footnotesize $\pm$0.5 & 62.2\footnotesize $\pm$0.3 & 59.7\footnotesize $\pm$0.5 & 48.8\footnotesize $\pm$0.6  \\
    & ELFS & 69.9\footnotesize $\pm$0.2 & 68.7\footnotesize $\pm$0.2 & 65.0\footnotesize $\pm$0.3 & 62.5\footnotesize $\pm$0.4 & 55.0\footnotesize $\pm$0.7 \\
    & CCS & 70.1\footnotesize $\pm$0.5 & 69.1\footnotesize $\pm$0.3 & 65.7\footnotesize $\pm$0.3 & 62.6\footnotesize $\pm$0.6 & 55.2\footnotesize $\pm$0.7 \\
    & FAST & 70.2\footnotesize $\pm$0.2 & 68.4\footnotesize $\pm$0.4 & 65.5\footnotesize $\pm$0.3 & 62.9\footnotesize $\pm$0.5 & 55.6\footnotesize $\pm$0.6 \\
    & D2  & 69.5\footnotesize $\pm$0.3 & 67.1\footnotesize $\pm$0.5 & 65.7\footnotesize $\pm$0.4 & 62.7\footnotesize $\pm$0.9 & 55.5\footnotesize $\pm$1.3 \\
    & \textbf{\sys}  & \textbf{70.8\footnotesize $\pm$0.2} & \textbf{69.5\footnotesize $\pm$0.2} & \textbf{66.2\footnotesize $\pm$0.1} & \textbf{63.1\footnotesize $\pm$0.3} & \textbf{56.1\footnotesize $\pm$0.2} \\
\bottomrule
\end{tabular}
}
\vspace{-\baselineskip}
\vspace{-\baselineskip}
\label{tab:acc_std}
\end{wraptable} 

\sys utilizes several tools and frameworks. Manifold projection is performed using \texttt{UMAP} \citep{umap-software}, \texttt{multipers} \citep{multipers} facilitates differential persistent homology which uses the \texttt{Gudhi} C++ library \citep{gudhi:RipsComplex} as a backend, and \texttt{DeepCore} \citep{deepcore} is used to standardize coreset selection and training across different methods.

To ensure fair comparisons in our experiments, we evaluate two versions of our framework. When benchmarking against other training-free methods, we use \textit{\sys (NLPS)}, which incorporates NLPS for mislabeled samples. When comparing against methods that require training-time information, we use \textit{\sys}, which incorporates the original AUM score. We compare \textbf{\sys (NLPS)} with several static geometry-based coreset selection methods: \textbf{Random} selection. \textbf{Moderate} \citep{moderateds} uses samples near the median distance to a class prototype (the barycenter of a point-mass distribution). \textbf{FDMat} \citep{fdmat} matches data distribution between dataset and coreset using optimal transport. We compare \textbf{\sys} with several geometry, score, and optimization-based methods that require training-time information: \textbf{Moderate (AUM)} incorporating Moderate with AUM-based mislabeled removal. \textbf{Forgetting} \citep{forgetting} uses the number of times an example is incorrectly classified after being correctly classified earlier during training. \textbf{Glister} \citep{glister} uses bi-level optimization. \textbf{LCMat-S} \citep{lcmats} matches loss curvature between dataset and coreset. \textbf{ELFS} \citep{elfs} computes training-dynamic difficulty scores from pseudo-labels obtained with the TEMI \citep{temi} deep clustering method. \textbf{CCS} \citep{ccs} uses stratified sampling of difficulty scores (such as AUM or Forgetting) with intra-strata random sampling. \textbf{FAST} \citep{cui2025fast} performs frequency-domain distribution matching regularized by the graph Laplacian of a $k$-NN graph, a notion of topology based on graph connectivity rather than the persistent homology used by \sys. \textbf{D2} \citep{d2} uses a message-passing graph network while also incorporating AUM for mislabeled samples. To isolate the selection criterion, ELFS and FAST use the same frozen encoder as the other geometric baselines (in place of the DINO or SwAV encoder used by ELFS) and receive the same mislabel filtering. Reported accuracies and standard deviations are computed over five independent training runs for the vision experiments and over three for the text experiments.

\subsection{Performant and Stable Coresets with \sys}
\label{sec:results_performance}

Our vision experiments, detailed in \cref{tab:acc_std}, demonstrate that \sys is competitive across CIFAR-10, CIFAR-100, and ImageNet-1K. Notably, this performance advantage scales with task difficulty. While competitive on simpler datasets, our method's dominance becomes most pronounced on the challenging ImageNet-1K benchmark and at extreme pruning rates (e.g., 90\%). Beyond accuracy, \sys exhibits notable precision gains. As evidenced by the standard deviation metrics in \cref{tab:acc_std} and validated via statistical testing (\cref{sec:appendix_statistical_significance}), our approach delivers better precision compared to extrinsic geometric baselines. For instance, on ImageNet-1K at 90\% pruning, \sys reduces variance by up to $6.5\times$ relative to the Euclidean graph-based D2. This precision ensures highly consistent data selection, demonstrating that our coresets faithfully represent the underlying distribution without requiring multiple costly trials. Finally, to explicitly confirm that these gains in accuracy and precision are driven fundamentally by our topological construction rather than the mislabel filtering, we provide a targeted component ablation in \cref{Sec:Appendix_filtering_ablation}.

\begin{table}[!tp]
\caption{\textbf{Text dataset pruning} (mean $\pm$ std over 3 seeds). Accuracy on ANLI, and held-out perplexity (PPL), ARC-Challenge (ArcC) and the average over MMLU and six other commonsense reasoning benchmarks (Avg.) after LoRA-tuning on Alpaca, with per-benchmark results in Appendix \cref{tab:nlp_full}. \sys is best on ANLI at every pruning rate and is most competitive on ARC-Challenge, the hardest commonsense reasoning benchmark.}
\setlength{\tabcolsep}{3.1pt}
\centering
\resizebox{\linewidth}{!}
{
\begin{tabular}{llcccccccc}
    \toprule
    & & \multicolumn{2}{c}{\textbf{ANLI}} & \multicolumn{6}{c}{\textbf{Alpaca \footnotesize(LoRA)}} \\
    \cmidrule(lr){3-4} \cmidrule(lr){5-10}
    & & RoBERTa-L & Qwen2.5-1.5B & \multicolumn{3}{c}{Llama-3.2-3B} & \multicolumn{3}{c}{Pythia-2.8B} \\
    \cmidrule(lr){3-3} \cmidrule(lr){4-4} \cmidrule(lr){5-7} \cmidrule(lr){8-10}
    Pruning Rate & Method & Acc. & Acc. & PPL $\downarrow$ & ArcC & Avg. & PPL $\downarrow$ & ArcC & Avg. \\
    \midrule
    \multirow{4}{*}{70\%} & Random & 55.1\footnotesize $\pm$1.9 & 59.2\footnotesize $\pm$0.9 & \underline{3.029\footnotesize $\pm$0.001} & 50.7\footnotesize $\pm$0.3 & 66.0\footnotesize $\pm$0.4 & \textbf{3.953\footnotesize $\pm$0.015} & 37.6\footnotesize $\pm$0.5 & 52.1\footnotesize $\pm$0.2 \\
    & Moderate & 56.0\footnotesize $\pm$0.9 & 56.8\footnotesize $\pm$0.7 & 3.041\footnotesize $\pm$0.004 & 51.1\footnotesize $\pm$0.7 & 65.8\footnotesize $\pm$0.3 & 3.960\footnotesize $\pm$0.026 & 37.7\footnotesize $\pm$0.6 & \underline{52.3\footnotesize $\pm$0.5} \\
    & D2 & 54.5\footnotesize $\pm$2.4 & 59.3\footnotesize $\pm$0.8 & 3.079\footnotesize $\pm$0.001 & 51.2\footnotesize $\pm$0.8 & \underline{66.2\footnotesize $\pm$0.2} & 3.990\footnotesize $\pm$0.013 & 37.3\footnotesize $\pm$1.0 & \underline{52.3\footnotesize $\pm$0.6} \\
    & \textbf{\sys} & \textbf{56.4\footnotesize $\pm$0.3} & \textbf{60.3\footnotesize $\pm$0.5} & \textbf{3.029\footnotesize $\pm$0.001} & \textbf{53.8\footnotesize $\pm$0.5} & \textbf{66.2\footnotesize $\pm$0.1} & 3.972\footnotesize $\pm$0.020 & \textbf{37.9\footnotesize $\pm$1.1} & \textbf{52.3\footnotesize $\pm$0.4} \\
    \midrule
    \multirow{4}{*}{80\%} & Random & 51.7\footnotesize $\pm$1.0 & 58.2\footnotesize $\pm$1.0 & 3.039\footnotesize $\pm$0.002 & 51.7\footnotesize $\pm$0.1 & 66.2\footnotesize $\pm$0.2 & 4.017\footnotesize $\pm$0.048 & 37.0\footnotesize $\pm$0.3 & 52.7\footnotesize $\pm$0.4 \\
    & Moderate & 52.1\footnotesize $\pm$1.6 & 55.6\footnotesize $\pm$0.7 & 3.048\footnotesize $\pm$0.001 & 51.8\footnotesize $\pm$0.6 & 65.9\footnotesize $\pm$0.1 & \textbf{3.977\footnotesize $\pm$0.017} & 37.6\footnotesize $\pm$0.9 & 52.9\footnotesize $\pm$0.1 \\
    & D2 & 52.6\footnotesize $\pm$1.8 & 59.5\footnotesize $\pm$0.9 & 3.048\footnotesize $\pm$0.006 & 51.3\footnotesize $\pm$1.0 & \textbf{66.6\footnotesize $\pm$0.2} & 4.018\footnotesize $\pm$0.035 & 36.9\footnotesize $\pm$0.4 & \textbf{53.1\footnotesize $\pm$0.1} \\
    & \textbf{\sys} & \textbf{52.9\footnotesize $\pm$0.8} & \textbf{60.2\footnotesize $\pm$0.5} & \textbf{3.036\footnotesize $\pm$0.001} & \textbf{52.4\footnotesize $\pm$0.4} & 66.4\footnotesize $\pm$0.2 & 3.996\footnotesize $\pm$0.015 & \textbf{38.1\footnotesize $\pm$0.5} & 52.9\footnotesize $\pm$0.3 \\
    \midrule
    \multirow{4}{*}{90\%} & Random & 47.1\footnotesize $\pm$0.5 & 56.0\footnotesize $\pm$1.5 & 3.055\footnotesize $\pm$0.001 & 50.6\footnotesize $\pm$0.9 & 66.5\footnotesize $\pm$0.3 & 4.044\footnotesize $\pm$0.015 & 36.6\footnotesize $\pm$0.4 & 52.8\footnotesize $\pm$0.3 \\
    & Moderate & 48.5\footnotesize $\pm$2.0 & 52.9\footnotesize $\pm$1.3 & 3.060\footnotesize $\pm$0.002 & 52.1\footnotesize $\pm$1.2 & \underline{66.6\footnotesize $\pm$0.5} & 4.068\footnotesize $\pm$0.024 & 37.3\footnotesize $\pm$1.2 & 52.9\footnotesize $\pm$0.6 \\
    & D2 & 47.4\footnotesize $\pm$1.2 & 54.8\footnotesize $\pm$3.7 & \textbf{3.051\footnotesize $\pm$0.003} & 51.5\footnotesize $\pm$1.1 & 66.5\footnotesize $\pm$0.4 & 4.079\footnotesize $\pm$0.013 & 37.1\footnotesize $\pm$0.8 & 53.2\footnotesize $\pm$0.2 \\
    & \textbf{\sys} & \textbf{49.1\footnotesize $\pm$1.1} & \textbf{58.5\footnotesize $\pm$0.6} & 3.053\footnotesize $\pm$0.006 & \textbf{52.6\footnotesize $\pm$0.5} & \textbf{66.6\footnotesize $\pm$0.2} & \textbf{4.038\footnotesize $\pm$0.022} & \textbf{37.9\footnotesize $\pm$0.7} & \textbf{53.3\footnotesize $\pm$0.7} \\
\bottomrule
\end{tabular}
}
\vspace{-\baselineskip}
\label{tab:nlp}
\end{table}

To evaluate broader applicability, we compare \sys against Random selection and geometry-based baselines, Moderate \citep{moderateds} and D2 \citep{d2}, on two text settings at 70\%, 80\% and 90\% pruning. (1) \textit{Classifier fine-tuning:} A linear classification head is trained for 10 epochs on top of a RoBERTa-Large \citep{roberta} or Qwen2.5-1.5B \citep{qwen25} backbone on ANLI \citep{anli} (3-way NLI), with embeddings for coreset construction from frozen \textit{[CLS]} of an NLI-fine-tuned  (RoBERTa-Large) model or the final-layer hidden state of the last token of a prompt-formatted premise--hypothesis pair under the base (Qwen) representations. (2) \textit{Instruction tuning:} LoRA \citep{lora} fine-tuning of Llama-3.2-3B \citep{llama3} and Pythia-2.8B \citep{pythia} on Alpaca \citep{alpaca} (rank 8, scaling factor 16, dropout 0.05, 3 epochs) using \textit{mean-pooled} representations. The scoring pipeline is mostly unchanged. The main difference is that Alpaca, unlike ANLI and our vision tasks, has no natural class labels, therefore we run k-means ($k=15$, fixed a priori and not tuned) on the frozen embeddings and treat the clusters as pseudo-classes. Mislabel filtering also does not apply, so coresets are built on the full dataset.
As shown in \cref{tab:nlp}, \sys achieves the highest ANLI accuracy at every pruning rate for both backbones. The gains over the strongest baseline are modest for RoBERTa-Large (0.4 / 0.3 / 0.6 points, within one standard deviation) and larger for Qwen2.5-1.5B (1.0 / 0.7 / 2.5 points), and \sys has the lowest standard deviation in five of the six settings. On Alpaca, reasoning benchmark averages are close across all methods, as most of the capability of a LoRA-tuned model at this scale is retained from pretraining. The clearest difference is on ARC-Challenge, the hardest benchmark, where \sys scores highest at all three rates for both models, by 0.2 to 2.6 points. We report the full held-out perplexity, 5-shot MMLU \citep{mmlu}, and six commonsense reasoning benchmarks via \texttt{lm-evaluation-harness} \citep{lmeval} in Appendix \cref{tab:nlp_full}. Together, these results suggest that the topological signal remains informative on text embeddings.

\subsection{Robustness to Noisy and Corrupted Representations}
\label{sec:results_noisy_features}

\begin{wraptable}{r}{0.5\linewidth}
    \vspace{-\baselineskip}
    \caption{\textbf{Robustness to noisy and corrupted representations at 90\% pruning rate.} In this extreme compression regime, where selection sensitivity is most pronounced, \sys maintains superior performance across all synthetic noise and realistic corruption types compared to geometric baselines.}
    \label{tab:robust_main}
    
    \centering
    \setlength{\tabcolsep}{3.0pt}
    \begin{subtable}{\linewidth}
        \centering
        \resizebox{\linewidth}{!}
        {
        \begin{tabular}{lcccc}
        \toprule
            Noise ($\rightarrow$) & $\mathcal{N}(0.25\sigma)$ & $\mathcal{N}(\sigma)$ & $\mathcal{N}(4\sigma)$ & $\mathcal{N}(8\sigma)$ \\
            \midrule
            Moderate & 33.2\footnotesize $\pm$0.9 & 33.9\footnotesize $\pm$0.2 & 32.0\footnotesize $\pm$1.2 & 32.1\footnotesize $\pm$1.4 \\
            D2 & 44.4\footnotesize $\pm$1.5 & 40.2\footnotesize $\pm$2.0 & 40.5\footnotesize $\pm$1.7 & 39.8\footnotesize $\pm$3.2 \\
            \textbf{\sys} & \textbf{45.4\footnotesize $\pm$0.8} & \textbf{45.5\footnotesize $\pm$0.6} & \textbf{46.1\footnotesize $\pm$0.7} & \textbf{43.9\footnotesize $\pm$0.4} \\
        \bottomrule
        \end{tabular}
        }
        \caption{Isotropic Latent Noise}
    \end{subtable}
    
    \begin{subtable}{\linewidth}
        \centering
        \resizebox{\linewidth}{!}
        {
        \begin{tabular}{lcccc}
        \toprule
            Noise ($\rightarrow$) & $\mathcal{N}(0.25\sigma)$ & $\mathcal{N}(\sigma)$ & $\mathcal{N}(4\sigma)$ & $\mathcal{N}(8\sigma)$ \\
            \midrule
            Moderate & 35.0\footnotesize $\pm$0.3 & 34.0\footnotesize $\pm$0.6 & 35.7\footnotesize $\pm$1.7 & 34.4\footnotesize $\pm$1.5 \\
            D2 & 40.4\footnotesize $\pm$1.5 & 43.4\footnotesize $\pm$0.9 & 41.3\footnotesize $\pm$2.2 & 42.6\footnotesize $\pm$2.5 \\
            \textbf{\sys} & \textbf{40.9\footnotesize $\pm$0.2} & \textbf{43.5\footnotesize $\pm$0.3} & \textbf{41.9\footnotesize $\pm$0.4} & \textbf{43.3\footnotesize $\pm$0.7} \\
        \bottomrule
        \end{tabular}
        }
        \caption{Input-level Noise}
    \end{subtable}

    \begin{subtable}{\linewidth}
        \centering
        \resizebox{\linewidth}{!}
        {
        \begin{tabular}{lcccc}
        \toprule
            Corr. ($\rightarrow$) & Contrast & Motion Blur & Frost & JPEG \\
            \midrule
            Moderate & 30.9\footnotesize $\pm$0.7 & 34.9\footnotesize $\pm$0.8 & 36.1\footnotesize $\pm$0.3 & 33.9\footnotesize $\pm$0.6 \\
            D2 & 43.2\footnotesize $\pm$2.2 & 41.3\footnotesize $\pm$1.2 & 43.2\footnotesize $\pm$1.4 & 40.9\footnotesize $\pm$1.1 \\
            \textbf{\sys} & \textbf{45.0\footnotesize $\pm$1.2} & \textbf{42.7\footnotesize $\pm$0.5} & \textbf{43.2\footnotesize $\pm$0.2} & \textbf{43.5\footnotesize $\pm$0.5}\\
        \bottomrule
        \end{tabular}
        }
        \caption{Image Corruptions}
    \end{subtable}
\end{wraptable}

To evaluate the robustness of our method against representation degradation, we compare \sys against geometry-based baselines, Moderate \cite{moderateds} and D2 \cite{d2} and evaluate three perturbation regimes on CIFAR-100. (1) \textit{Isotropic latent noise:} Gaussian noise injected into the penultimate layer embeddings. For each sample's feature vector $\mathbf{z} \in \mathbb{R}^D$ with standard deviation $\sigma_{\mathbf{z}}$, we create a perturbed version $\mathbf{z'} = \mathbf{z} + \epsilon$, where $\epsilon \sim \mathcal{N}(0, \sigma_{\mathbf{z}})$. (2) \textit{Input-level noise:} Pixel-level Gaussian perturbations applied to the images prior to feature extraction to observe how standard input noise propagates. (3) \textit{Realistic image corruptions:} To test structured, non-isotropic latent distortions, we utilize the CIFAR-C benchmark \cite{hendrycksbenchmarking}. We extract embeddings from images altered by four diverse corruption types (contrast, motion blur, frost, and JPEG compression) at severity level 3. Coresets are selected using these corrupted embeddings, but the resulting models are trained and evaluated on clean data. This serves to test whether topological sample importance is preserved under realistic degradation.

As shown in \cref{tab:robust_main}, \sys consistently demonstrates superior resilience across all three perturbation types at the challenging 90\% pruning rate, confirming that \sys identifies a high-fidelity coreset even when the underlying feature space is distorted. Detailed quantitative results for all pruning rates are provided in \cref{tab:noise_ablation}, \cref{tab:input_noise_ablation}, and \cref{tab:corruption_ablation} in the Appendix.

\paragraph{Implications.} The preservation of topological importance under both synthetic noise and realistic, structured corruptions indicates our method is not overly dependent on a perfectly optimized source model. This suggests that effective coresets could be generated using embeddings from models that are partially trained, quantized for edge devices, or applied to slightly out-of-distribution data as similarly shown in \cite{ph3}.

\subsection{Transferability \& Unified Representations Across Architectures}
\label{sec:results_arch_transfer}

\begin{figure}[!tp]
  \centering
  \begin{minipage}[b]{0.45\textwidth}
    \centering
    \includegraphics[width=\linewidth]{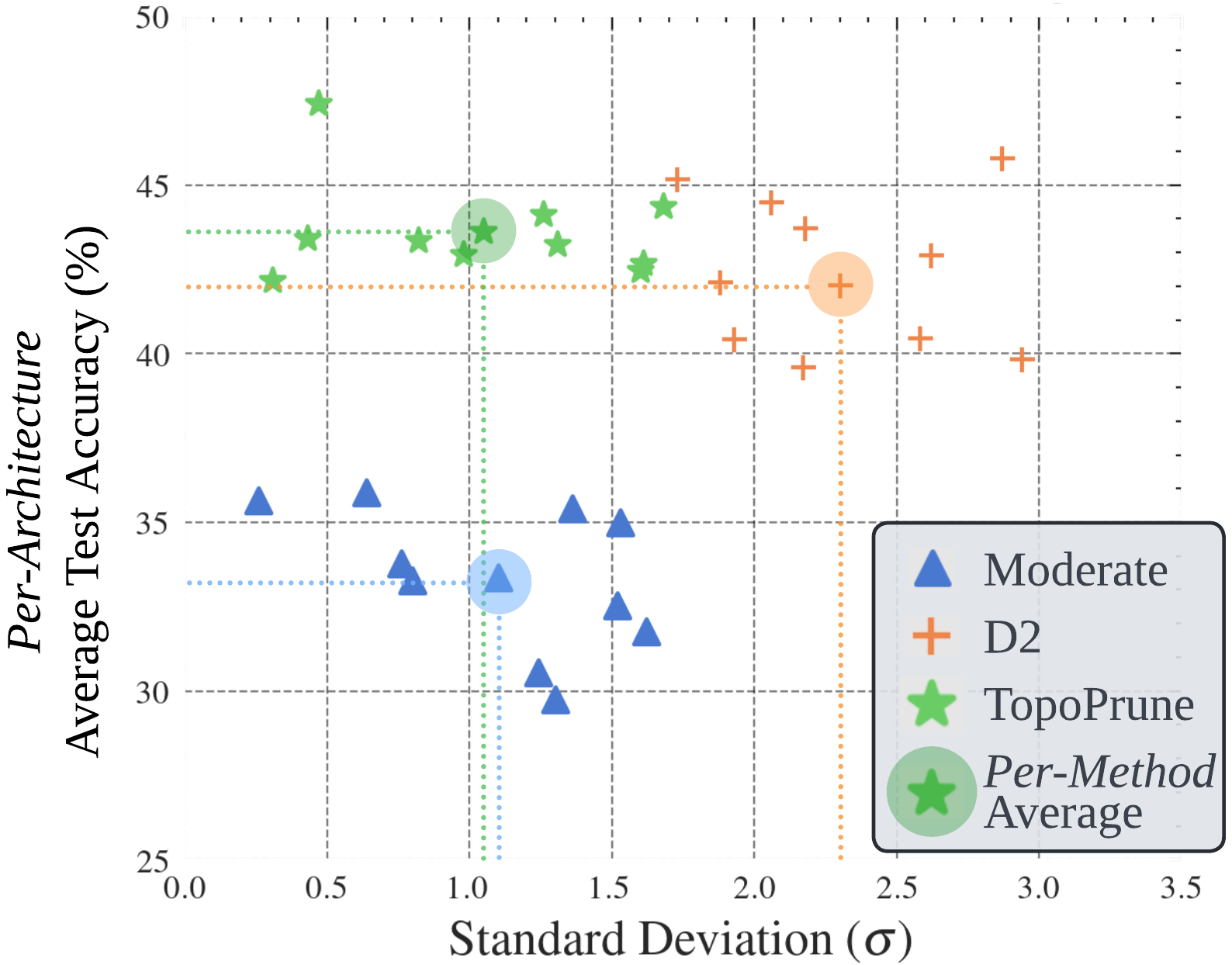}
    \caption{\textbf{Transferability of \textit{Diverse (larger) Embeddings} $\rightarrow$ \textit{Fixed (smaller) Target}.} \sys achieves higher mean accuracy and lower standard deviation across 10 architectures at a high pruning rate of 90\% for CIFAR-100, where top left (low standard deviation and high accuracy) is best.}
    \label{fig:diverse_to_fixed}
  \end{minipage}
  \hfill 
  \begin{minipage}[b]{0.53\textwidth}
    \setlength{\tabcolsep}{0.5pt}
    \centering
    \resizebox{\linewidth}{!}{
    \begin{tabular}{l c c c @{\hskip 5pt} c c c}
        \toprule
        Prun. Rate ($\rightarrow$) & \multicolumn{3}{c}{80\%} & \multicolumn{3}{c}{90\%}\\
        \cmidrule(lr){2-4} \cmidrule(lr){5-7}
        \textbf{CIFAR-100} & Oracle & ResNet-18 & $\Delta$ & Oracle & ResNet-18 & $\Delta$ \\
        \midrule
        ResNet-50 & 57.0\footnotesize $\pm$0.6 & 56.1\footnotesize $\pm$0.7 & \textbf{-0.9} & 38.7\footnotesize $\pm$1.4 & 39.6\footnotesize $\pm$1.6 & \textbf{+0.9} \\
        EffNet-B0 & 55.0\footnotesize $\pm$2.1 & 54.4\footnotesize $\pm$1.4 & \textbf{-0.6} & 39.8\footnotesize $\pm$1.9 & 39.7\footnotesize $\pm$3.3 & \textbf{-0.1} \\
        SwinV2-T & 38.1\footnotesize $\pm$0.7 & 39.9\footnotesize $\pm$0.9 & \textbf{+1.8} & 27.6\footnotesize $\pm$1.2 & 29.6\footnotesize $\pm$1.5 & \textbf{+2.0} \\
    \bottomrule
    \end{tabular}}

    \resizebox{\linewidth}{!}{
    \begin{tabular}{l c c c @{\hskip 5pt} c c c}
        \toprule
        Prun. Rate ($\rightarrow$) & \multicolumn{3}{c}{80\%} & \multicolumn{3}{c}{90\%}\\
        \cmidrule(lr){2-4} \cmidrule(lr){5-7}
        \textbf{ImageNet-1K} & Oracle & ResNet-50 & $\Delta$ & Oracle & ResNet-50 & $\Delta$ \\
        \midrule
        EffNetV2-M & 39.1\footnotesize\footnotesize $\pm$1.4 & 40.8\footnotesize $\pm$0.3 & \textbf{+1.7}  & 35.9\footnotesize $\pm$0.3 & 37.1\footnotesize $\pm$1.3 & \textbf{+1.2}  \\
        SwinV2-T & 57.8\footnotesize $\pm$0.7 & 59.0\footnotesize $\pm$1.1 & \textbf{+1.2} & 38.3\footnotesize $\pm$3.4 & 38.0\footnotesize $\pm$4.1 & \textbf{-0.3} \\
        SwinV2-B & 44.7\footnotesize $\pm$1.2 & 45.8\footnotesize $\pm$2.1 & \textbf{+1.1} & 44.4\footnotesize $\pm$2.1 & 45.1\footnotesize $\pm$1.1 & \textbf{+0.7}  \\
    \bottomrule
    \end{tabular}}
    \captionof{table}{\textbf{Transferability of \textit{Fixed (smaller) Embedding $\rightarrow$ Diverse (larger) Targets}}. We compare the "Oracle" performance (coreset selected from \sys using the target model's own embeddings) against coreset selected from \sys using smaller proxy model embeddings (ResNet-18 for CIFAR-100 and ResNet-50 for ImageNet-1k) to train a diverse range of models.}
    \label{tab:arch_ablation_2}
  \end{minipage}
\end{figure}

We investigate the limitations of geometric metrics across a wide range of network architectures, finding that metric stability increases dramatically as we move toward topology-based metrics. (1) \textit{Metric Stability}: First, we analyze the Euclidean distance of samples to their class prototype. This metric proves highly inconsistent across architectures, demonstrating ``geometric brittleness'' (\cref{fig:proxy_networks_distance}). The global Density Score (\cref{fig:proxy_networks_density}) improves uniformity, while the local Persistence Score (\cref{fig:proxy_networks_persistence}) aligns almost perfectly across all tested architectures. (2) \textit{Diverse (larger) Embeddings $\rightarrow$ Fixed (smaller) Target}: Leveraging this metric stability, we evaluate transferability from diverse feature embeddings to a single target model. As detailed in \cref{fig:diverse_to_fixed}, \sys consistently yields higher accuracy and lower standard deviation across diverse proxies (ResNet, EfficientNet, Swin, ViT, OpenCLIP) to train a fixed ResNet-18 target, significantly outperforming geometric baselines. Please see Appendix \cref{tab:arch_ablation} for detailed values and \cref{Sec:Appendix_euclidean_topology} for further justification. (3) \textit{Fixed (smaller) Embedding $\rightarrow$ Diverse (larger) Targets}: Finally, we validate transferring from a static feature embedding to diverse target models. Using a standard ResNet proxy to select coresets for training EfficientNet and Swin Transformers, we find that \sys achieves performance competitive with, and in some cases exceeding, ``Oracle'' selection where the target model selects its own coreset (see \cref{tab:arch_ablation_2}). This confirms that topological importance derived from standard proxy embeddings is highly generalizable, allowing selection of a single, high-quality coreset effective for a wide range of downstream architectures.

\paragraph{Implications.} Recent literature spanning representational similarity \citep{kriegeskorte2008rsa}, model stitching \citep{lenc2015understanding, bansal2021revisiting}, and topological divergence \citep{barannikov2022rtd} establishes that as models scale, their latent spaces converge toward a shared statistical model of reality, formalized as the Platonic Representation Hypothesis \citep{platonic1}. The recent ``Aristotelian'' refinement \citep{platonic2} specifies that alignment is strictly anchored in \textit{local neighborhood connectivity}, while global geometry remain highly variable across architectures.

\sys builds on this insight. Rather than relying on brittle global geometry, we extract a topology-aware projection coupled with persistent homology (both grounded in topological connectivity rather than raw distances), the very signal the Aristotelian refinement identifies as preserved. Our metric stability analysis (\cref{fig:proxy_networks_all}) validates this: Euclidean-to-prototype distributions vary wildly, while our topology-aware density prior is more uniform and local persistence achieves near-perfect uniformity. This lets us access the aligned component of representations without requiring full geometric alignment, improving transferability for two regimes: (1) efficient scaling using small proxy embeddings to select data for larger targets (\cref{tab:arch_ablation_2}) and (2) transferable coresets from a single pretrained model (\cref{fig:diverse_to_fixed}).

\section{Conclusion}
\label{Sec:Conclusion}

In this work, we address the critical challenge of instability in geometric coreset selection methods, which arises from their reliance on extrinsic metrics. We present \sys, a novel framework that overcomes this ``geometric brittleness'' by leveraging topology to capture the data's intrinsic structure. Our dual-scale topological approach combines a global topology-aware manifold projection with a local importance score derived from differentiable persistent homology. \sys exhibits several key advantages: (1) yields coresets with higher accuracy and precision, (2) is resilient to representation degradation, withstanding latent noise, input-space perturbations, realistic image corruptions and (3) is more stable across a wide range of network architectures and transfer directions. By grounding data selection in the stable invariants of manifold density and persistent homology, \sys provides a principled, resilient foundation necessary for data-efficient learning.
\section*{Reproducibility Statement}
\label{Sec:Reproducibility}

We are committed to reproducible research and prioritize transparency in our work. The Methodology section (\cref{Sec:Methodology}) provides a comprehensive description of our framework, including the global manifold projection, differentiable persistent homology optimization, and unified scoring formulation. The Experimental Setup subsection (\cref{sec:experimental_setup}) details the datasets, baselines, libraries (\texttt{UMAP}, \texttt{multipers}, \texttt{DeepCore}), and hardware environments used. Complete pseudocode for TopoPrune is provided in \cref{Sec:Appendix_pseudocode}, and all training, manifold projection, and persistent homology hyperparameters are documented in Appendix \cref{table:hyperparameters}. Our code is also available at \url{https://github.com/ar-roy/topoprune}.
\section*{Limitations}

While \sys delivers strong accuracy, precision, and transferability gains, we highlight some limitations and suggest directions for future work. First, our topological backend currently lacks GPU acceleration and multi-processing support, resulting in higher wall-clock latency than non-topological baselines (\cref{Sec:appendix_complexity_analysis}). We expect maturing topological software infrastructure to help close this gap. Also, while we provide background for the transferability of topological metrics (\cref{Sec:Appendix_euclidean_topology}), the relationship between latent-space topology and downstream training dynamics remains an open question facing the broader topological data analysis community.
\section*{Impact Statement}
\label{Sec:Impact}

This work aims to reduce the cost of training deep networks by selecting smaller, representative training subsets. Its main positive impact is lower compute and energy use for training and fine-tuning, which can make experimentation more accessible to groups with limited resources. These savings are partly offset by the cost of selection itself, so the net benefit grows with the cost of the downstream training it replaces. \sys produces highly consistent coresets across training runs and transfers across architectures which also makes training pipelines easier to reproduce and audit.

As with any data pruning method, removing data can change what a model learns about rare groups, as samples from underrepresented subpopulations often lie in sparse regions of the embedding. We do not test \sys on subgroup populations or with fairness metrics, so practitioners applying \sys to data involving people should check the composition of the selected coreset and the resulting model's performance across relevant subgroups.

\section*{Acknowledgments}

This project was supported in part by the Purdue Center for Secure Microelectronics Ecosystem – CSME\#210205 and the Center for the Co-Design of Cognitive Systems (CoCoSys), a DARPA-sponsored JUMP 2.0 center.

\bibliography{main.bib}

@inproceedings{moderateds,
    title={Moderate Coreset: A Universal Method of Data Selection for Real-world Data-efficient Deep Learning},
    author={Xiaobo Xia and Jiale Liu and Jun Yu and Xu Shen and Bo Han and Tongliang Liu},
    booktitle={International Conference on Learning Representations},
    year={2023},
}

@inproceedings{ccs,
    title={Coverage-centric Coreset Selection for High Pruning Rates},
    author={Haizhong Zheng and Rui Liu and Fan Lai and Atul Prakash},
    booktitle={International Conference on Learning Representations},
    year={2023},
}

@article{fdmat, 
    title={Feature Distribution Matching by Optimal Transport for Effective and Robust Coreset Selection}, 
    journal={AAAI Conference on Artificial Intelligence}, 
    author={Xiao, Weiwei and Chen, Yongyong and Shan, Qiben and Wang, Yaowei and Su, Jingyong}, 
    year={2024}
}

@inproceedings{d2,
    title={D2 Pruning: Message Passing for Balancing Diversity \& Difficulty in Data Pruning},
    author={Adyasha Maharana and Prateek Yadav and Mohit Bansal},
    booktitle={International Conference on Learning Representations},
    year={2024},
}

@InProceedings{mindboundary,
  title = 	 {Mind the Boundary: Coreset Selection via Reconstructing the Decision Boundary},
  author =       {Yang, Shuo and Cao, Zhe and Guo, Sheng and Zhang, Ruiheng and Luo, Ping and Zhang, Shengping and Nie, Liqiang},
  booktitle = 	 {International Conference on Machine Learning},
  year = 	 {2024},
}

@inproceedings{fairwass,
    title={Fair Wasserstein Coresets},
    author={Zikai Xiong and Niccolo Dalmasso and Shubham Sharma and Freddy Lecue and Daniele Magazzeni and Vamsi K. Potluru and Tucker Balch and Manuela Veloso},
    booktitle={Advances in Neural Information Processing Systems},
    year={2024},
}

@inproceedings{ses,
    title={Structural-Entropy-Based Sample Selection for Efficient and Effective Learning},
    author={Tianchi Xie and Jiangning Zhu and Guozu Ma and Minzhi Lin and Wei Chen and Weikai Yang and Shixia Liu},
    booktitle={International Conference on Learning Representations},
    year={2025},
}

@ARTICLE{bluenoise,
    author={Chen, Haidong and Chen, Wei and Mei, Honghui and Liu, Zhiqi and Zhou, Kun and Chen, Weifeng and Gu, Wentao and Ma, Kwan-Liu},
    journal={IEEE Transactions on Visualization \& Computer Graphics },
    title={Visual Abstraction and Exploration of Multi-class Scatterplots},
    year={2014},
}

@article{cui2025fast,
  title={FAST: Topology-Aware Frequency-Domain Distribution Matching for Coreset Selection},
  author={Cui, Jin and Zhao, Boran and Xu, Jiajun and Guo, Jiaqi and Guan, Shuo and Ren, Pengju},
  journal={IEEE/CVF Conference on Computer Vision and Pattern Recognition},
  year={2026}
}

@article{grad1,
  title={Coresets via Bilevel Optimization for Continual Learning and Streaming},
  author={Borsos, Zal{\'a}n and Mutny, Mojmir and Krause, Andreas},
  journal={Advances in Neural Information Processing Systems},
  year={2020}
}

@inproceedings{grad2,
title={Data Pruning via Moving-one-Sample-out},
author={Haoru Tan and Sitong Wu and Fei Du and Yukang Chen and Zhibin Wang and Fan Wang and Xiaojuan Qi},
booktitle={Advances in Neural Information Processing Systems},
year={2023},
}

@inproceedings{gradmatch,
  title={Grad-match: Gradient matching based data subset selection for efficient deep model training},
  author={Killamsetty, Krishnateja and Ramakrishnan, Ganesh and De, Abir and Iyer, Rishabh},
  booktitle={International Conference on Machine Learning},
  year={2021},
}

@inproceedings{glister,
  title={Glister: Generalization based data subset selection for efficient and robust learning},
  author={Killamsetty, Krishnateja and Sivasubramanian, Durga and Ramakrishnan, Ganesh and Iyer, Rishabh},
  booktitle={AAAI Conference on Artificial Intelligence},
  year={2021}
}

@inproceedings{grad4,
  title={Adaptive second order coresets for data-efficient machine learning},
  author={Pooladzandi, Omead and Davini, David and Mirzasoleiman, Baharan},
  booktitle={International Conference on Machine Learning},
  year={2022},
}

@inproceedings{grad5,
  title={Prioritized training on points that are learnable, worth learning, and not yet learnt},
  author={Mindermann, S{\"o}ren and Brauner, Jan M and Razzak, Muhammed T and Sharma, Mrinank and Kirsch, Andreas and Xu, Winnie and H{\"o}ltgen, Benedikt and Gomez, Aidan N and Morisot, Adrien and Farquhar, Sebastian and others},
  booktitle={International Conference on Machine Learning},
  year={2022},
}

@inproceedings{craig,
  title={Coresets for Data-efficient Training of Machine Learning Models},
  author={Baharan Mirzasoleiman and Jeff A. Bilmes and Jure Leskovec},
  booktitle={International Conference on Machine Learning},
  year={2019},
}

@inproceedings{dual,
    title={Lightweight Dataset Pruning without Full Training via Example Difficulty and Prediction Uncertainty},
    author={Yeseul Cho and Baekrok Shin and Changmin Kang and Chulhee Yun},
    booktitle={International Conference on Machine Learning},
    year={2025},
}

@inproceedings{lcmats,
  title={Loss-curvature matching for dataset selection and condensation},
  author={Shin, Seungjae and Bae, Heesun and Shin, Donghyeok and Joo, Weonyoung and Moon, Il-Chul},
  booktitle={International Conference on Artificial Intelligence and Statistics},
  year={2023},
}

@article{aum,
  title={Identifying mislabeled data using the area under the margin ranking},
  author={Pleiss, Geoff and Zhang, Tianyi and Elenberg, Ethan and Weinberger, Kilian Q},
  journal={Advances in Neural Information Processing Systems},
  year={2020}
}

@article{nagaraj2025coresets,
    title={Coresets from Trajectories: Selecting Data via Correlation of Loss Differences},
    author={Manish Nagaraj and Deepak Ravikumar and Kaushik Roy},
    journal={Transactions on Machine Learning Research},
    year={2025},
}

@inproceedings{sener2018active,
    title={Active Learning for Convolutional Neural Networks: A Core-Set Approach},
    author={Ozan Sener and Silvio Savarese},
    booktitle={International Conference on Learning Representations},
    year={2018},
}

@InProceedings{slocurve,
    author    = {Garg, Isha and Roy, Kaushik},
    title     = {Samples With Low Loss Curvature Improve Data Efficiency},
    booktitle = {IEEE/CVF Conference on Computer Vision and Pattern Recognition},
    year      = {2023},
}

@inproceedings{elfs,
    title={{ELFS}: Label-Free Coreset Selection with Proxy Training Dynamics},
    author={Haizhong Zheng and Elisa Tsai and Yifu Lu and Jiachen Sun and Brian R. Bartoldson and Bhavya Kailkhura and Atul Prakash},
    booktitle={International Conference on Learning Representations},
    year={2025},
}

@article{temi,
  title={Exploring the limits of deep image clustering using pretrained models},
  author={Adaloglou, Nikolas and Michels, Felix and Kalisch, Hamza and Kollmann, Markus},
  journal={British Machine Vision Conference},
  year={2023}
}

@inproceedings{score1,
    title={You Only Condense Once: Two Rules for Pruning Condensed Datasets},
    author={Yang He and Lingao Xiao and Joey Tianyi Zhou},
    booktitle={Advances in Neural Information Processing Systems},
    year={2023},
}

@inproceedings{score2,
    title={Large-scale dataset pruning with dynamic uncertainty},
    author={He, Muyang and Yang, Shuo and Huang, Tiejun and Zhao, Bo},
    booktitle={IEEE/CVF Conference on Computer Vision and Pattern Recognition},
    year={2024}
}

@article{el2n,
    title={Deep learning on a data diet: Finding important examples early in training},
    author={Paul, Mansheej and Ganguli, Surya and Dziugaite, Gintare Karolina},
    journal={Advances in Neural Information Processing Systems},
    year={2021}
}

@inproceedings{forgetting,
    title={An Empirical Study of Example Forgetting during Deep Neural Network Learning},
    author={Mariya Toneva and Alessandro Sordoni and Remi Tachet des Combes and Adam Trischler and Yoshua Bengio and Geoffrey J. Gordon},
    booktitle={International Conference on Learning Representations},
    year={2019},
}

@article{score3,
  title={Dataset cartography: Mapping and diagnosing datasets with training dynamics},
  author={Swayamdipta, Swabha and Schwartz, Roy and Lourie, Nicholas and Wang, Yizhong and Hajishirzi, Hannaneh and Smith, Noah A and Choi, Yejin},
  journal={ArXiv},
  year={2020}
}

@inproceedings{proxy0,
    title={Selection via Proxy: Efficient Data Selection for Deep Learning},
    author={Cody Coleman and Christopher Yeh and Stephen Mussmann and Baharan Mirzasoleiman and Peter Bailis and Percy Liang and Jure Leskovec and Matei Zaharia},
    booktitle={International Conference on Learning Representations},
    year={2020},
}

@inproceedings{scaling2,
    title={Beyond neural scaling laws: beating power law scaling via data pruning},
    author={Ben Sorscher and Robert Geirhos and Shashank Shekhar and Surya Ganguli and Ari S. Morcos},
    booktitle={Advances in Neural Information Processing Systems},
    editor={Alice H. Oh and Alekh Agarwal and Danielle Belgrave and Kyunghyun Cho},
    year={2022},
}

@inproceedings{platonic1,
  title={Position: The Platonic Representation Hypothesis},
  author={Minyoung Huh and Brian Cheung and Tongzhou Wang and Phillip Isola},
  booktitle={International Conference on Machine Learning},
  year={2024},
}

@article{platonic2,
  title={Revisiting the platonic representation hypothesis: An aristotelian view},
  author={Gr{\"o}ger, Fabian and Wen, Shuo and Brbi{\'c}, Maria},
  journal={International Conference on Machine Learning},
  year={2026}
}

@article{kriegeskorte2008rsa,
  title={Representational similarity analysis -- connecting the branches of systems neuroscience},
  author={Kriegeskorte, Nikolaus and Mur, Marieke and Bandettini, Peter},
  journal={Frontiers in Systems Neuroscience},
  year={2008}
}

@inproceedings{lenc2015understanding,
  title={Understanding image representations by measuring their equivariance and equivalence},
  author={Lenc, Karel and Vedaldi, Andrea},
  booktitle={IEEE/CVF Conference on Computer Vision and Pattern Recognition},
  year={2015}
}

@inproceedings{bansal2021revisiting,
  title={Revisiting model stitching to compare neural representations},
  author={Bansal, Yamini and Nakkiran, Preetum and Barak, Boaz},
  booktitle={Advances in Neural Information Processing Systems},
  year={2021}
}

@inproceedings{barannikov2022rtd,
  title={Representation Topology Divergence: A method for comparing neural network representations},
  author={Barannikov, Serguei and Trofimov, Ilya and Balabin, Nikita and Burnaev, Evgeny},
  booktitle={International Conference on Machine Learning},
  year={2022}
}

@article{umap,
  title={Umap: Uniform manifold approximation and projection for dimension reduction},
  author={McInnes, Leland and Healy, John and Melville, James},
  journal={ArXiv},
  year={2018}
}

@article{pacmap,
  title={Understanding how dimension reduction tools work: an empirical approach to deciphering t-SNE, UMAP, TriMAP, and PaCMAP for data visualization},
  author={Wang, Yingfan and Huang, Haiyang and Rudin, Cynthia and Shaposhnik, Yaron},
  journal={Journal of Machine Learning Research},
  year={2021}
}

@article {densmap,
	title = {Assessing single-cell transcriptomic variability through density-preserving data visualization},
    author = {Narayan, Ashwin and Berger, Bonnie and Cho, Hyunghoon},
	year = {2021},
	journal = {Nature Biotechnology}
}

@article{tsne,
  author  = {Laurens van der Maaten and Geoffrey Hinton},
  title   = {Visualizing Data using t-SNE},
  journal = {Journal of Machine Learning Research},
  year    = {2008},
}

@article{pca,
    author = {Karl Pearson},
    title = {On lines and planes of closest fit to systems of points in space},
    journal = {The London, Edinburgh, and Dublin Philosophical Magazine and Journal of Science},
    year = {1901},
}

@article{man1,
  title={Dimensionality reduction for visualizing single-cell data using UMAP},
  author={Etienne Becht and Leland McInnes and John Healy and Charles-Antoine Dutertre and Immanuel Kwok and Lai Guan Ng and Florent Ginhoux and Evan William Newell},
  journal={Nature Biotechnology},
  year={2018},
}

@inproceedings{man3,
    title={A Holistic Approach to Unifying Automatic Concept Extraction and Concept Importance Estimation},
    author={Thomas Fel and Victor Boutin and Louis B{\'e}thune and Remi Cadene and Mazda Moayeri and L{\'e}o And{\'e}ol and Mathieu Chalvidal and Thomas Serre},
    booktitle={Advances in Neural Information Processing Systems},
    year={2023},
}

@inproceedings{man4,
    title={Understanding Visual Feature Reliance through the Lens of Complexity},
    author={Thomas Fel and Louis B{\'e}thune and Andrew Kyle Lampinen and Thomas Serre and Katherine Hermann},
    booktitle={Advances in Neural Information Processing Systems},
    year={2024},
}

@book{topo0,
  author    = {Seifert, Herbert and Threlfall, W.},
  title     = {A textbook of topology},
  publisher = {Academic Press},
  year      = {1980},
  isbn      = {0-12-634850-2},
  notes     = {Translation of: Lehrbuch der Topologie},
}

@article{topo2,
  title={Low-dimensional embeddings of high-dimensional data},
  author={de Bodt, Cyril and Diaz-Papkovich, Alex and Bleher, Michael and Bunte, Kerstin and Coupette, Corinna and Damrich, Sebastian and Sanmartin, Enrique Fita and Hamprecht, Fred A and Horv{\'a}t, Em{\H{o}}ke-{\'A}gnes and Kohli, Dhruv and others},
  journal={ArXiv},
  year={2025}
}

@inproceedings{topoae,
  title={Topological autoencoders},
  author={Moor, Michael and Horn, Max and Rieck, Bastian and Borgwardt, Karsten},
  booktitle={International Conference on Machine Learning},
  year={2020},
}

@inproceedings{rtd,
    title={Learning topology-preserving data representations},
    author={Ilya Trofimov and Daniil Cherniavskii and Eduard Tulchinskii and Nikita Balabin and Evgeny Burnaev and Serguei Barannikov},
    booktitle={International Conference on Learning Representations},
    year={2023},
}

@inproceedings{kim2024topological,
    title={Do Topological Characteristics Help in Knowledge Distillation?},
    author={Jungeun Kim and Junwon You and Dongjin Lee and Ha Young Kim and Jae-Hun Jung},
    booktitle={International Conference on Machine Learning},
    year={2024},
}

@inproceedings{zhang2024homology,
    title={Homology Consistency Constrained Efficient Tuning for Vision-Language Models},
    author={Huatian Zhang and Lei Zhang and Yongdong Zhang and Zhendong Mao},
    booktitle={Advances in Neural Information Processing Systems},
    year={2024},
}

@inproceedings{fan2024persistence,
    title={Persistence Homology Distillation for Semi-supervised Continual Learning},
    author={Yan Fan and Yu Wang and Pengfei Zhu and Dongyue Chen and Qinghua Hu},
    booktitle={Advances in Neural Information Processing Systems},
    year={2024},
}

@inproceedings{you2026topological,
    title={Topological Alignment of Shared Vision-Language Embedding Space},
    author={Junwon You and Dasol Kang and Jae-Hun Jung},
    booktitle={International Conference on Artificial Intelligence and Statistics},
    year={2026},
}

@inproceedings{ph1,
  title={Stability of persistence diagrams},
  author={Cohen-Steiner, David and Edelsbrunner, Herbert and Harer, John},
  booktitle={Symposium on Computational Geometry},
  year={2005}
}

@article{ph1.1,
  title={The simplex tree: An efficient data structure for general simplicial complexes},
  author={Boissonnat, Jean-Daniel and Maria, Cl{\'e}ment},
  journal={Algorithmica},
  year={2014},
}

@article{ph1.2,
  title={An introduction to multiparameter persistence},
  author={Botnan, Magnus Bakke and Lesnick, Michael},
  journal={ArXiv},
  year={2022},
}

@inproceedings{ph2,
    title={Intrinsic Dimension, Persistent Homology and Generalization in Neural Networks},
    author={Tolga Birdal and Aaron Lou and Leonidas Guibas and Umut Simsekli},
    booktitle={Advances in Neural Information Processing Systems},
    year={2021},
}

@article{ph3,
  title={On the effectiveness of persistent homology},
  author={Turkes, Renata and Montufar, Guido F and Otter, Nina},
  journal={Advances in Neural Information Processing Systems},
  year={2022},
  notes={* PH features allow to solve the above tasks even in the case of limited training data (Section 3), noisy (Section 3) and out-of-distribution (Section 5) test data, and limited computational resources}
}

@article{ph4,
  title={On characterizing the evolution of embedding space of neural networks using algebraic topology},
  author={Suresh, Suryaka and Das, Bishshoy and Abrol, Vinayak and Roy, S Dutta},
  journal={Pattern Recognition Letters},
  year={2024},
  notes={* Investigates the topological structure of latent embeddings over training by looking at Betti numbers}
}

@inproceedings{ph5,
    title={A Framework for Fast and Stable Representations of Multiparameter Persistent Homology Decompositions},
    author={David Loiseaux and Mathieu Carri{\`e}re and Andrew Blumberg},
    booktitle={Advances in Neural Information Processing Systems},
    year={2023},
}

@article{ph7,
  title={Stability and machine learning applications of persistent homology using the Delaunay-Rips complex},
  author={Mishra, Amish and Motta, Francis C},
  journal={Frontiers in Applied Mathematics and Statistics},
  year={2023},
}

@article{beyondeuclid,
  title={Beyond Euclid: an illustrated guide to modern machine learning with geometric, topological, and algebraic structures},
  author={Papillon, Mathilde and Sanborn, Sophia and Mathe, Johan and Cornelis, Louisa and Bertics, Abby and Buracas, Domas and Lillemark, Hansen J and Shewmake, Christian and Dinc, Fatih and Pennec, Xavier and Miolane, Nina},
  journal={Machine Learning: Science and Technology},
  year={2025},
}

@article{ph9,
  title={A roadmap for the computation of persistent homology},
  author={Otter, Nina and Porter, Mason A and Tillmann, Ulrike and Grindrod, Peter and Harrington, Heather A},
  journal={EPJ Data Science},
  year={2017},
}

@inproceedings{ph11,
  title={Optimizing persistent homology based functions},
  author={Carri{\`e}re, Mathieu and Chazal, Fr{\'e}d{\'e}ric and Glisse, Marc and Ike, Yuichi and Kannan, Hariprasad and Umeda, Yuhei},
  booktitle={International Conference on Machine Learning},
  year={2021},
}

@article{ph13,
  author  = {Gregory Naitzat and Andrey Zhitnikov and Lek-Heng Lim},
  title   = {Topology of Deep Neural Networks},
  journal = {Journal of Machine Learning Research},
  year    = {2020},
}

@article{ph14,
  title={Stable vectorization of multiparameter persistent homology using signed barcodes as measures},
  author={Loiseaux, David and Scoccola, Luis and Carri{\`e}re, Mathieu and Botnan, Magnus Bakke and Oudot, Steve},
  journal={Advances in Neural Information Processing Systems},
  year={2023}
}

@article{ph15,
  title={Computing minimal presentations and bigraded betti numbers of 2-parameter persistent homology},
  author={Lesnick, Michael and Wright, Matthew},
  journal={SIAM Journal on Applied Algebra and Geometry},
  year={2022},
}

@article{ph16,
  title={Keeping it sparse: Computing persistent homology revisited},
  author={Bauer, Ulrich and Masood, Talha Bin and Giunti, Barbara and Houry, Guillaume and Kerber, Michael and Rathod, Abhishek},
  journal={ArXiv},
  year={2022}
}

@article{ph17,
    title = {Delaunay Bifiltrations of Functions on Point Clouds},
    author = {Alonso, {\`A}ngel and Kerber, Michael and Lam, Tung and Lesnick, Michael},
    year = {2024},
    journal = {ACM-SIAM Symposium on Discrete Algorithms}
}

@article{ph18,
    title = {Topological Persistence and Simplification},
    author = {Edelsbrunner, Herbert and Letscher, David and Zomorodian, Afra},
    year = {2002},
    journal = {Discrete and Computational Geometry}
}

@article{ph19,
    title = {Computing Persistent Homology},
    author = {Afra Zomorodian and Gunnar Carlsson},
    year = {2005},
    journal = {Discrete and Computational Geometry}
}

@article{ph21,
    title = {The Theory of Multidimensional Persistence},
    author = {Carlsson, Gunnar and Zomorodian, Afra},
    year = {2007},
    journal = {Discrete and Computational Geometry},
}

@inproceedings{ph22,
 author = {Carri{\`e}re, Mathieu and Blumberg, Andrew},
 booktitle = {Advances in Neural Information Processing Systems},
 title = {Multiparameter Persistence Image for Topological Machine Learning},
 year = {2020}
}

@inproceedings{ph23,
    title={The Shape of Adversarial Influence: Characterizing {LLM} Latent Spaces with Persistent Homology},
    author={Aideen Fay and In{\'e}s Garc{\'\i}a-Redondo and Qiquan Wang and Haim Dubossarsky and Anthea Monod},
    booktitle={International Conference on Learning Representations},
    year={2026},
}

@inproceedings{dif_ph1,
    title={Differentiability and Optimization of Multiparameter Persistent Homology},
    author={Luis Scoccola and Siddharth Setlur and David Loiseaux and Mathieu Carri{\`e}re and Steve Oudot},
    booktitle={International Conference on Machine Learning},
    year={2024},
}

@inproceedings{dif_ph2,
 author = {Carri\`{e}re, Mathieu and Theveneau, Marc and Lacombe, Th\'{e}o},
 booktitle = {Advances in Neural Information Processing Systems},
 title = {Diffeomorphic interpolation for efficient persistence-based topological optimization},
 year = {2024}
}

@article{dif_ph3,
  title={D-gril: End-to-end topological learning with 2-parameter persistence},
  author={Mukherjee, Soham and Samaga, Shreyas N and Xin, Cheng and Oudot, Steve and Dey, Tamal K},
  journal={International Symposium on Computational Geometry},
  year={2026}
}

@inproceedings{mem1,
    title={Does learning require memorization? a short tale about a long tail},
    author={Feldman, Vitaly},
    booktitle={ACM SIGACT Symposium on Theory of Computing},
    year={2020}
}

@article{mem2,
    title={What Neural Networks Memorize and Why: Discovering the Long Tail via Influence Estimation},
    author={Feldman, Vitaly and Zhang, Chiyuan},
    journal={Advances in Neural Information Processing Systems},
    year={2020}
}

@inproceedings{mem3,
    title={Memorization through the lens of curvature of loss function around samples},
    author={Garg, Isha and Ravikumar, Deepak and Roy, Kaushik},
    booktitle={International Conference on Machine Learning},
    year={2024}
}

@book{cover2006elements,
  title={Elements of Information Theory},
  author={Cover, Thomas M. and Thomas, Joy A.},
  edition={2nd},
  publisher={Wiley-Interscience},
  year={2006}
}

@inproceedings{sun2022information,
  title={Information-theoretic Online Memory Selection for Continual Learning},
  author={Sun, Shengyang and Calandriello, Daniele and Hu, Huiyi and Li, Ang and Titsias, Michalis K.},
  booktitle={International Conference on Learning Representations (ICLR)},
  year={2022}
}

@article{multipers,
  title = {Multipers: {{Multiparameter Persistence}} for {{Machine Learning}}},
  shorttitle = {Multipers},
  author = {Loiseaux, David and Schreiber, Hannah},
  year = {2024},
  month = nov,
  journal = {Journal of Open Source Software},
  volume = {9},
  number = {103},
  pages = {6773},
  issn = {2475-9066},
  doi = {10.21105/joss.06773},
  langid = {english},
}

@inproceedings{deepcore,
  title={Deepcore: A comprehensive library for coreset selection in deep learning},
  author={Guo, Chengcheng and Zhao, Bo and Bai, Yanbing},
  booktitle={International Conference on Database and Expert Systems Applications},
  year={2022},
  organization={Springer}
}

@article{umap-software,
  title={UMAP: Uniform Manifold Approximation and Projection},
  author={McInnes, Leland and Healy, John and Saul, Nathaniel and Grossberger, Lukas},
  journal={The Journal of Open Source Software},
  volume={3},
  number={29},
  pages={861},
  year={2018}
}

@incollection{gudhi:RipsComplex
, author    = {Cl{\'{e}}ment Maria and Pawel Dlotko and Vincent Rouvreau and Marc Glisse}
, title     = {Rips complex}
, publisher = {GUDHI Editorial Board}
, edition   = {3.11.0}
, booktitle = {GUDHI User and Reference Manual}
, year      = {2025}
}

@inproceedings{resnet,
  title={Deep residual learning for image recognition},
  author={He, Kaiming and Zhang, Xiangyu and Ren, Shaoqing and Sun, Jian},
  booktitle={IEEE/CVF Conference on Computer Vision and Pattern Recognition},
  year={2016}
}

@inproceedings{efficientnet,
  title={Efficientnet: Rethinking model scaling for convolutional neural networks},
  author={Tan, Mingxing and Le, Quoc},
  booktitle={International Conference on Machine Learning},
  year={2019},
}

@inproceedings{efficientnetv2,
  title={Efficientnetv2: Smaller models and faster training},
  author={Tan, Mingxing and Le, Quoc},
  booktitle={International Conference on Machine Learning},
  year={2021},
}

@inproceedings{swinv2,
  title={Swin transformer v2: Scaling up capacity and resolution},
  author={Liu, Ze and Hu, Han and Lin, Yutong and Yao, Zhuliang and Xie, Zhenda and Wei, Yixuan and Ning, Jia and Cao, Yue and Zhang, Zheng and Dong, Li and others},
  booktitle={IEEE/CVF Conference on Computer Vision and Pattern Recognition},
  year={2022}
}

@inproceedings{vit,
  title={An Image is Worth 16x16 Words: Transformers for Image Recognition at Scale},
  author={Dosovitskiy, Alexey and Beyer, Lucas and Kolesnikov, Alexander and Weissenborn, Dirk and Zhai, Xiaohua and Unterthiner, Thomas and Dehghani, Mostafa and Minderer, Matthias and Heigold, Georg and Gelly, Sylvain and others},
  booktitle={International Conference on Learning Representations},
  year={2021}
}

@inproceedings{laion,
  title={{LAION}-5B: An open large-scale dataset for training next generation image-text models},
  author={Christoph Schuhmann and
          Romain Beaumont and
          Richard Vencu and
          Cade W Gordon and
          Ross Wightman and
          Mehdi Cherti and
          Theo Coombes and
          Aarush Katta and
          Clayton Mullis and
          Mitchell Wortsman and
          Patrick Schramowski and
          Srivatsa R Kundurthy and
          Katherine Crowson and
          Ludwig Schmidt and
          Robert Kaczmarczyk and
          Jenia Jitsev},
  booktitle={Advances in Neural Information Processing Systems Datasets and Benchmarks Track},
  year={2022},
}

@inproceedings{clip,
  title={Learning Transferable Visual Models From Natural Language Supervision},
  author={Alec Radford and Jong Wook Kim and Chris Hallacy and A. Ramesh and Gabriel Goh and Sandhini Agarwal and Girish Sastry and Amanda Askell and Pamela Mishkin and Jack Clark and Gretchen Krueger and Ilya Sutskever},
  booktitle={International Conference on Machine Learning},
  year={2021}
}

@book{sum_scalar0,
  title={Multicriteria Optimization (2. ed.)},
  author={Matthias Ehrgott},
  publisher={Springer},
  year={2005},
}

@article{sum_scalar1,
  title={The weighted sum method for multi-objective optimization: new insights},
  author={R. Timothy Marler and Jasbir Arora},
  journal={Structural and Multidisciplinary Optimization},
  year={2010},
}

@book{edelsbrunner2010computational,
  title={Computational Topology: An Introduction},
  author={Edelsbrunner, Herbert and Harer, John L},
  year={2010},
  publisher={American Mathematical Society}
}

@article{chazal2009gromov,
  title={Gromov-Hausdorff stable signatures for shapes using persistence},
  author={Chazal, Fr{\'e}d{\'e}ric and Cohen-Steiner, David and Guibas, Leonidas J and M{\'e}moli, Facundo and Oudot, Steve Y},
  journal={Computer Graphics Forum},
  year={2009}
}

@inproceedings{hendrycksbenchmarking,
  title={Benchmarking Neural Network Robustness to Common Corruptions and Perturbations},
  author={Hendrycks, Dan and Dietterich, Thomas},
  booktitle={International Conference on Learning Representations},
  year={2019}
}

@inproceedings{Zhu2022DetectingCL,
  title={Detecting Corrupted Labels Without Training a Model to Predict},
  author={Zhaowei Zhu and Zihao Dong and Yang Liu},
  booktitle={International Conference on Machine Learning},
  year={2022},
}

@article{carlsson2010characterization,
  title={Characterization, Stability and Convergence of Hierarchical Clustering Methods},
  author={Carlsson, Gunnar and M{\'e}moli, Facundo},
  journal={Journal of Machine Learning Research},
  year={2010}
}

@article{chazal2013persistence,
  title={Persistence-Based Clustering in Riemannian Manifolds},
  author={Chazal, Fr{\'e}d{\'e}ric and Guibas, Leonidas J and Oudot, Steve Y and Skraba, Primoz},
  journal={Journal of the ACM},
  year={2013}
}

@inproceedings{chaudhuri2010rates,
  title={Rates of Convergence for the Cluster Tree},
  author={Chaudhuri, Kamalika and Dasgupta, Sanjoy},
  booktitle={Advances in Neural Information Processing Systems},
  year={2010}
}

@article{roberta,
  title={{RoBERTa}: A Robustly Optimized {BERT} Pretraining Approach},
  author={Liu, Yinhan and Ott, Myle and Goyal, Naman and Du, Jingfei and Joshi, Mandar and Chen, Danqi and Levy, Omer and Lewis, Mike and Zettlemoyer, Luke and Stoyanov, Veselin},
  journal={ArXiv},
  year={2019}
}

@inproceedings{anli,
  title={Adversarial {NLI}: A New Benchmark for Natural Language Understanding},
  author={Nie, Yixin and Williams, Adina and Dinan, Emily and Bansal, Mohit and Weston, Jason and Kiela, Douwe},
  booktitle={Annual Meeting of the Association for Computational Linguistics},
  year={2020}
}

@article{llama3,
  title={The {Llama} 3 Herd of Models},
  author={Grattafiori, Aaron and Dubey, Abhimanyu and Jauhri, Abhinav and Pandey, Abhinav and Kadian, Abhishek and Al-Dahle, Ahmad and Letman, Aiesha and Mathur, Akhil and Schelten, Alan and Vaughan, Alex and others},
  journal={ArXiv},
  year={2024}
}

@article{qwen25,
  title={{Qwen2.5} Technical Report},
  author={Yang, An and Yang, Baosong and Zhang, Beichen and Hui, Binyuan and Zheng, Bo and Yu, Bowen and others},
  journal={ArXiv},
  year={2024}
}

@inproceedings{pythia,
  title={Pythia: A Suite for Analyzing Large Language Models Across Training and Scaling},
  author={Biderman, Stella and Schoelkopf, Hailey and Anthony, Quentin Gregory and Bradley, Herbie and O'Brien, Kyle and Hallahan, Eric and Khan, Mohammad Aflah and Purohit, Shivanshu and Prashanth, USVSN Sai and Raff, Edward and Skowron, Aviya and Sutawika, Lintang and van der Wal, Oskar},
  booktitle={International Conference on Machine Learning},
  year={2023}
}

@article{alpaca,
  title={Alpaca: A strong, replicable instruction-following model},
  author={Taori, Rohan and Gulrajani, Ishaan and Zhang, Tianyi and Dubois, Yann and Li, Xuechen and Guestrin, Carlos and Liang, Percy and Hashimoto, Tatsunori B.},
  journal={Stanford Center for Research on Foundation Models. https://crfm. stanford. edu/2023/03/13/alpaca. html},
  year={2023}
}

@inproceedings{lora,
  title={{LoRA}: Low-Rank Adaptation of Large Language Models},
  author={Hu, Edward J and Shen, Yelong and Wallis, Phillip and Allen-Zhu, Zeyuan and Li, Yuanzhi and Wang, Shean and Wang, Lu and Chen, Weizhu},
  booktitle={International Conference on Learning Representations},
  year={2022}
}

@misc{lmeval,
  title={The Language Model Evaluation Harness},
  author={Gao, Leo and Tow, Jonathan and Abbasi, Baber and Biderman, Stella and Black, Sid and DiPofi, Anthony and Foster, Charles and Golding, Laurence and Hsu, Jeffrey and Le Noac'h, Alain and Li, Haonan and McDonell, Kyle and Muennighoff, Niklas and Ociepa, Chris and Phang, Jason and Reynolds, Laria and Schoelkopf, Hailey and Skowron, Aviya and Sutawika, Lintang and Tang, Eric and Thite, Anish and Wang, Ben and Wang, Kevin and Zou, Andy},
  howpublished={Zenodo},
  year={2024}
}

@inproceedings{mmlu,
  title={Measuring Massive Multitask Language Understanding},
  author={Hendrycks, Dan and Burns, Collin and Basart, Steven and Zou, Andy and Mazeika, Mantas and Song, Dawn and Steinhardt, Jacob},
  booktitle={International Conference on Learning Representations},
  year={2021}
}
\newpage
\appendix
\etocdepthtag.toc{atoc}
\section*{Appendix}

\newcommand{\appendixtoc}{%
    \begingroup
    \etocsettagdepth{mtoc}{none}
    \etocsettagdepth{atoc}{subsubsection}
    
    \parbox[b]{0.96\textwidth}{
        \etocsettocstyle{\subsubsection*{Table of Contents\\ \vspace{-0.75em}\rule{\textwidth}{0.4pt}}\vspace{-.75em}}{}%
        \etoctableofcontents 
        \vspace{-0.25em} 
        \rule{\textwidth}{0.4pt}
    }\\
    \endgroup
}
\appendixtoc

\section{Conceptual Framework \& Intuition}
\subsection{Overview of Simplicial Complexes and Persistent Homology}
\label{Sec:appendix_persistent_homology}

Persistent homology is used to characterize the topological variations in the shape of a finite metric space across multiple scales. At a high level, this can be described as the ``birth'' and ``death'' (persistence) of topological structures (defined by a homology group). The process begins by constructing a \textit{simplicial complex}, a collection of points (0-simplices), edges (1-simplices), triangles (2-simplices), and their higher-dimensional counterparts that represents the data's structure \citep{ph1.1}. To analyze how this structure changes with scale, a \textit{filtration} is created (see \cref{fig:persistant_homology_overview}a). This is a nested sequence of simplicial complexes, $K_{r_1} \subseteq K_{r_2} \subseteq \dots \subseteq K_{r_n}$, indexed by a non-decreasing scale parameter $r$. For each complex $K_r$ in the filtration, we can compute its \textit{homology groups}, $H_k(K_r)$, which are vector spaces that algebraically capture its $k$-dimensional features. 

The rank of this group, known as the $k$-th \textit{Betti number} ($\beta_k = \text{rank}(H_k(K_r))$), provides a count of these features: $\beta_0$ counts connected components, $\beta_1$ counts loops or tunnels, $\beta_2$ counts voids, and so on. These values are central to understanding the distinction between an object's \textit{extrinsic} geometry and its \textit{intrinsic} topological properties. A classic example which illustrates this difference is that of a coffee mug and a torus (donut). These two objects are topologically equivalent because they share the same Betti numbers (see \cref{fig:persistant_homology_overview}b). Although their extrinsic geometries (including their shape, curvatures, and distances as embedded in 3D space are very different), their intrinsic topology is identical. This is because one can be continuously deformed into the other without tearing or gluing, preserving the single hole that defines them both.

\begin{figure*}[!tp]
    \centering
    \includegraphics[width=0.75\textwidth]{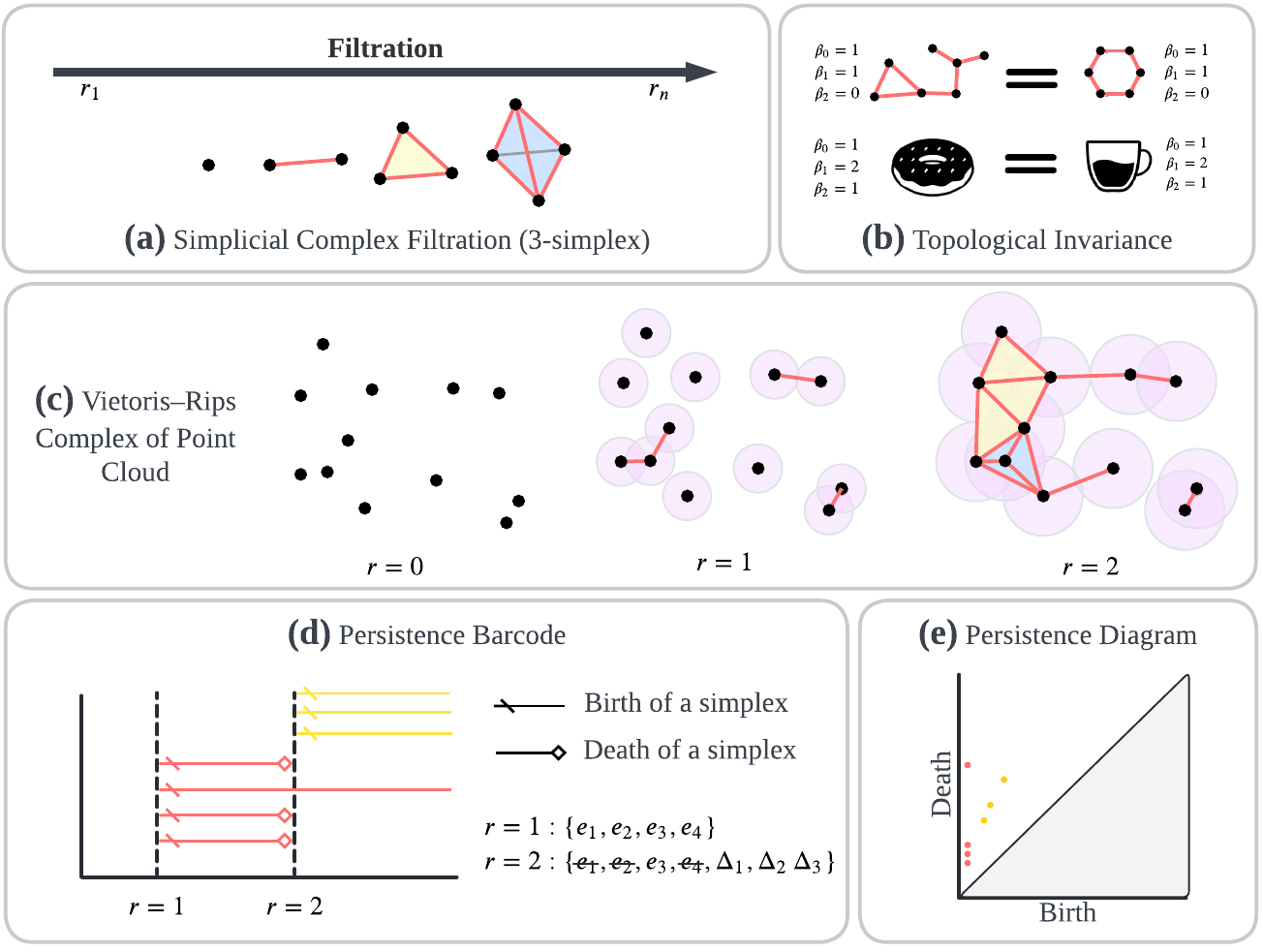}
    \caption{Overview of Simplicial Complexes and Persistent Homology}
    \label{fig:persistant_homology_overview}
\vspace{-\baselineskip}
\end{figure*}

Persistent homology is not wedded to any form of metric construction, and in fact you can do persistence on purely abstract simplicial complexes and any filtration on it. For conceptual clarity and alignment with foundational stability proofs, we illustrate this process on a point cloud $P = \{\mathbf{x}_i\}$ using the common \textit{Vietoris-Rips (VR) complex} (see \cref{fig:persistant_homology_overview}c). Note that while the VR complex serves as our pedagogical example in this section, the actual implementation of \sys employs a Delaunay filtration to maximize computational scalability \citep{ph7}. For a given scale $r \ge 0$, the complex $VR(P, r)$ contains all simplices $\sigma \subseteq P$ such that the Euclidean distance between any two points in $\sigma$ is at most $2r$. As $r$ increases, simplices are added to the complex, causing new components to merge with older components. \textit{Persistent homology tracks the birth and death of these topological features throughout the filtration.} The inclusion map $K_{r_i} \hookrightarrow K_{r_j}$ for $r_i \le r_j$ induces a homomorphism between the homology groups, $H_k(K_{r_i}) \to H_k(K_{r_j})$. A feature is said to be "born" at a scale $r_{\text{birth}}$ when it first appears and "dies" at a scale $r_{\text{death}}$ when it merges with an older feature visualized by the persistence barcode (\cref{fig:persistant_homology_overview}d).

\begin{definition}[Vietoris-Rips Filtration]\label{alg:vr}
For a point cloud $P \subset \mathbb{R}^n$ and a scale parameter $r \ge 0$, the Vietoris-Rips complex $VR(P, r)$ is the simplicial complex whose vertices are the points in $P$ and whose simplices are all finite subsets of $P$ with a diameter of at most $2r$. A filtration is the nested sequence of complexes $\{VR(P, r)\}_{r \ge 0}$.
\end{definition}

The output of this process is summarized in a \textit{persistence diagram} $\text{Dgm}(P)$, a multiset of points in the plane where each point corresponds to a single topological feature plotted at its $(\text{birth}, \text{death}) \rightarrow (b, d)$ coordinates (see \cref{fig:persistant_homology_overview}e). The \textit{persistence} of a feature is defined as its lifespan, $d - b$. \textit{Points in the diagram that are further from the diagonal line $y=x$ represent robust, structurally significant features of the data}, while points close to the diagonal are interpreted as topological noise with short lifespans. This provides a stable, multi-scale signature of the data's underlying shape.

\begin{definition}[Persistence Diagram]\label{alg:dgm}
Applying the homology functor $H_k(\cdot)$ (for a fixed dimension $k$, e.g., $k=0$ for connected components) to a filtration yields a set of birth-death pairs $(b, d)$ representing the scales at which topological features appear and disappear. This multiset of pairs is the persistence diagram, denoted $\text{Dgm}(P)$. The persistence of a feature $(b, d)$ is defined as $d - b$.
\end{definition}

Please note that for clarity and ease of visualization in this overview section, we present the 1-parameter persistence analysis. It is important to note, however, that our method, \sys, employs a multi-parameter persistence module, which is more complex to visualize but provides a richer description of the data's topology.

\subsection{Background: On the Transferability of Topological vs. Euclidean Features}
\label{Sec:Appendix_euclidean_topology}

Previous literature formally demonstrates the superior transferability of topological features derived from persistent homology compared to conventional Euclidean measurements \citep{beyondeuclid}. The performance gap across disparate network architectures stems from two distinct structural asymmetries: (1) Euclidean centroids are extrinsic vector-space constructions that do not commute with non-linear mappings, whereas persistent homology is computed purely from intrinsic pairwise distances; and (2) persistent homology enjoys a rigorous Lipschitz bound against metric distortion, whereas centroid-based distance rankings have no such guarantee.

\paragraph{Preliminaries and Notation.} Let $X$ be the input data space and $Y = \{1, \dots, K\}$ be the set of $K$ class labels. A neural network architecture is a function $f: X \to \mathbb{R}^n$ mapping input data to an $n$-dimensional embedding space. Let $f_A$ and $f_B$ denote two distinct network architectures (e.g., ResNet18 and ViT-L-16). The outputs of these networks for the dataset $X$ are the point clouds $X_A = f_A(X)$ and $X_B = f_B(X)$. We equip these spaces with the standard Euclidean metric, $d_E$.

\begin{definition}[Bottleneck Distance]
The similarity between two persistence diagrams $\text{Dgm}_1$ and $\text{Dgm}_2$ is measured by the bottleneck distance $d_B(\text{Dgm}_1, \text{Dgm}_2)$, defined as the infimum over all bijections $\eta: \text{Dgm}_1 \to \text{Dgm}_2$ of the supremum of distances between matched points, where $p \in \text{Dgm}$ represents a birth-death pair $p=(b, d)$:
$$d_B(\text{Dgm}_1, \text{Dgm}_2) = \inf_{\eta} \sup_{p \in \text{Dgm}_1} \| p - \eta(p) \|_{\infty}$$
\end{definition}

\begin{definition}[Gromov-Hausdorff Distance]
The distance between two metric spaces $(M_1, d_1)$ and $(M_2, d_2)$ is measured by the Gromov-Hausdorff distance $d_{GH}(M_1, M_2)$, which is the infimum of distances over all possible isometric embeddings into a common metric space. It bounds the maximum metric distortion required to map one space onto another.
\end{definition}

\subsubsection{Structural Fragility of Euclidean Centroids}

The fundamental limitation of Euclidean prototype (centroid) distances is that the centroid is a vector-space construction, not a metric one. For any non-linear transformation $\phi$ between network embedding spaces, the centroid of the mapped points is not equal to the mapped centroid of the original points: $\phi\left(\frac{1}{|X|}\sum x_i\right) \neq \frac{1}{|X|}\sum \phi(x_i)$. Consequently, under non-linear architectural shifts, the target centroid has no structural analog to the source centroid. Furthermore, even under simple affine transformations, distance rankings to the centroid are highly brittle.

\begin{definition}[Class Prototype and Distance Distribution]
For an embedding $f(X)$ and a class $k \in Y$, the class prototype (centroid) is $c_k = \frac{1}{|X_k|} \sum_{x \in X_k} f(x)$. The set of distances to the prototype is $S_k(f) = \{ d_E(f(x), c_k) \mid \text{label}(x)=k \}$. 
\end{definition}

\begin{proposition}[Sensitivity to Anisotropic Distortion]
Let $f_A$ be a network embedding. Consider a new embedding $f_B$ defined by a non-uniform anisotropic scaling transformation, $f_B(x) = \Lambda f_A(x)$, where $\Lambda$ is a diagonal matrix with non-equal strictly positive entries. Under this transformation, the relative Euclidean distance rankings of samples to the class prototype can be inverted.
\end{proposition}

\begin{proof}
Without loss of generality, let the centroid under $f_A$ be at the origin, $c_k = (0,0)$. Consider two samples $x_1 = (1, 0)$ and $x_2 = (0, 1.1)$. Under $f_A$, $x_1$ is strictly closer to the centroid than $x_2$ because $d_E(x_1, c_k) = 1 < d_E(x_2, c_k) = 1.1$. 

Now apply an anisotropic scaling matrix $\Lambda = \text{diag}(2, 1)$. The new centroid remains at the origin, $c'_k = (0,0)$. The mapped samples are $f_B(x_1) = (2, 0)$ and $f_B(x_2) = (0, 1.1)$. In the new embedding space, $d_E(f_B(x_1), c'_k) = 2$ and $d_E(f_B(x_2), c'_k) = 1.1$. Thus, $d_E(f_B(x_1), c'_k) > d_E(f_B(x_2), c'_k)$, completely inverting the relative importance ranking of the samples.
\end{proof}

\begin{remark}
This illustrates that even arbitrarily small amounts of anisotropic distortion, which are unavoidable when mapping representations across different deep learning architectures, can invert sample distance rankings.
\end{remark}

\subsubsection{Stability Guarantees for Persistent Homology}

Unlike Euclidean centroids, persistent homology is constructed exclusively from intrinsic pairwise distances (e.g., via the Vietoris-Rips filtration). Because it does not rely on averaging coordinates, its transferability across embedding spaces is directly governed by established stability theorems.

\begin{proposition}[Stability of Persistent Homology \citep{chazal2009gromov, ph1}]\label{prop:ph_stability}
Let $X_A$ and $X_B$ be two point clouds representing the same data manifold mapped into distinct embedding spaces. The bottleneck distance between their respective Vietoris-Rips persistence diagrams is bounded by the Gromov-Hausdorff distance between their metric spaces:
$$d_B(\text{Dgm}(X_A), \text{Dgm}(X_B)) \le 2 \cdot d_{GH}((X_A, d_E), (X_B, d_E))$$
\end{proposition}

\begin{remark}
This provides a Lipschitz-type bound against metric distortion. The factor of $2$ explicitly arises from the use of the Vietoris-Rips complex (which relies strictly on pairwise distances) as an approximation of the \v Cech complex. While structural distortions across varying model capacities (measured by $d_{GH}$) will inevitably perturb the persistence diagrams, this bounded behavior ensures that the fundamental topological features of the dataset (and thus their resulting importance scores) degrade gracefully.
\end{remark}

\subsection{Topological Connectivity and the Geometry of Sample Memorization}
\label{Sec:Appendix_topology_memorization}

Grounded in established coreset frameworks \citep{scaling2, moderateds}, we utilize the intra-class manifold density as a global signal for sample prototypicality (where dense regions represent common samples and sparse regions represent rare, atypical samples). To qualitatively illustrate this, we compare our density metric against sample memorization \citep{mem1, mem2}, measured via input curvature \citep{mem3}. We observe a general correspondence: \textit{high-density} samples consistently exhibit \textit{low input curvature} (un-memorized), while \textit{low-density} samples show \textit{high input curvature} (memorized).

While intended as a qualitative illustration, both metrics successfully identify similar prototypical subsets (e.g., orange highlights in \cref{fig:mem_atypical}). However, divergent selections arise because input curvature relies on the Hessian of the loss, making it highly sensitive to training dynamics and intrinsic biases. As depicted in \cref{fig:mem_prototypical}, it can capture spurious correlations such as ``boys with red shirts'' as the least memorized. In contrast, our density score utilizes frozen, pre-trained embeddings. This training-agnostic approach yields a robust structural prior for our fine-grained persistence optimization.

\begin{figure}[!h]
    \centering
    \includegraphics[width=0.95\textwidth]{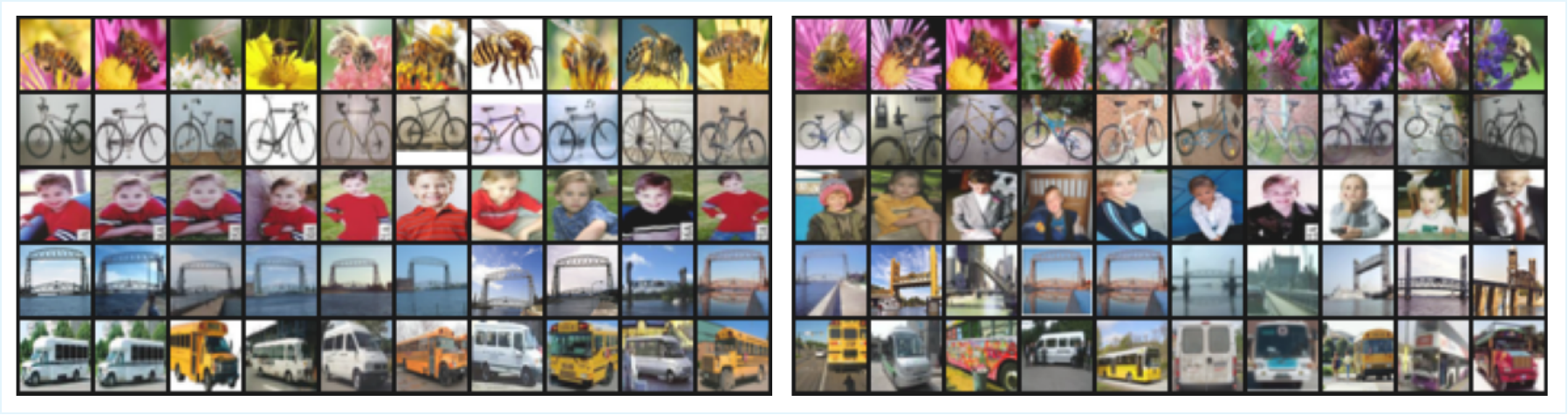}
    \caption{\textbf{Prototypical Samples:} Top-10 lowest curvature samples (left) vs. highest density samples (right) of the same class, for five CIFAR-100 classes.}
    \label{fig:mem_prototypical}
\end{figure}

\begin{figure}[!h]
    \centering
    \includegraphics[width=0.95\textwidth]{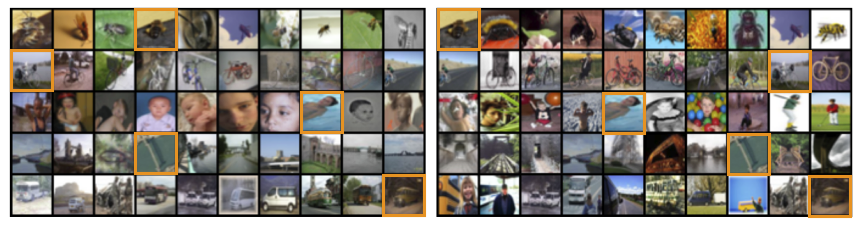}
    \caption{\textbf{Atypical Samples:} Top-10 highest curvature samples (left) vs. lowest density samples (right) of the same class, for five CIFAR-100 classes. We highlight a few examples of matched samples for each class in orange.}
    \label{fig:mem_atypical}
\end{figure}

\section{Implementation \& Computational Analysis}
\subsection{Comparative Analysis of Manifold Projection Techniques}
\label{Sec:Appendix_manifold_projection_techniques}

Our global manifold projection is critical for achieving metric stability across diverse neural network architectures. While high-dimensional embeddings vary drastically in their extrinsic geometry and absolute cluster densities, they share a common intrinsic topology. UMAP leverages this shared structure to construct a low-dimensional manifold that explicitly normalizes the data distribution, effectively abstracting away architecture-specific scaling and artifacts. This standardization ensures that the global density score, a core component of our sample importance calculation, is a stable and reliable metric regardless of the source network.

\subsubsection{Evaluation of Standard Linear and Non-linear Manifold Approximations}
To further elaborate the standardization of topology-based manifold approximation and projection across perturbations in the embedding space we look at correlation (\cref{fig:manifold_projection_ablation}) of per-sample distance to prototypes across different manifold projection and feature reduction techniques (a) PCA \citep{pca}, (b) t-SNE \citep{tsne} (c) PaCMAP \citep{pacmap}, (d) DensMAP \citep{densmap}, and (e) UMAP \citep{umap}. We see that the topology-based methods PaCMAP, DensMAP and UMAP demonstrate significantly higher correlation and thus better transferability across architectures compared to linear PCA or the more locally-focused t-SNE. Notably, UMAP exhibits strong transferability, outperforming PaCMAP. Although DensMAP yields a marginal +0.02 improvement in average correlation, UMAP provides more than sufficient structural fidelity for our framework. This high correlation between smaller (e.g., ResNet-18) and larger models is particularly valuable, as it confirms that computationally inexpensive networks can reliably generate embeddings to guide data selection for much larger models.

\begin{figure}[!h]
    \centering
    
    \begin{subfigure}[b]{0.32\textwidth}
        \centering
        \includegraphics[width=\linewidth]{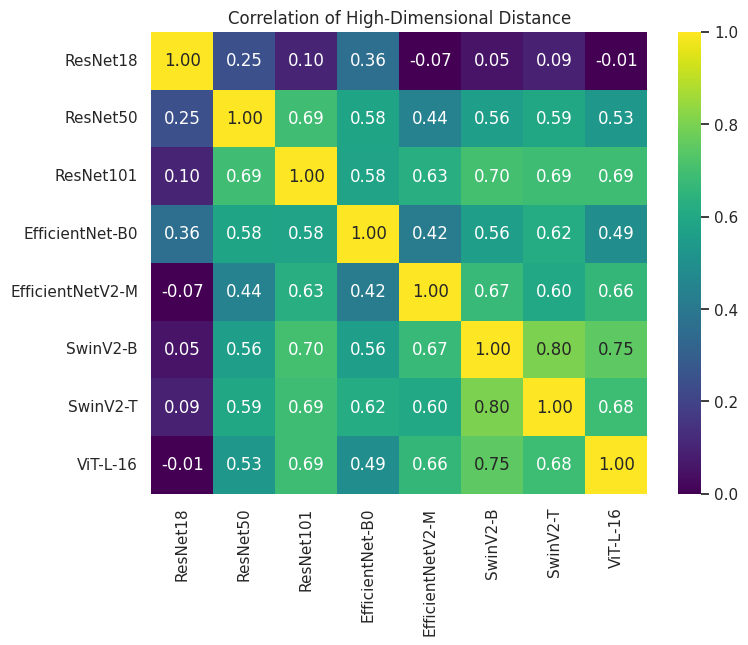}
        \caption{Euclidean Distance to Prototype}
        \label{fig:corr_euclid}
    \end{subfigure}
    \begin{subfigure}[b]{0.32\textwidth}
        \centering
        \includegraphics[width=\linewidth]{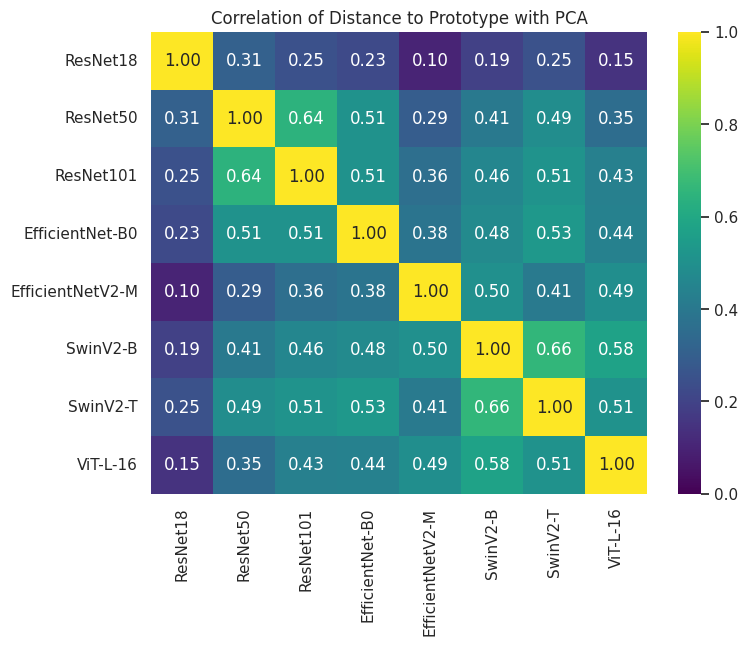}
        \caption{PCA}
        \label{fig:corr_pca}
    \end{subfigure}
    \begin{subfigure}[b]{0.32\textwidth}
        \centering
        \includegraphics[width=\linewidth]{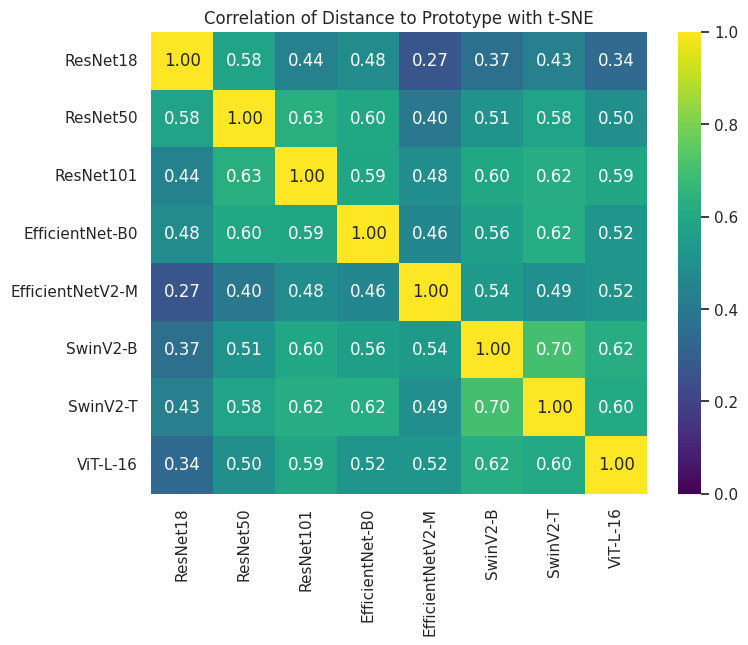}
        \caption{t-SNE}
        \label{fig:corr_tsne}
    \end{subfigure}
    \begin{subfigure}[b]{0.32\textwidth}
        \centering
        \includegraphics[width=\linewidth]{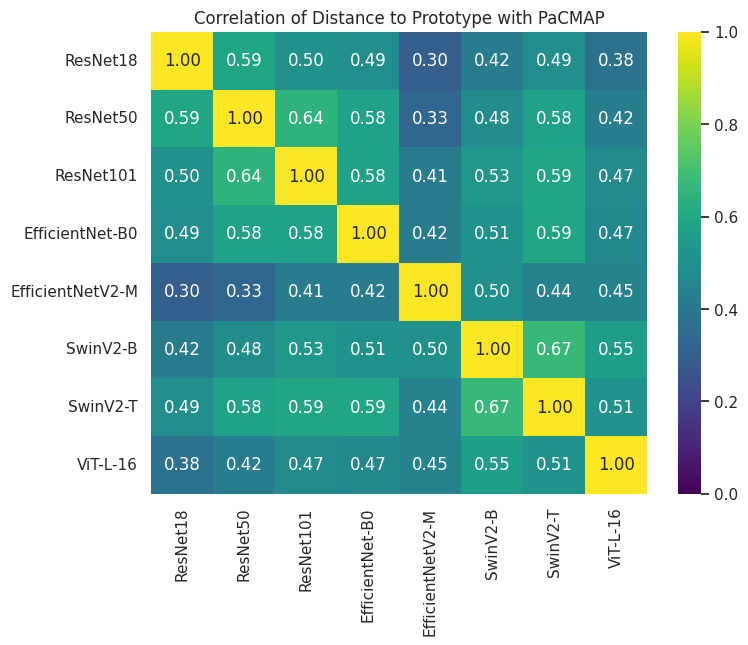}
        \caption{PaCMAP}
        \label{fig:corr_pacmap}
    \end{subfigure}
    \begin{subfigure}[b]{0.32\textwidth}
        \centering
        \includegraphics[width=\linewidth]{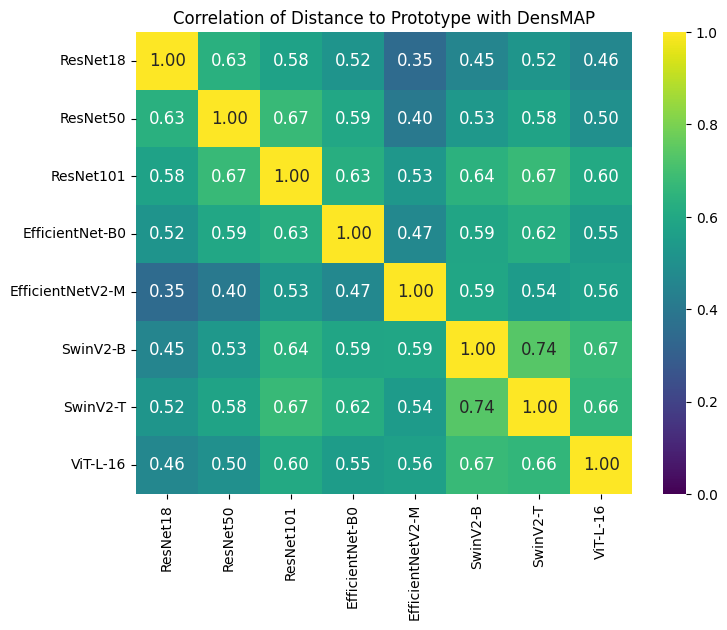}
        \caption{DensMAP}
        \label{fig:corr_umap}
    \end{subfigure}
    \begin{subfigure}[b]{0.32\textwidth}
        \centering
        \includegraphics[width=\linewidth]{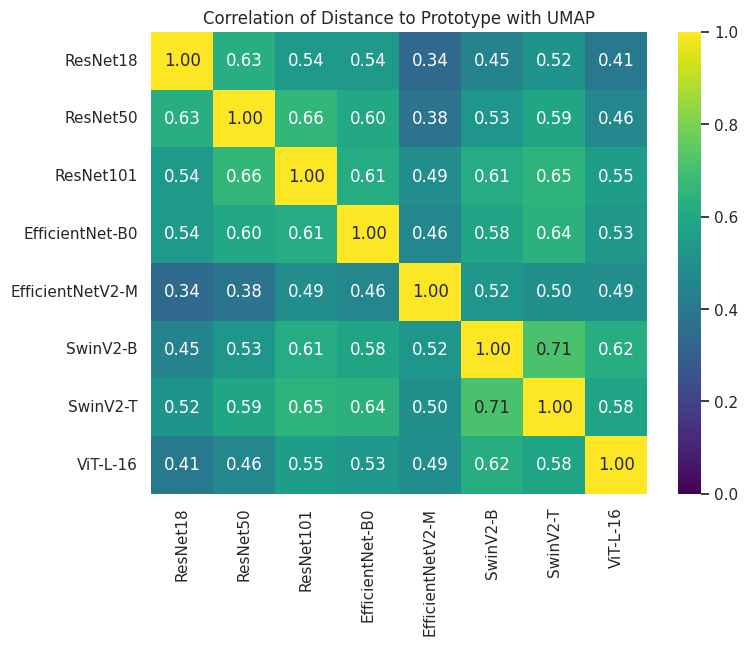}
        \caption{UMAP}
        \label{fig:corr_densmap}
    \end{subfigure}
    
    \caption{Correlation of per-sample distance to prototype across different architectures when applying different linear and non-linear manifold projection techniques.}
    \label{fig:manifold_projection_ablation}
\end{figure}

\subsubsection{Evaluation of Deep Topological Autoencoders}
\label{sec:appendix_topo_ae}
We also considered Topological Auto-Encoders (TopoAE) \citep{topoae} and Representation Topological Divergence (RTD) \citep{rtd} for our manifold embedding. While these methods are powerful for preserving global topology, we chose UMAP for computational efficiency, domain suitability, and alignment with coreset selection goals.

\sys aims to be a lightweight, training-free method applicable to frozen features, a requirement UMAP fits perfectly. In contrast, training a topological autoencoder is computationally prohibitive for preprocessing; training RTD on CIFAR-100 takes $\sim$8 hours compared to UMAP's $\sim$22 seconds, representing a speedup of $> 1000 \times$ (see \cref{tab:topo_ae_latency}). To empirically validate our choice, we trained both TopoAE and RTD models on CIFAR-10/100 and used their embeddings as a replacement for UMAP in our pipeline. As shown in \cref{tab:topo_ae_acc}, UMAP consistently yields superior or comparable accuracy without this massive training overhead.

\paragraph{Domain suitability and structural alignment.} The observed drop in coreset accuracy when using topological autoencoders likely stems from a fundamental misalignment of objectives. First, regarding latent space quality, TopoAE struggles to produce clean, separated representations for complex datasets. As noted in the TopoAE paper itself, CIFAR-10 is ``challenging to embed... in a purely unsupervised manner'' (Section 5.2.2 in \citet{topoae}), often resulting in latent spaces where classes are homogeneously mixed rather than cleanly separated (see Figure 4 in \citet{topoae}). This lack of separation severely hampers the effectiveness of our per-class density estimation, contrasting with the distinct cluster delineation achieved by UMAP. Furthermore, this issue is exacerbated because methods like TopoAE and RTD prioritize preserving global structural similarity (e.g., maintaining relative distances between distinct mammoth ``head'' and ``foot'' clusters as shown in Figure 1 in \citet{rtd}). While this global constraint is valuable for visualization, it is less relevant for coreset selection, where we partition the data into class-based manifolds.

\begin{table}[!tp]
    \vspace{-\baselineskip}
    \caption{\textbf{Throughput and accuracy when using Topological autoencoders.} \textbf{(a)} Topological autoencoders require costly training vs. UMAP's algorithmic projection. \textbf{(b)} Substituting UMAP with TopoAE or RTD for the global manifold embedding shows performance degradation.}
    \vspace{1em}
    \label{tab:topo_ae}
    
    \centering
    \begin{subtable}[b]{0.47\textwidth}
        \centering

        {
        \begin{tabular}{lcc}\toprule
        \multicolumn{1}{c}{\textbf{Method}} & \textbf{CIFAR-10} & \textbf{CIFAR-100} \\
        \midrule
        TopoAE \citep{topoae} & 14,847.79 & 15,248.11 \\
        RTD \citep{rtd} & 26,622.89 & 28,661.29 \\
        \textbf{UMAP} \citep{umap} & \textbf{22.42} & \textbf{22.74} \\
        \bottomrule
        \end{tabular}
        }
        
        \caption{Latency (s)}
        \label{tab:topo_ae_latency}
    \end{subtable}
    \hspace{0.02\textwidth}
    \begin{subtable}[b]{0.47\textwidth}
        \centering

        {
        \begin{tabular}{lcc}\toprule
        \multicolumn{1}{c}{\textbf{Method}} & \textbf{CIFAR-10} & \textbf{CIFAR-100} \\
        \midrule
        TopoAE \citep{topoae} & 75.0$\pm$0.3 & 41.2$\pm$1.0 \\
        RTD \citep{rtd} & \underline{78.0$\pm$1.7} & \textbf{46.4$\pm$0.4} \\
        \textbf{UMAP} \citep{umap} & \textbf{82.1$\pm$0.3} & \underline{45.8$\pm$0.7} \\
        \bottomrule
        \end{tabular}
        }
        
        \caption{Accuracy at 90\% Pruning Rate (Avg. over 3 runs)}
        \label{tab:topo_ae_acc}
    \end{subtable}
    \vspace{-2em}

\end{table}

In summary, while topological autoencoders are robust tools for manifold learning, UMAP provides a more efficient, domain-appropriate, and higher-performing foundation for our specific multi-scale framework. We encourage future work exploring topological autoencoders that flexibly balance global and local structural priorities, specifically optimized for the task of point-cloud sparsification.

\subsection{Computational Complexity Analysis}
\label{Sec:appendix_complexity_analysis}

\paragraph{Theoretical Complexity Relative to Geometric Baselines.} Our analysis follows the framework for evaluation of coreset selection complexity in \citet{nagaraj2025coresets}. The core computational overhead of \sys stems from the local topological optimization. However, we maintain tractability by leveraging efficient reductions of multi-parameter persistence:

\begin{itemize}
    \item Crucially, we utilize the Hilbert decomposition signed measure, which \textit{reduces the multi-parameter problem to one-parameter persistence slices along a grid}. As detailed in Appendix D.1 of \citet{ph14}, for a 2-parameter filtration (Delaunay + Density) on a grid of size $m$, the algorithm performs $m$ runs of a 1-parameter persistence optimization.
    \item Consequently, the persistent homology optimization cost is equivalent to that of a 1-parameter optimization on $N_c$ points. While the theoretical worst-case for persistence is cubic \citep{ph15}, in the 1-parameter persistence case the computation is empirically linear \citep{ph16}.
    \item Furthermore, computing the gradient of the loss $\mathcal{L}_{\text{pers}}$ (\cref{eq:lpers}) simplifies to summing feature persistences (see Corollary E.2 in \citet{dif_ph1}). This operation is bounded by a constant $K$ derived from the simplicial complex $\mathcal{K}$, making the backward pass $\mathcal{O}(K)$. Thus, a single local optimization cost is strictly $\mathcal{O}(m \cdot N_c \log N_c)$.
\end{itemize}

By applying this efficient reduction strategy per-class (where $N_c$ is relatively small), \sys scales far more efficiently in high-dimensional spaces than current geometric baselines. For instance, while D2 \cite{d2} reports a graph construction complexity of $\mathcal{O}(Nkd)$, this accounts only for edge weighting. Their official implementation relies on exact $k$-NN search without approximation (\texttt{sklearn.neighbors.kneighbors\_graph}) and dense matrix allocation (\texttt{todense()}), which fundamentally scale unfeasibly at $\mathcal{O}(N^2 d)$ in time and $\mathcal{O}(N^2)$ in space for high-dimensional data. In contrast, \sys avoids this geometric memory bottleneck entirely by utilizing approximate nearest neighbor descent via UMAP, scaling efficiently at $\mathcal{O}(N \log N)$.

\begin{table}[!h]
\caption{\textbf{Complexity analysis of geometry-based methods.} $N$ is the dataset size with $C$ classes, $N_c \approx N/C$ samples per class, $d$ dimension, and $k$ neighbors. For \sys, cost is dominated by $T$ optimization steps and $m$ grid resolution.}
\vspace{0.5em}
\centering
\resizebox{\textwidth}{!}
{
\begin{tabular}{lll}\toprule
& \textbf{Computational Complexity} & \textbf{Explanation} \\

\midrule
Moderate \citep{moderateds} & $\mathcal{O}(Nd) + C*\mathcal{O}(N_c \log N_c)$ & Distance calc. ($Nd$) + Prototype sorting \\
D2 \citep{d2} & $\mathcal{O}(Nkd) + \mathcal{O}(T \cdot Nk)$ & kNN graph ($Nkd$) + Message passing \\
\textbf{\sys} & $\mathcal{O}(N \log N) + C * \mathcal{O}(Tm \cdot N_c \log N_c)$ & Global UMAP + Local persistence \\

\bottomrule
\end{tabular}

}
\label{table:complexity}
\end{table}

\paragraph{Empirical Wall-Clock Analysis.} We benchmarked the latency and resource utilization of \sys against baselines on an AMD EPYC 7502 (32-Core) CPU (\cref{table:wall_clock}). Our profiling reveals that the dominant cost, local persistence optimization, is independent across classes, so the original sequential implementation significantly under-utilized available hardware. This indicates an implementation-specific bottleneck rather than a fundamental algorithmic flaw. Our released code implementation addresses it by \textit{distributing classes over parallel single-threaded workers}, which yields bit-identical scores while accelerating the local step by $6.8\times$ (8 workers) and $10.6\times$ (16 workers) for the Delaunay complex, and by $7.4\times$ and $12.6\times$ for the Vietoris-Rips complex. End to end, \sys now selects a CIFAR-100 coreset in about two minutes while consuming less total CPU time than D2 (982\,s vs.\ 1{,}319\,s), which achieves its lower latency by occupying 42 cores on average. The remaining serial step is the seeded (deterministic) UMAP projection, and GPU support in our topological backend (\texttt{multipers}) offers further room for acceleration. Still, \sys offers a positive cost-benefit for data selection:
\begin{itemize}
    \item \sys is orders of magnitude faster than ``training-dynamic'' coreset methods (e.g., Glister, Forgetting), which require training a proxy model from scratch (hours of compute) compared to our topological probe of frozen embeddings (minutes).
    \item Unlike fast geometric heuristics (e.g., D2, Moderate) which suffer from high variance and lower precision (as shown in \cref{tab:acc_std}), \sys requires a comparable compute budget to guarantee a precise, high-fidelity coreset. This structural reliability eliminates the need for repeated, redundant selection runs to mitigate randomness, ultimately yielding a highly efficient and trustworthy pipeline.
\end{itemize}

\begin{table}[!h]
\caption{\textbf{Latency and resource usage} when performing selection on CIFAR-100 (AMD EPYC 7502, 32 cores; 6 optimization steps). Global: UMAP projection. Local: per-class persistence optimization. Total also includes coreset selection; feature extraction, shared by all methods, is excluded. CPU time sums over all processes, and Avg.\ Cores $=$ CPU time / Total.}
\centering
\resizebox{\linewidth}{!}{
\begin{tabular}{lcccccc}\toprule
& \textbf{Workers} & \textbf{Global (s)} & \textbf{Local (s)} & \textbf{Total (s)} & \textbf{CPU Time (s)} & \textbf{Avg.\ Cores} \\
\midrule
Moderate & -- & -- & -- & 4.7 & 10.9 & 2.3 \\
D2 & -- & -- & -- & 31.1 & 1{,}319.3 & 42.4 \\
\midrule
\textbf{\sys} (Vietoris-Rips Complex) & 1 & 59.2 & 1{,}564.2 & 1{,}624.7 & 2{,}563.0 & 1.6 \\
\textbf{\sys} (Vietoris-Rips Complex) & 8 & 59.2 & 210.9 & 271.4 & 1{,}687.8 & 6.2 \\
\textbf{\sys} (Vietoris-Rips Complex) & 16 & 59.2 & 123.7 & 184.2 & 1{,}781.4 & 9.7 \\
\midrule
\textbf{\sys} (Delaunay Complex) & 1 & 59.2 & 711.5 & 772.0 & 1{,}816.8 & 2.4 \\
\textbf{\sys} (Delaunay Complex) & 8 & 59.2 & 104.8 & 165.3 & 882.0 & 5.3 \\
\textbf{\sys} (Delaunay Complex) & 16 & 59.2 & 67.2 & 127.7 & 982.0 & 7.7 \\
\bottomrule
\end{tabular}
}
\label{table:wall_clock}
\end{table}

\newpage
\subsection{\sys Pseudocode}
\label{Sec:Appendix_pseudocode}
\renewcommand{\algorithmiccomment}[1]{\hfill $\rhd$ #1}          
\newcommand{\PARFOR}[2][default]{\FOR[#1]{#2 \textbf{in parallel}}}

\begin{algorithm}[!h]
   \caption{Coreset Selection with \texttt{TopoPrune}}
   \label{alg:topoprune}
\begin{algorithmic}
    \STATE {\bfseries Input:} 
    Dataset $\mathcal{D} = \{(x_i, y_i)\}_{i=1}^N$, Penultimate Layer Encoder $h_\theta$, Pruning Rate $p$, Mislabel Ratio $\gamma$, Weights $\alpha, \beta$, Persistence Optimization Steps $T$
    \STATE {\bfseries Callables:} \texttt{UMAP}, \texttt{KDE}, \texttt{KNN} \hfill $\rhd$ see Hyperparameters in \cref{table:hyperparameters}
    \STATE {\bfseries Output:} Selected Indices $\mathcal{S}$
    \newline
    
    \STATE \hrulefill
    \STATE \textit{\# Dual-Scale Topological Scoring}
    \STATE \hrulefill
    \STATE Extract embeddings: $Z \leftarrow h_\theta(X)$
    \STATE Global manifold projection: $Y \leftarrow \texttt{UMAP}(Z)$
    \STATE Initialize score vectors $S_{pers}, S_{dens} \leftarrow \mathbf{0}^N$
    
        
    \PARFOR[Independent per class]{each class $c \in \{1, \dots, C\}$}
        \STATE Get indices $\mathcal{I}_c \leftarrow \{i \mid y_i = c\}$ and subset $Y_c \leftarrow Y[\mathcal{I}_c]$
        \STATE $S_{dens}[\mathcal{I}_c] \leftarrow \texttt{KDE}(Y_c)$ \hfill $\rhd$ Global Density Score
        \STATE Initialize optimizable inputs $Y'_c \leftarrow Y_c$
        \FOR{$t=1$ {\bfseries to} $T$}
            \STATE Compute Hilbert decomposition signed measure: $\scriptstyle \mu_{H_1(Del_{Y'_c}, \hat{f})}^{Hil}$
            \STATE Compute topological loss: $\mathcal{L}_{\text{pers}}(Y'_c) \leftarrow \text{OT}(\mu_{H_1(Del_{Y'_c}, \hat{f})}^{Hil}, \mathbf{0})$
            \STATE Update coordinates: $Y'_c \leftarrow Y'_c + \eta \nabla_{Y'_c} \mathcal{L}_{pers}$
        \ENDFOR
        \STATE $S_{pers}[i] \leftarrow \|\mathbf{y}_i - \mathbf{y}'_i\|_2 \quad \forall i \in \mathcal{I}_c$ \hfill $\rhd$ Local Persistence Score
    \ENDFOR
    \newline
    \STATE \hrulefill
    \STATE \textit{\# Mislabel Detection and Filtering}
    \STATE \hrulefill
    \IF{method is NLPS}
        \FOR{each sample $i \in \{1, \dots, N\}$}
            \STATE Compute $k$ nearest neighbors: $\mathcal{N}_k(z_i) \leftarrow \texttt{KNN}(Z)$
            \STATE Count mismatched neighbors: $m_i \leftarrow \sum_{j \in \mathcal{N}_k(z_i)} \mathbb{I}(y_j \neq y_i)$
            \STATE Compute label purity ratio: $S_{mis}^{(i)} \leftarrow m_i / k$ \hfill $\rhd$ Training-Free
        \ENDFOR
    \ELSIF{method is AUM}
        \STATE Grab precomputed score: $S_{mis} \leftarrow \text{AUM}(Z)$ \hfill $\rhd$ with Training Dynamics
    \ENDIF
    \STATE Identify mislabeled sample indices: $\mathcal{I}_{mis} \leftarrow \text{TopK}(S_{mis}, \gamma \cdot N)$
    \STATE Define clean candidate set: $\mathcal{D}_{clean} \leftarrow \mathcal{D} \setminus \mathcal{I}_{mis}$
    \newline
    \STATE \hrulefill
    \STATE \textit{\# Stratified Sampling on Unified Score}
    \STATE \hrulefill
    \STATE $S_{unified}^{(i)} \leftarrow \alpha \cdot \mathcal{N} (S_{pers}^{(i)}) + \beta \cdot \mathcal{N}(S_{dens}^{(i)})$ \hfill $\rhd$ Min-max normalization $\mathcal{N}$
    \STATE $\mathcal{S} \leftarrow \text{StratifiedSample}(\mathcal{D}_{clean}, S_{unified}, p)$ \hfill $\rhd$ see Algorithm 1 in \citet{ccs}
    
    \STATE \textbf{return} $\mathcal{S}$
\end{algorithmic}
\end{algorithm}

\newpage
\section{Ablations \& Component Analyses}

\subsection{Interplay of Local Persistence ($\alpha$) and Global Density ($\beta$)}
\label{Sec:appendix_density_persistence_ablation}

As discussed in \cref{sec:comprehensive_score}, relying exclusively on either global density or local persistence is insufficient for optimal coreset selection at high compression rates. In this section, we provide the extended empirical analysis validating the orthogonal nature of these metrics and demonstrating why their unified formulation (\cref{eq:toposcore}) is critical for model optimization.

\subsubsection{Multi-Objective Formulation and Metric Independence}
\label{sec:appendix_pearson}

Formally, our unified metric functions as a weighted sum-scalarization \citep{sum_scalar0, sum_scalar1} of two complementary objectives: global representativeness (density) and local structural informativeness (persistence). For this scalarization to yield a non-trivial Pareto front, the underlying objectives must be non-redundant. To empirically verify this independence, we conducted a joint distribution analysis by extracting latent representations from 8 diverse network architectures evaluated on CIFAR-100. For each architecture, we computed both scores for all training samples and calculated the Pearson correlation coefficient ($r$). As detailed in \cref{tab:distribution_corr} and 
\cref{fig:joint_dist}, the linear correlation is consistently minimal across all architectures, yielding an average of $r = 0.102$. This near-zero correlation quantitatively confirms that global manifold density and local persistence do not redundantly encode the same feature space properties. Consequently, combining them with strictly positive weights ($\alpha, \beta > 0$) guarantees a Pareto optimal coreset selection that strictly outperforms the weakly Pareto optimal endpoints of using either metric in isolation.

\begin{table}[h!]
\caption{Pearson correlation ($r$) between density and persistence scores.}
\centering
\resizebox{0.95\linewidth}{!}
{
\begin{tabular}{lccccccccc}\toprule
& RN-18 & RN-50 & RN-101 & EffNet-B0 & EffNetV2-M & SwinV2-T & SwinV2-B & ViT-L-16 & \textbf{Avg.} \\

\midrule
$r$ & 0.117 & 0.106 & 0.076 & 0.102 & 0.110 & 0.140 & 0.104 & 0.061 & \textbf{0.102} \\

\bottomrule
\end{tabular}

}
\label{tab:distribution_corr}
\end{table}

\begin{figure}[!h]
    \vspace{-\baselineskip}
    \centering
    \begin{subfigure}[b]{0.24\textwidth}
        \centering
        \includegraphics[width=\textwidth]{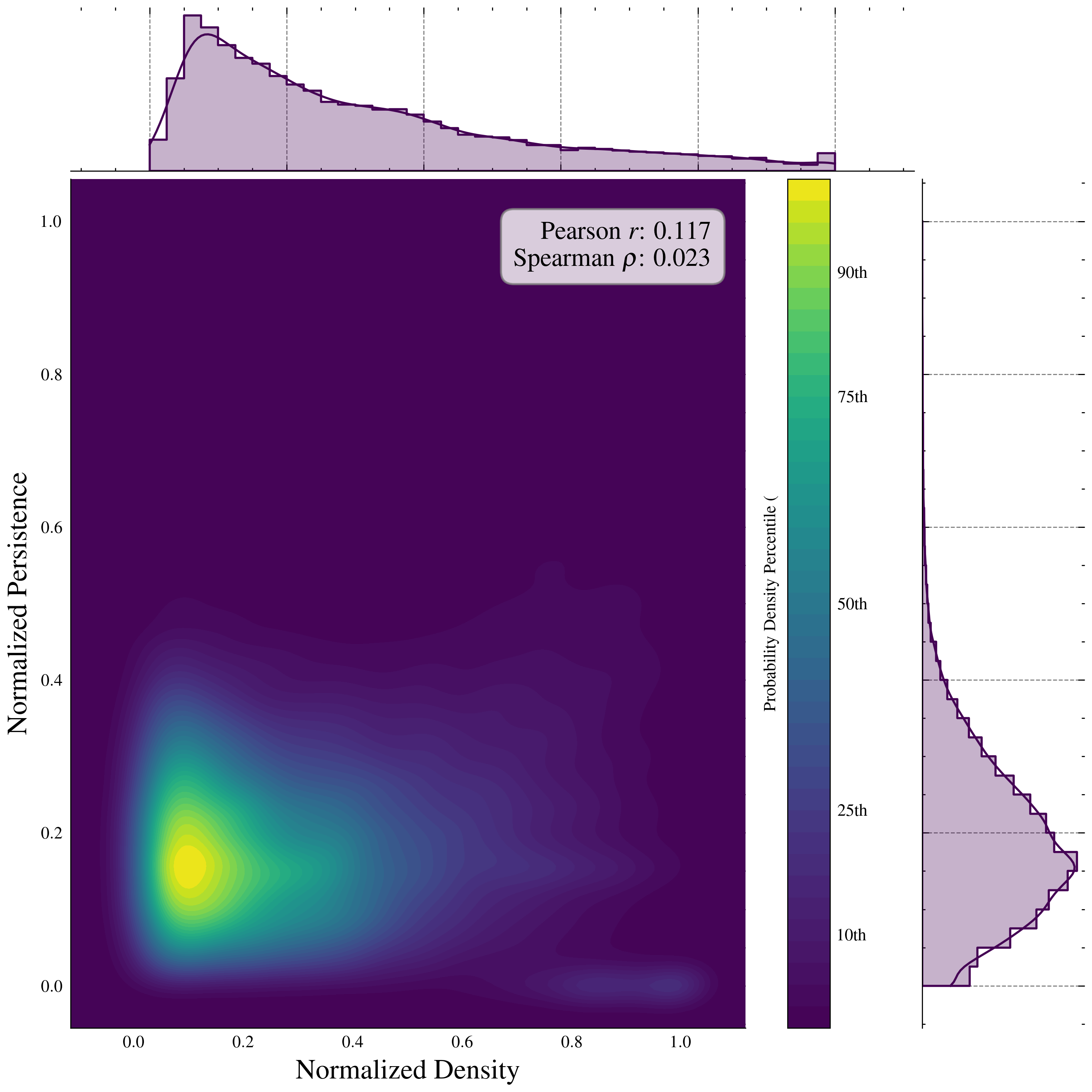}
        \caption{ResNet-18}
    \end{subfigure}
    \begin{subfigure}[b]{0.24\textwidth}
        \centering
        \includegraphics[width=\textwidth]{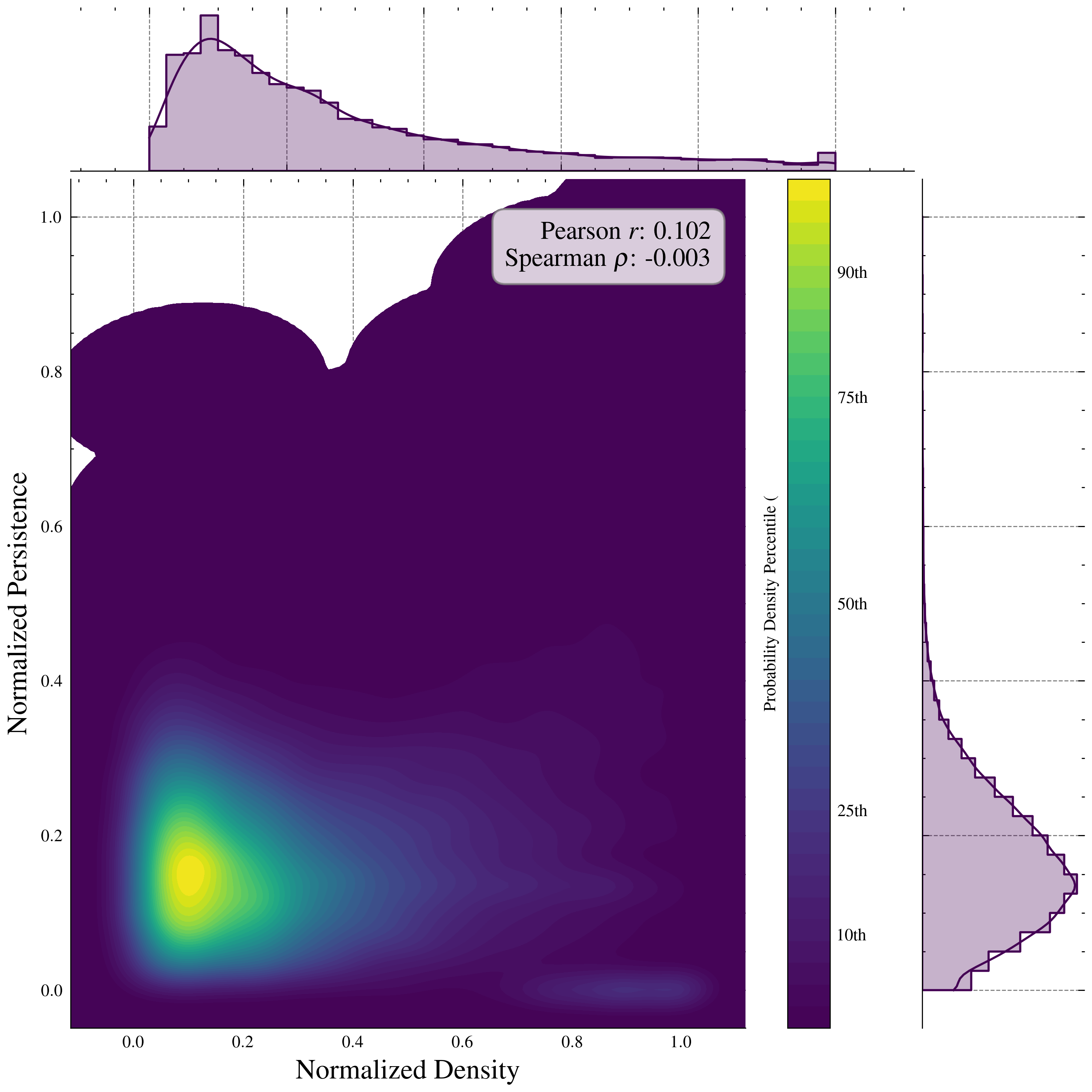}
        \caption{EfficientNet-B0}
    \end{subfigure}
    \begin{subfigure}[b]{0.24\textwidth}
        \centering
        \includegraphics[width=\textwidth]{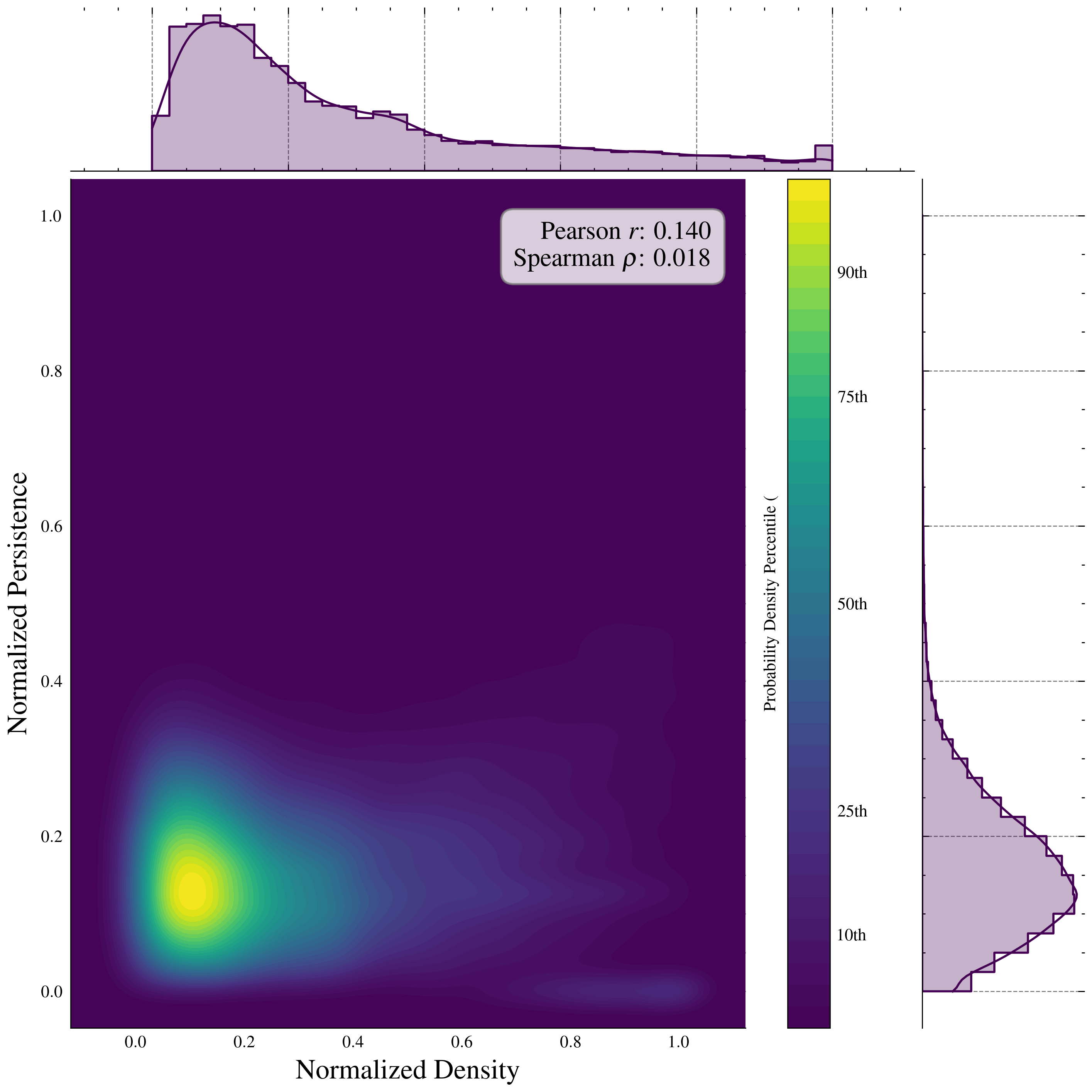}
        \caption{SwinV2-T}
    \end{subfigure}
    \begin{subfigure}[b]{0.24\textwidth}
        \centering
        \includegraphics[width=\textwidth]{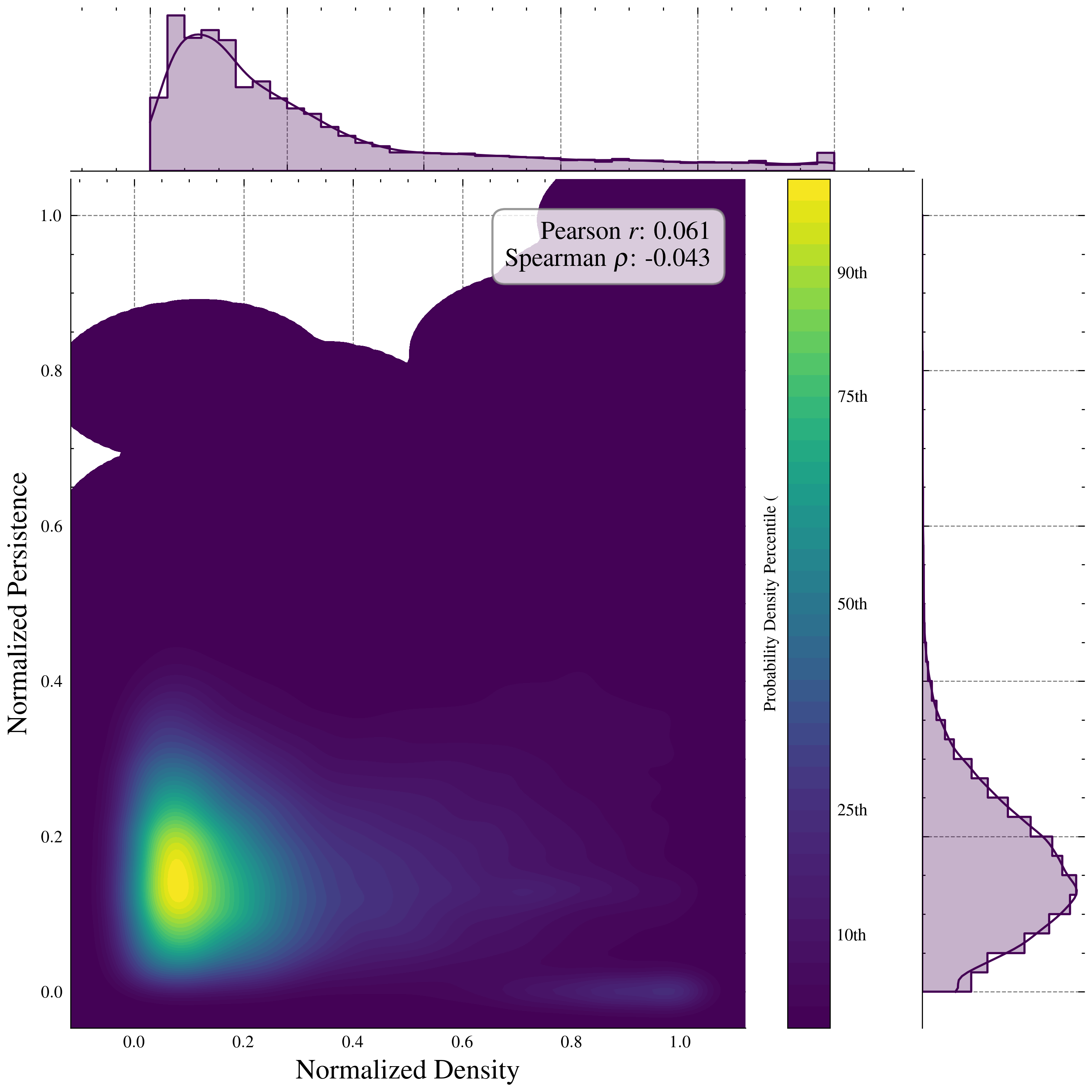}
        \caption{ViT-L-16}
    \end{subfigure}

    \caption{Joint probability distributions of persistence and density scores across networks. As the joint distributions are similar across architectures (all long tailed) we provide visualizations for just one architecture per network family.}
    \label{fig:joint_dist}

\end{figure}

\subsubsection{Qualitative Visualization across Networks}
\label{sec:appendix_visualizations}
To illustrate the importance of fine-grained local neighborhood structure, we visualize the global density (KDE), local persistence, and our unified score of a single target class. This qualitative comparison demonstrates how combining these metrics effectively resolves the redundancy problem inherent among densely co-located samples.
\begin{itemize}[itemsep=1mm, parsep=0pt, topsep=0mm]
    \item \textit{Density (\cref{fig:viz_kde_label14}):} Exhibits smooth, continuous gradients. Samples situated tightly together in the dense core of the cluster share nearly identical scores. If a coreset algorithm samples purely based on high density, it will select highly redundant points, starving the model of boundary information.
    \item \textit{Persistence (\cref{fig:viz_pers_label14}):} Exhibits a highly localized, non-continuous distribution. Within the high-density core, persistence highlights specific structural "anchors" while assigning low scores to immediately adjacent, redundant neighbors.
    \item \textit{Unified Score (\cref{fig:viz_unified_label14}):} Effectively overlays the fine-grained structural map onto the global distributional map. 
\end{itemize}

\begin{figure}[!tbp]
    \centering
    \begin{subfigure}[b]{0.2\textwidth}
        \centering
        \includegraphics[width=\textwidth]{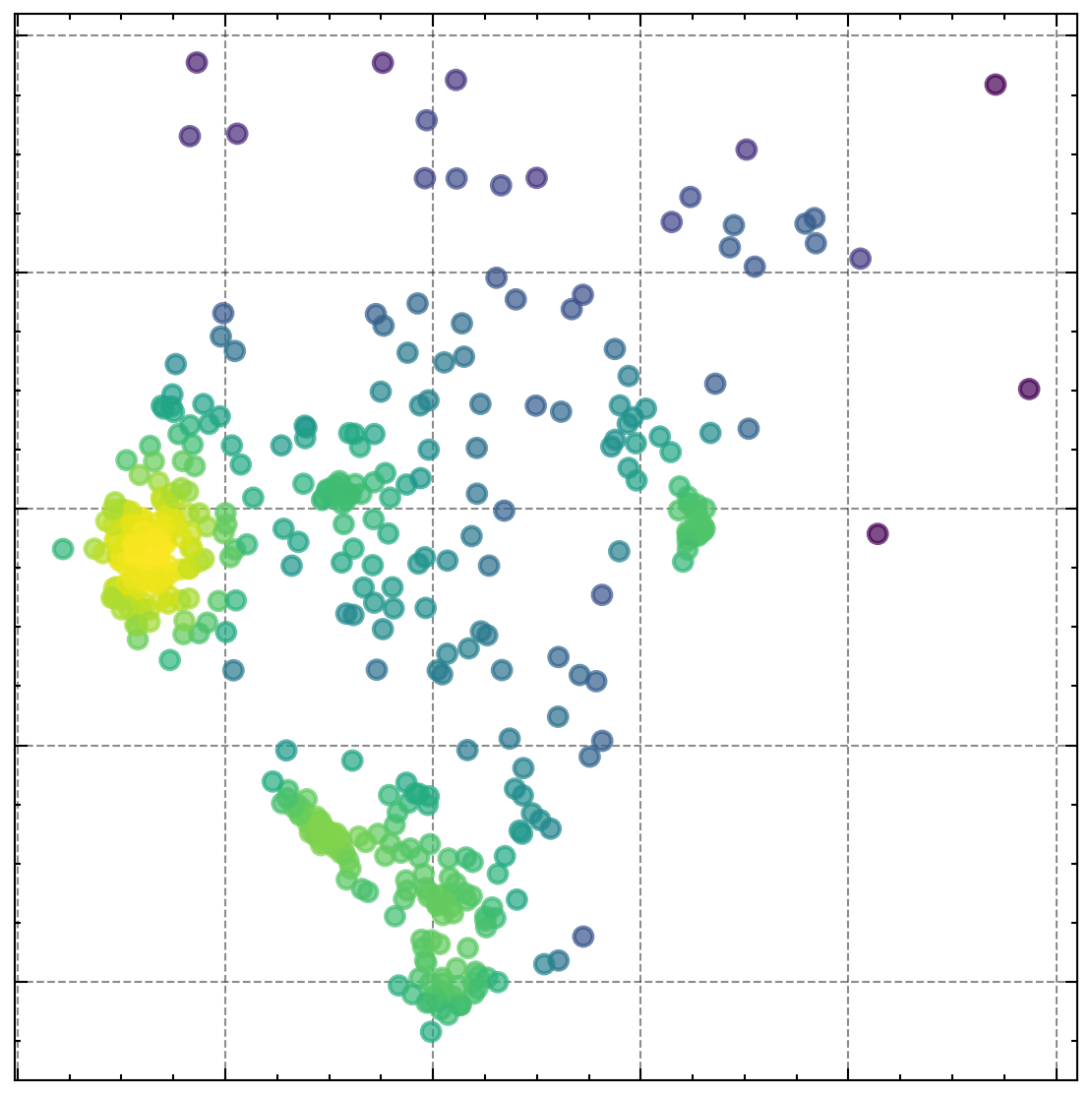}
        \caption{ResNet-18}
    \end{subfigure}
    \hspace{0.01\textwidth}
    \begin{subfigure}[b]{0.2\textwidth}
        \centering
        \includegraphics[width=\textwidth]{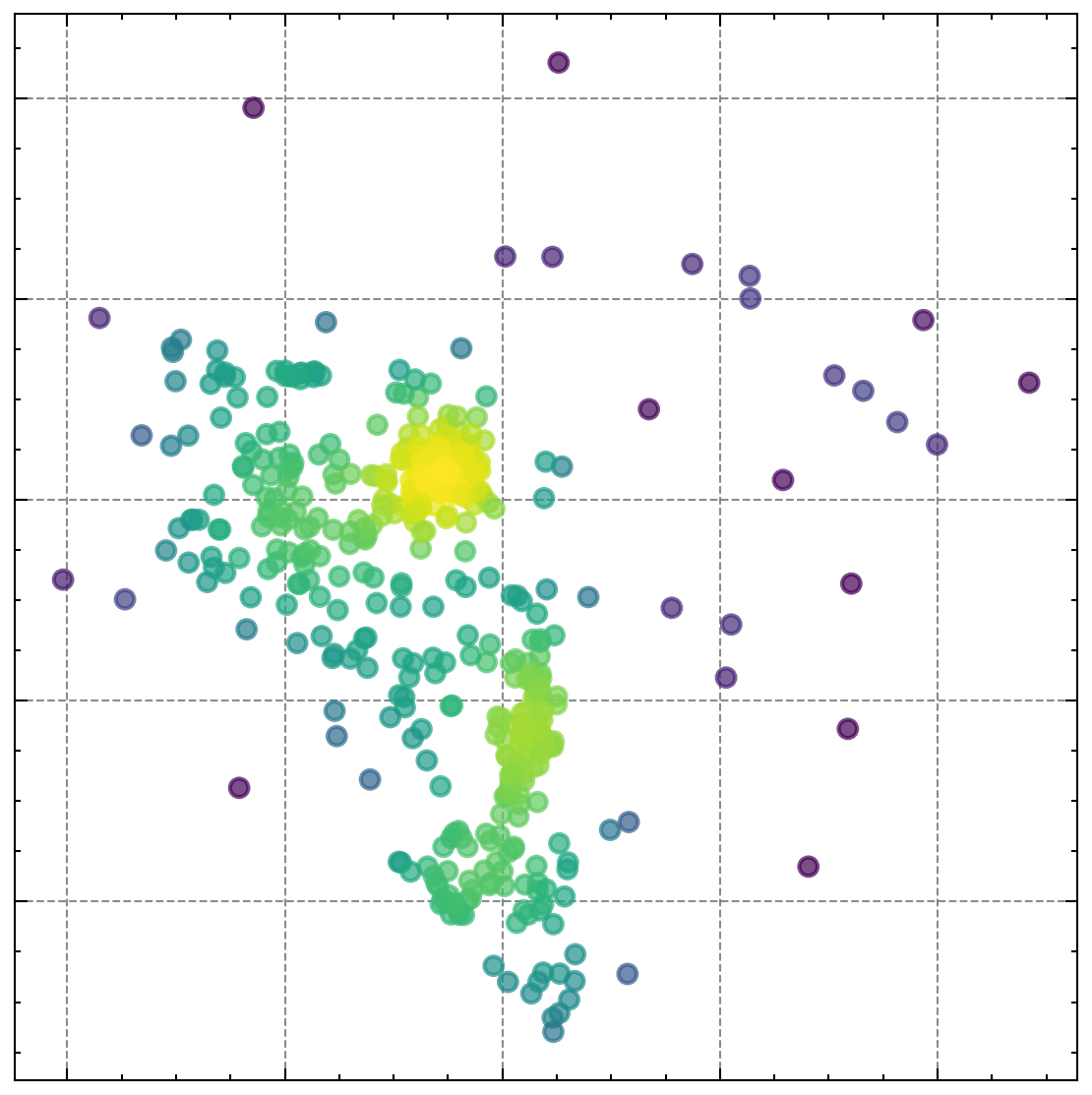}
        \caption{ResNet-50}
    \end{subfigure}
    \hspace{0.01\textwidth}
    \begin{subfigure}[b]{0.2\textwidth}
        \centering
        \includegraphics[width=\textwidth]{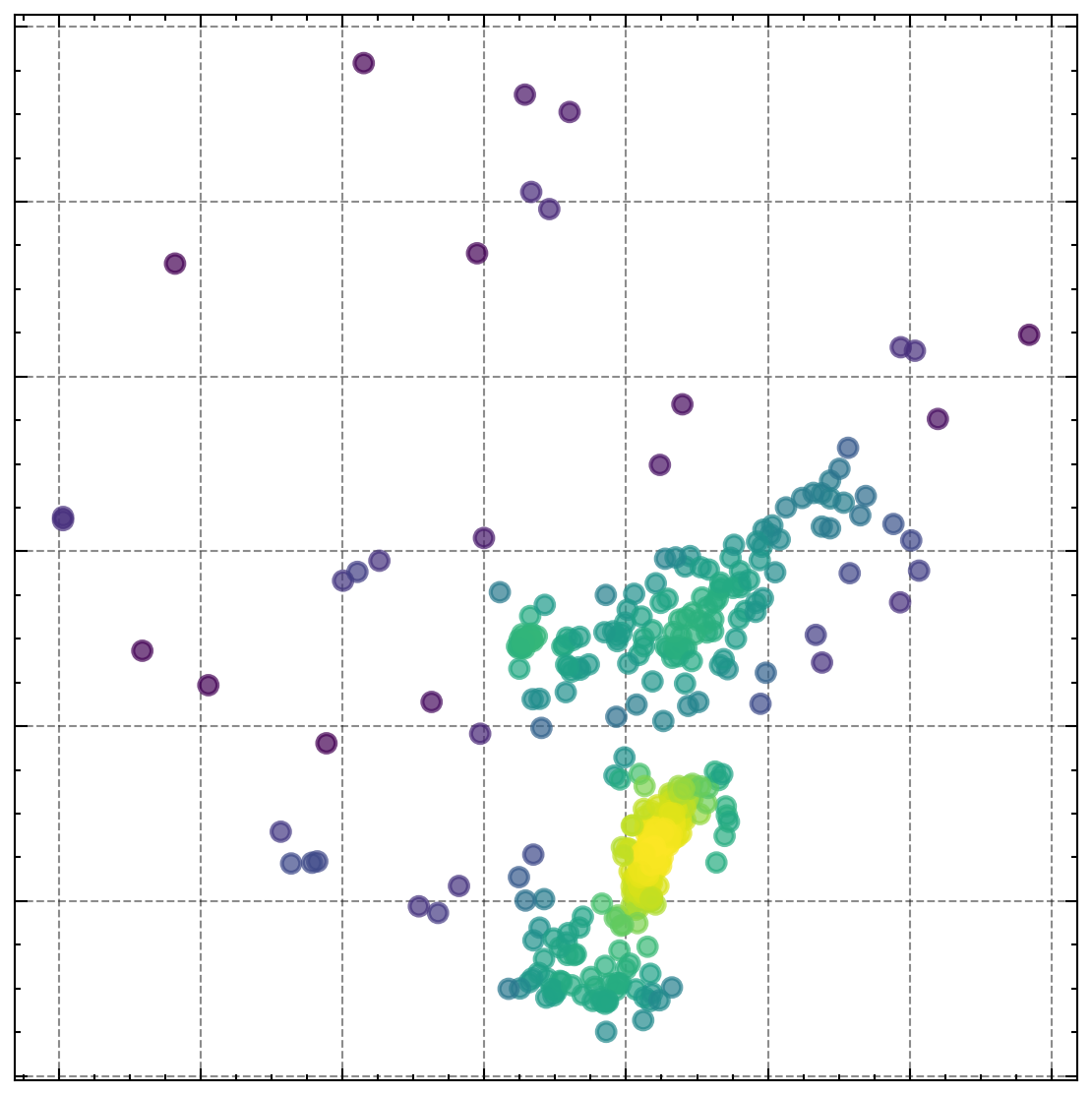}
        \caption{ResNet-101}
    \end{subfigure}
    \hspace{0.01\textwidth}
    \begin{subfigure}[b]{0.235\textwidth}
        \centering
        \includegraphics[width=\textwidth]{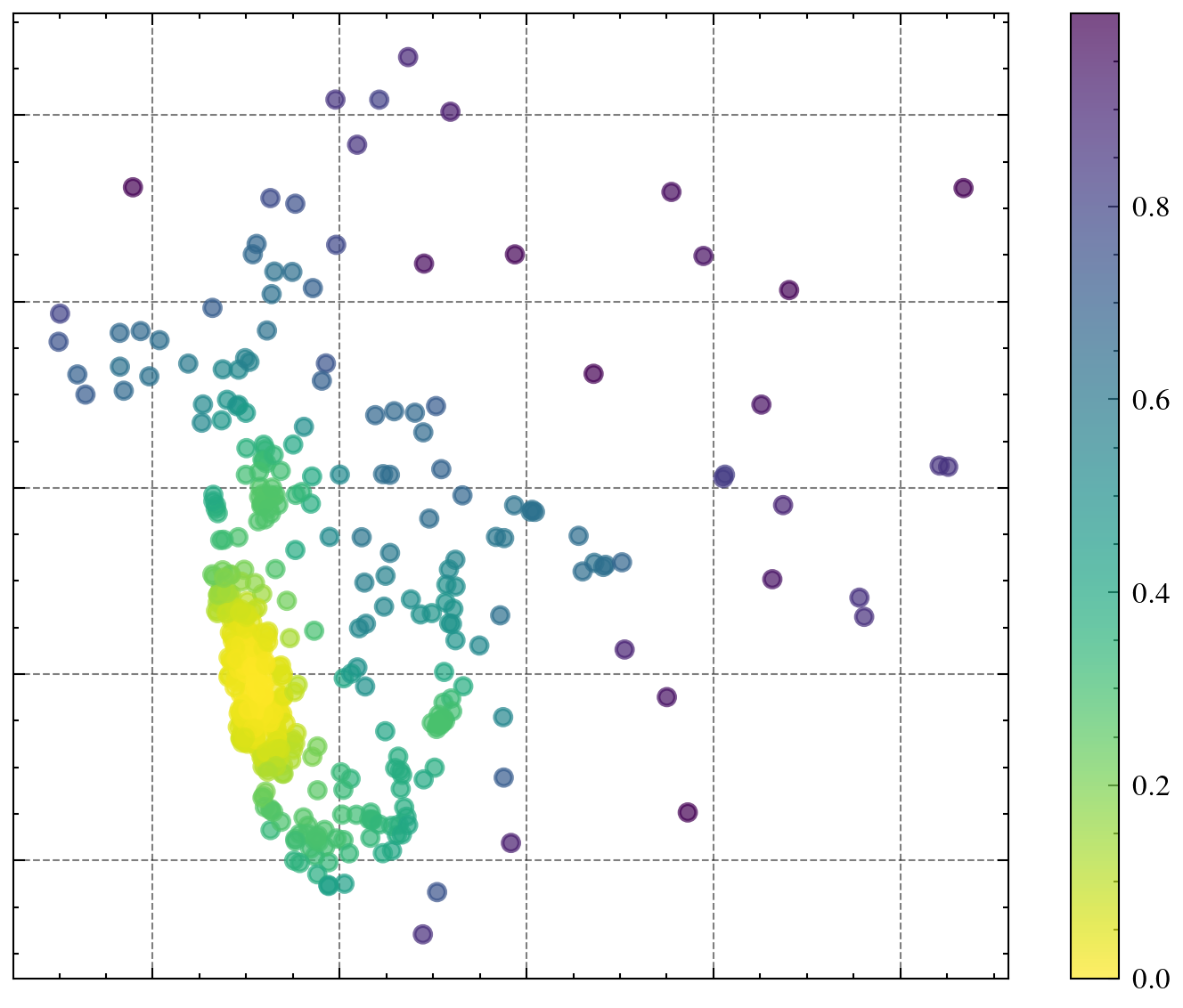}
        \caption{EfficientNet-B0}
    \end{subfigure}

    \vspace{0.01\textwidth}
    \begin{subfigure}[b]{0.2\textwidth}
        \centering
        \includegraphics[width=\textwidth]{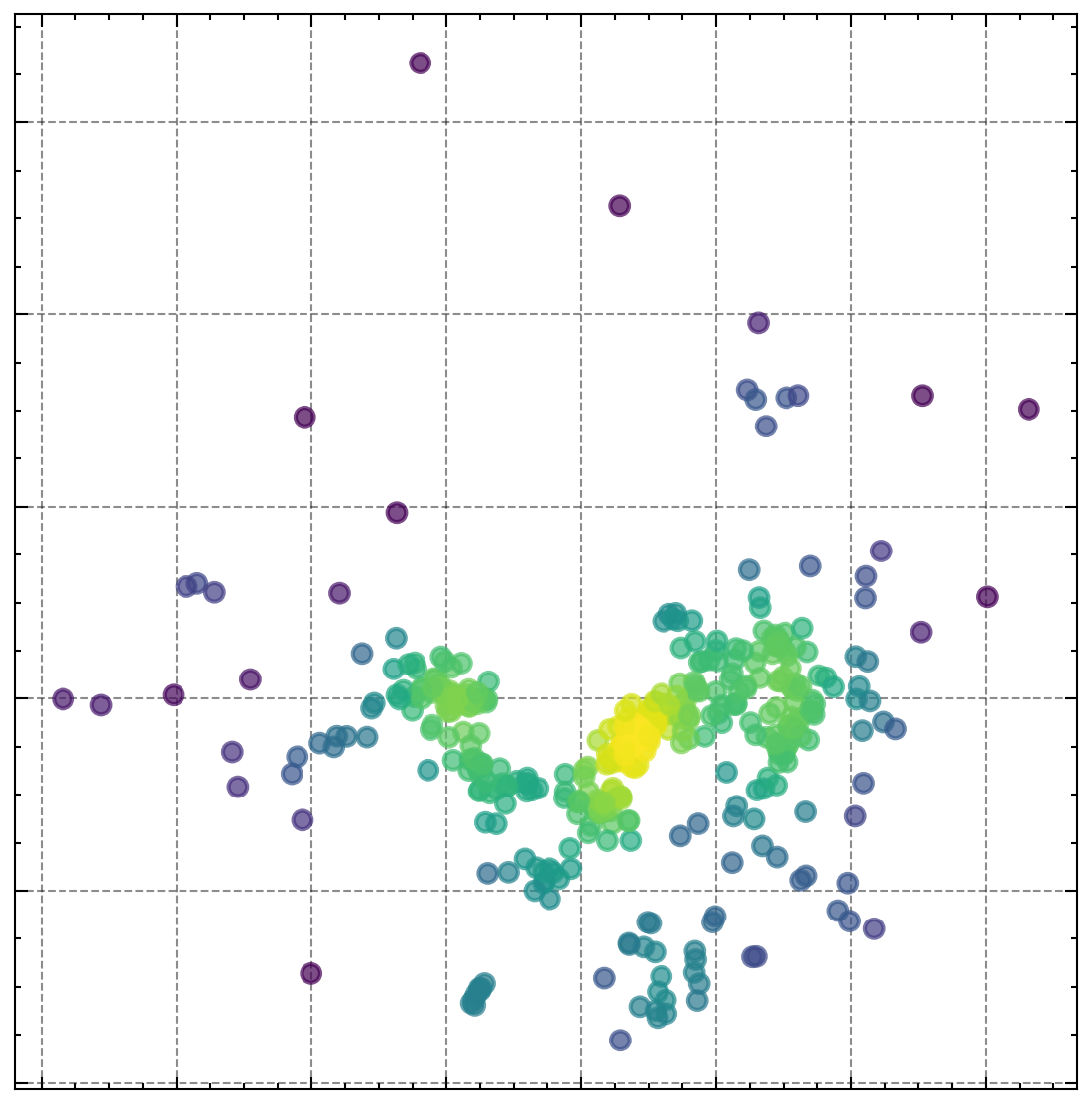}
        \caption{EffNetV2-M}
    \end{subfigure}
    \hspace{0.01\textwidth}
    \begin{subfigure}[b]{0.2\textwidth}
        \centering
        \includegraphics[width=\textwidth]{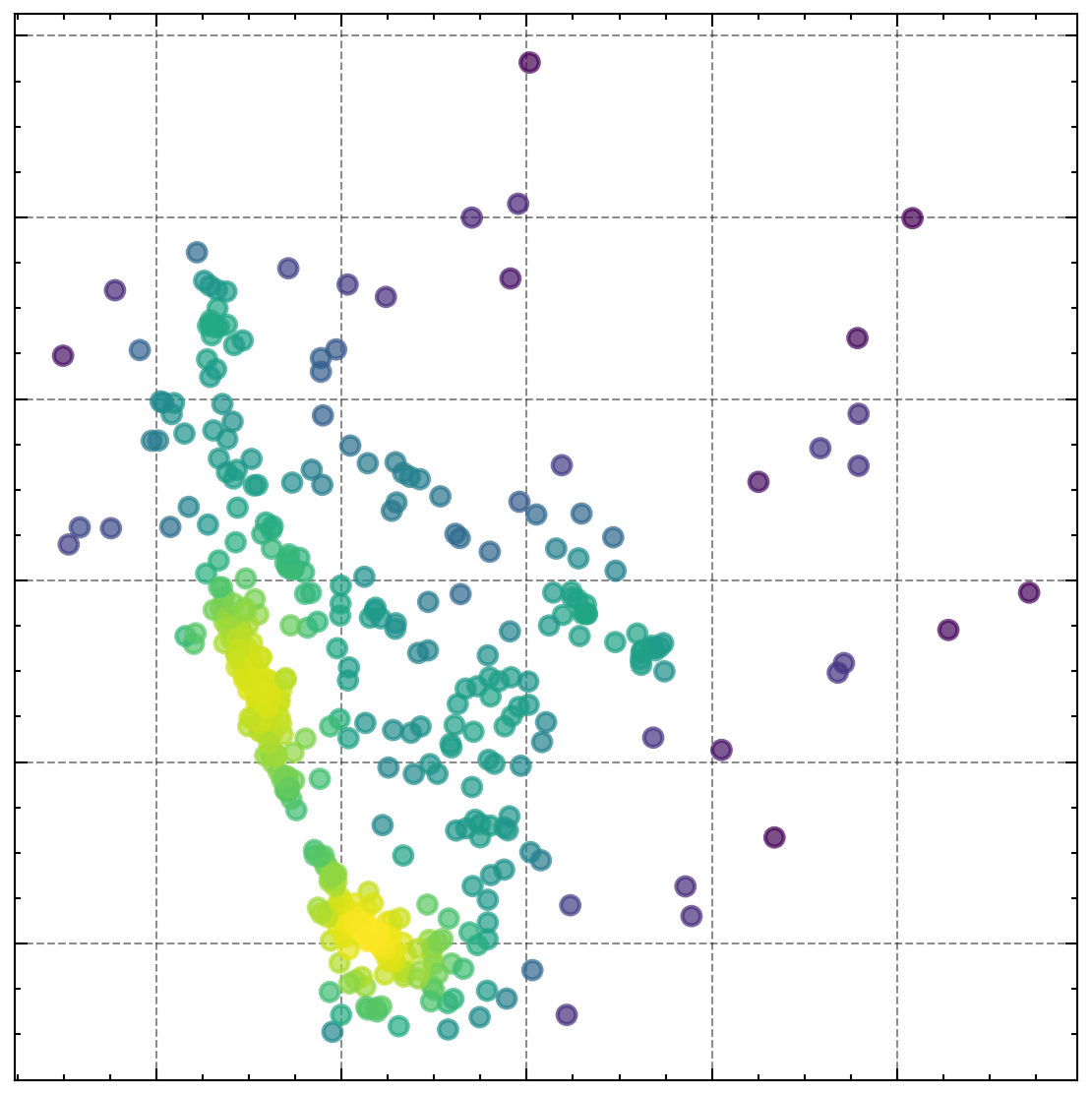}
        \caption{SwinV2-T}
    \end{subfigure}
    \hspace{0.01\textwidth}
    \begin{subfigure}[b]{0.2\textwidth}
        \centering
        \includegraphics[width=\textwidth]{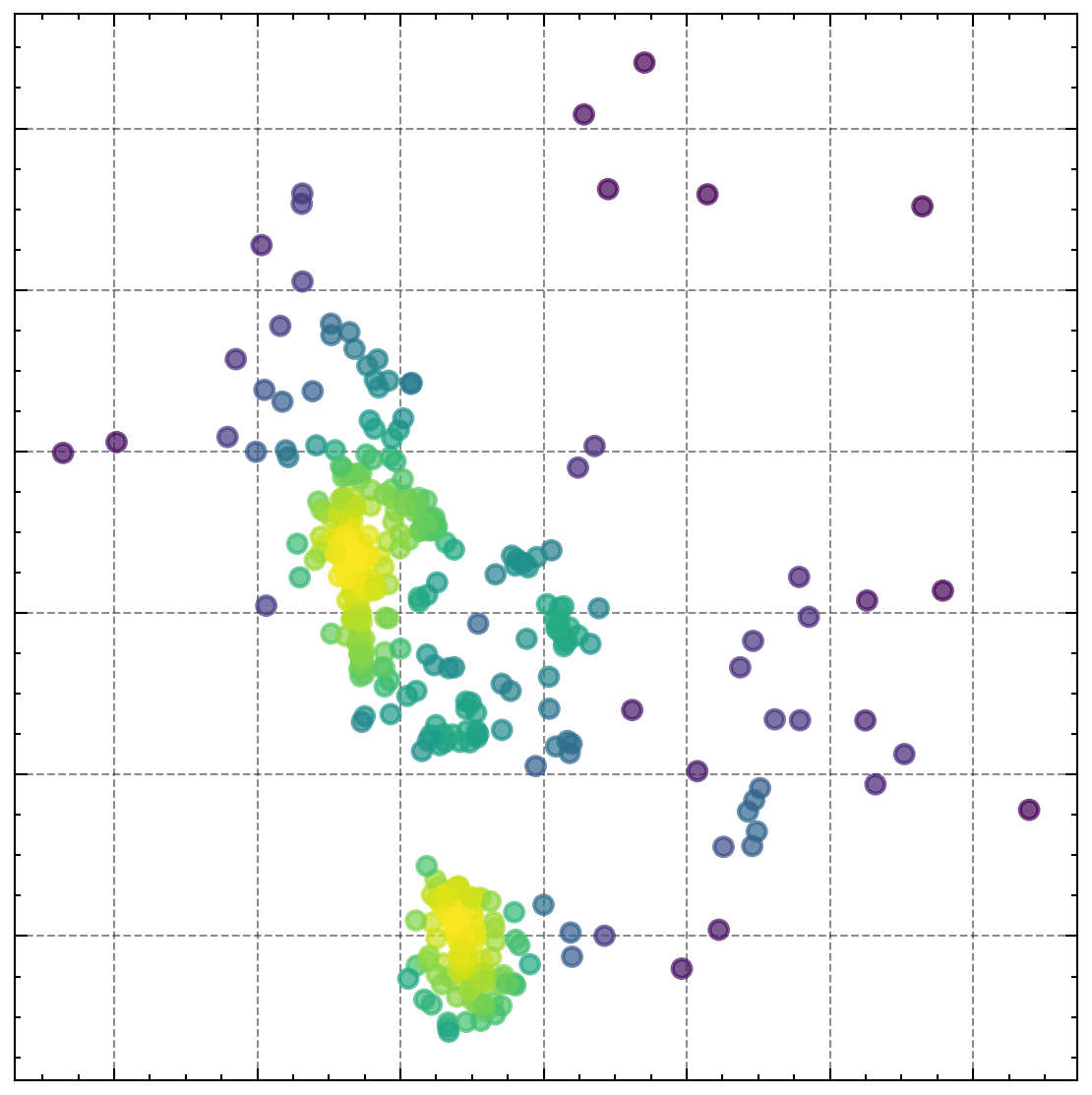}
        \caption{SwinV2-B}
    \end{subfigure}
    \hspace{0.01\textwidth}
    \begin{subfigure}[b]{0.2\textwidth}
        \centering
        \includegraphics[width=\textwidth]{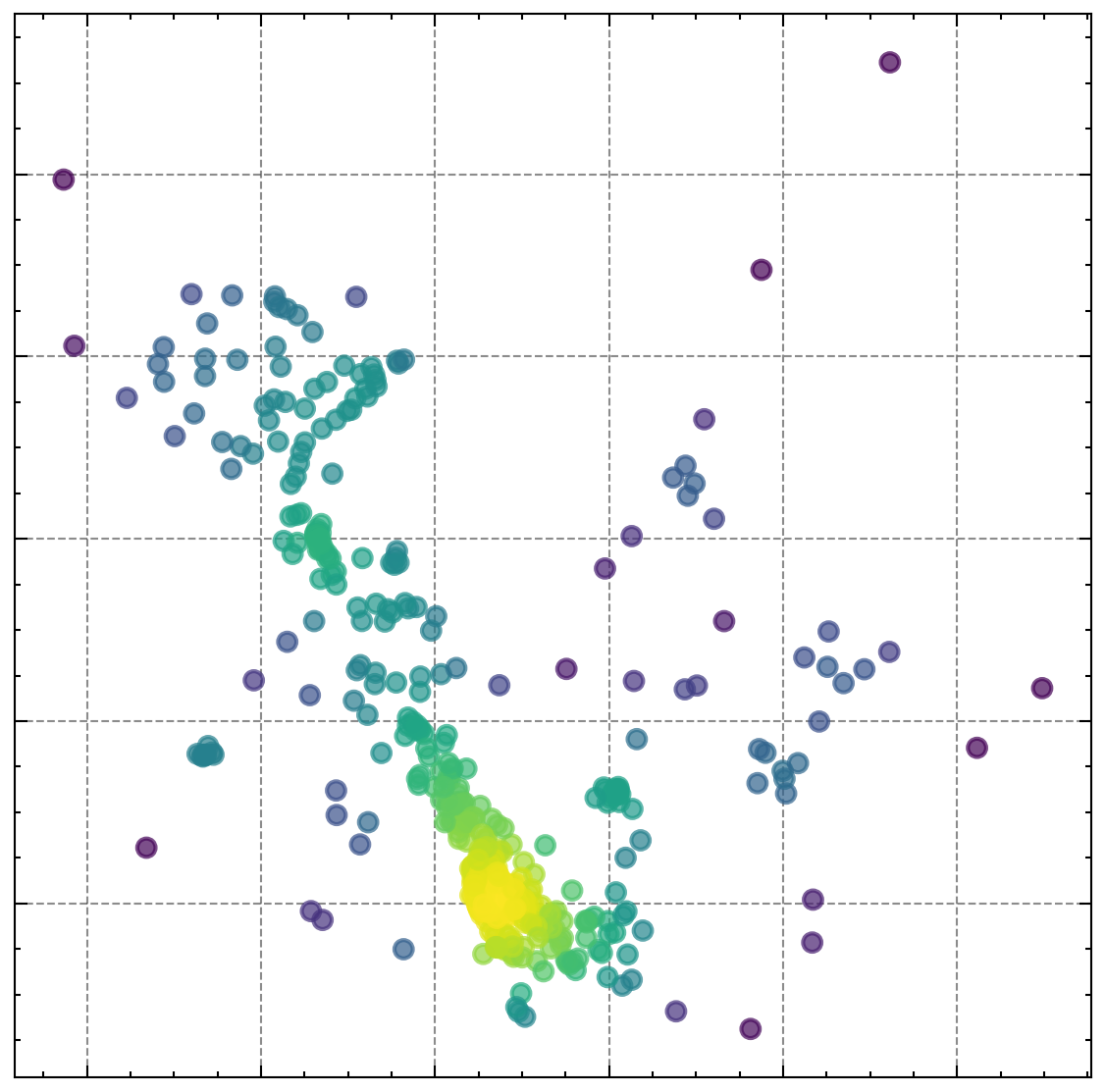}
        \caption{ViT-L-16}
    \end{subfigure}
    
    \caption{$\text{Score}_{\text{dens}}$ on Cifar-100 (butterfly class; label 14). Score [0, 1]: (0: Yellow indicated high density, 1: blue indicated low density).}
    \label{fig:viz_kde_label14}
\end{figure}

\begin{figure}[!tbp]
    \centering
    \begin{subfigure}[b]{0.2\textwidth}
        \centering
        \includegraphics[width=\textwidth]{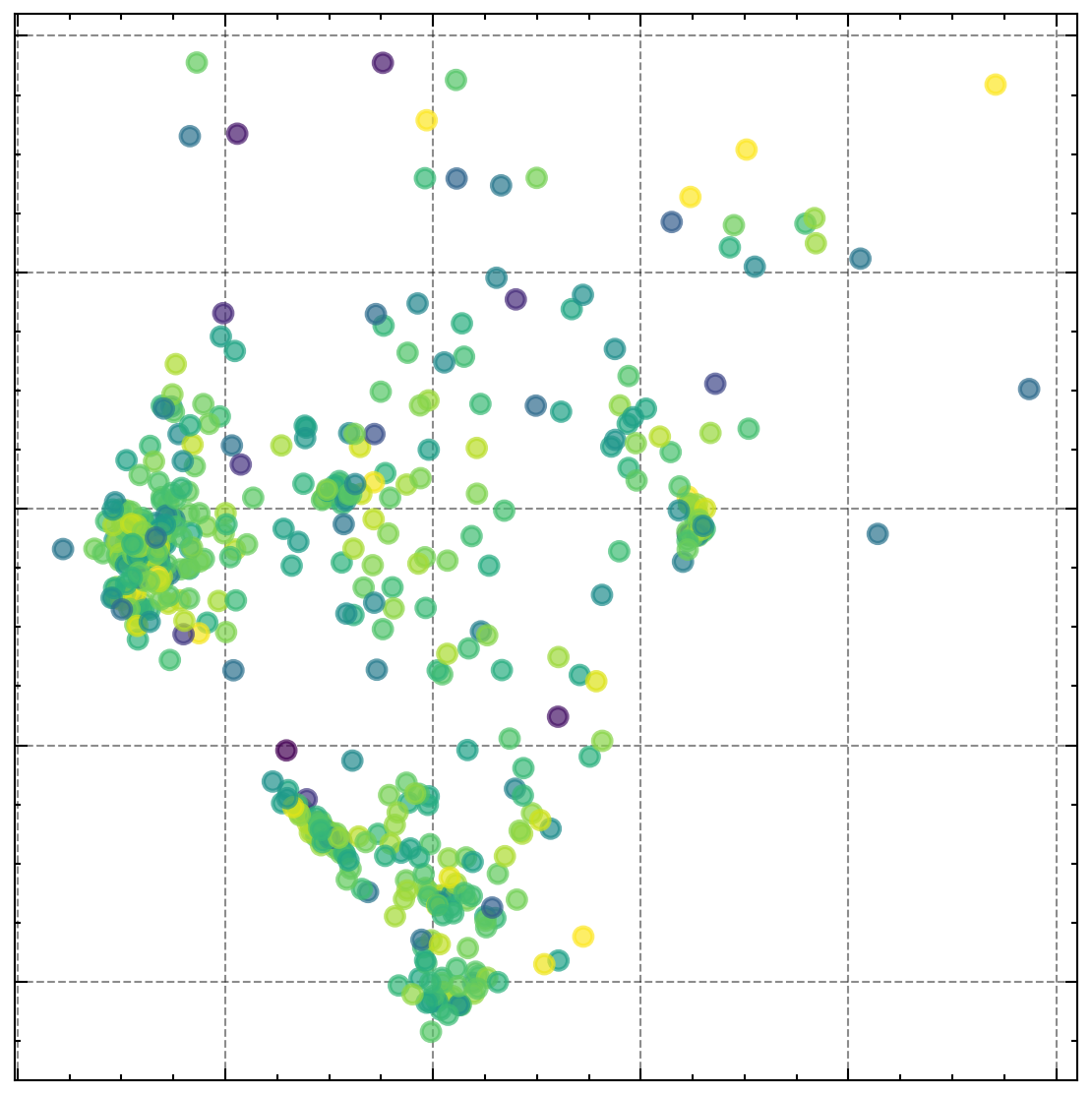}
        \caption{ResNet-18}
    \end{subfigure}
    \hspace{0.01\textwidth}
    \begin{subfigure}[b]{0.2\textwidth}
        \centering
        \includegraphics[width=\textwidth]{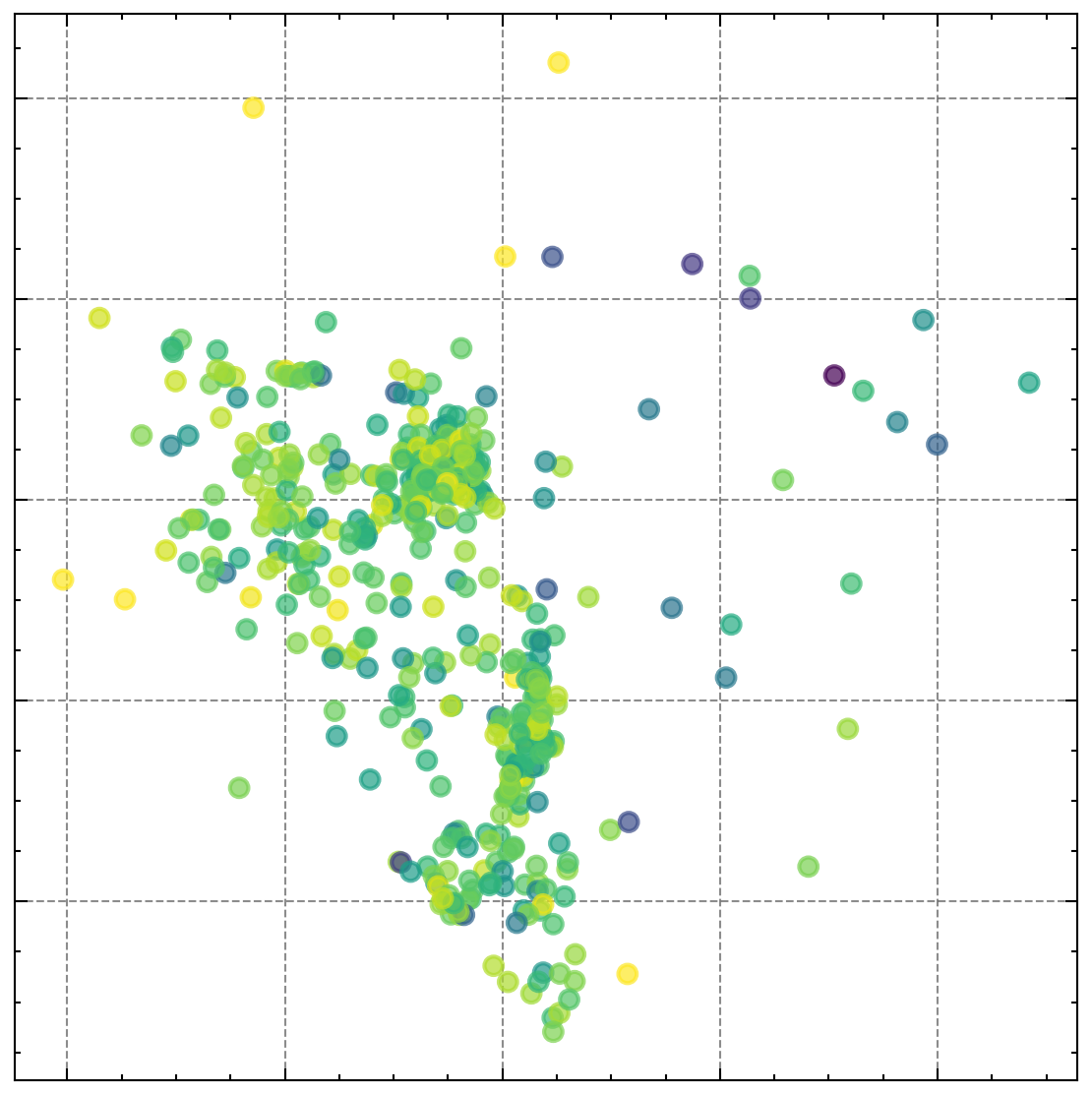}
        \caption{ResNet-50}
    \end{subfigure}
    \hspace{0.01\textwidth}
    \begin{subfigure}[b]{0.2\textwidth}
        \centering
        \includegraphics[width=\textwidth]{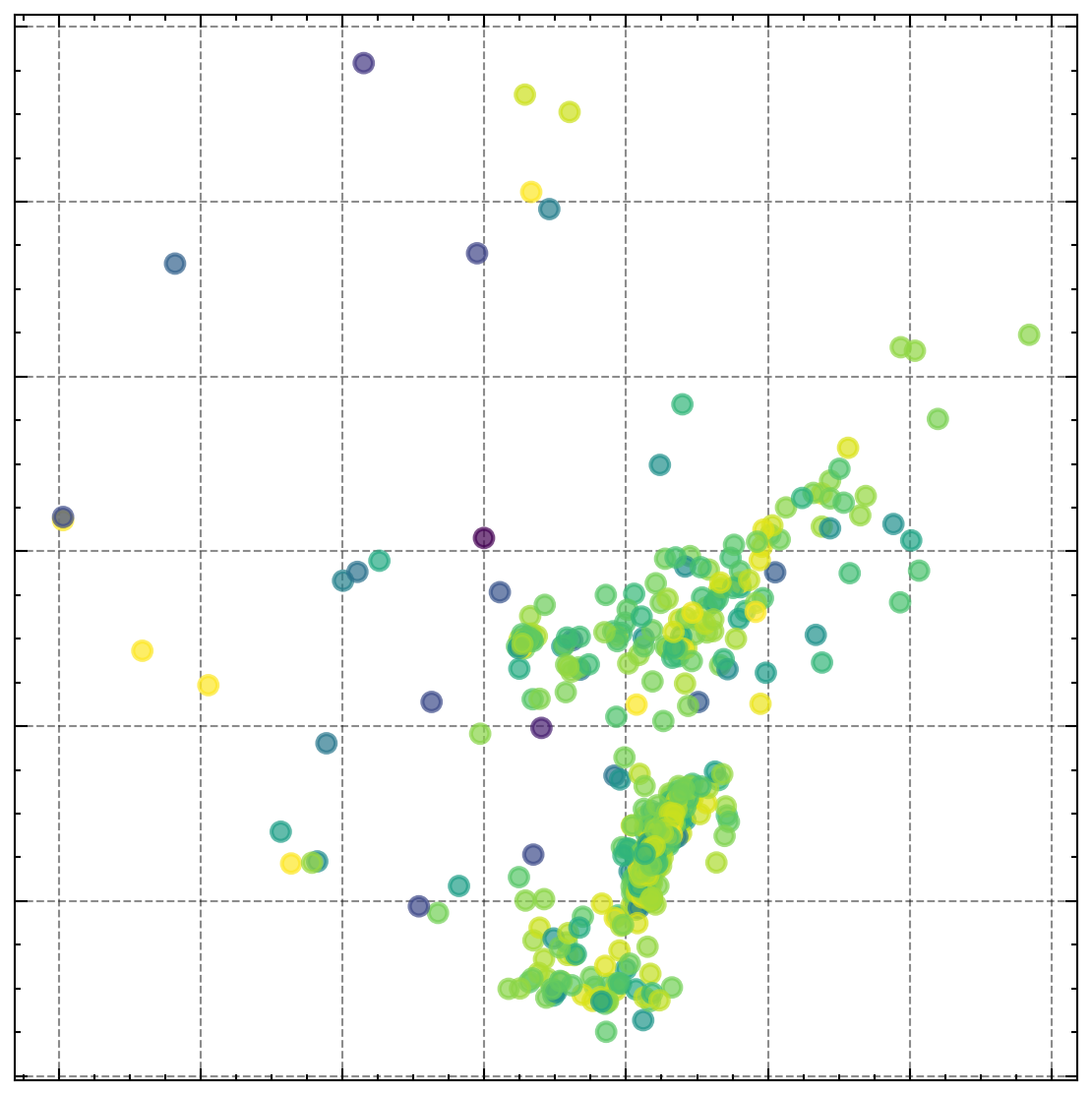}
        \caption{ResNet-101}
    \end{subfigure}
    \hspace{0.01\textwidth}
    \begin{subfigure}[b]{0.235\textwidth}
        \centering
        \includegraphics[width=\textwidth]{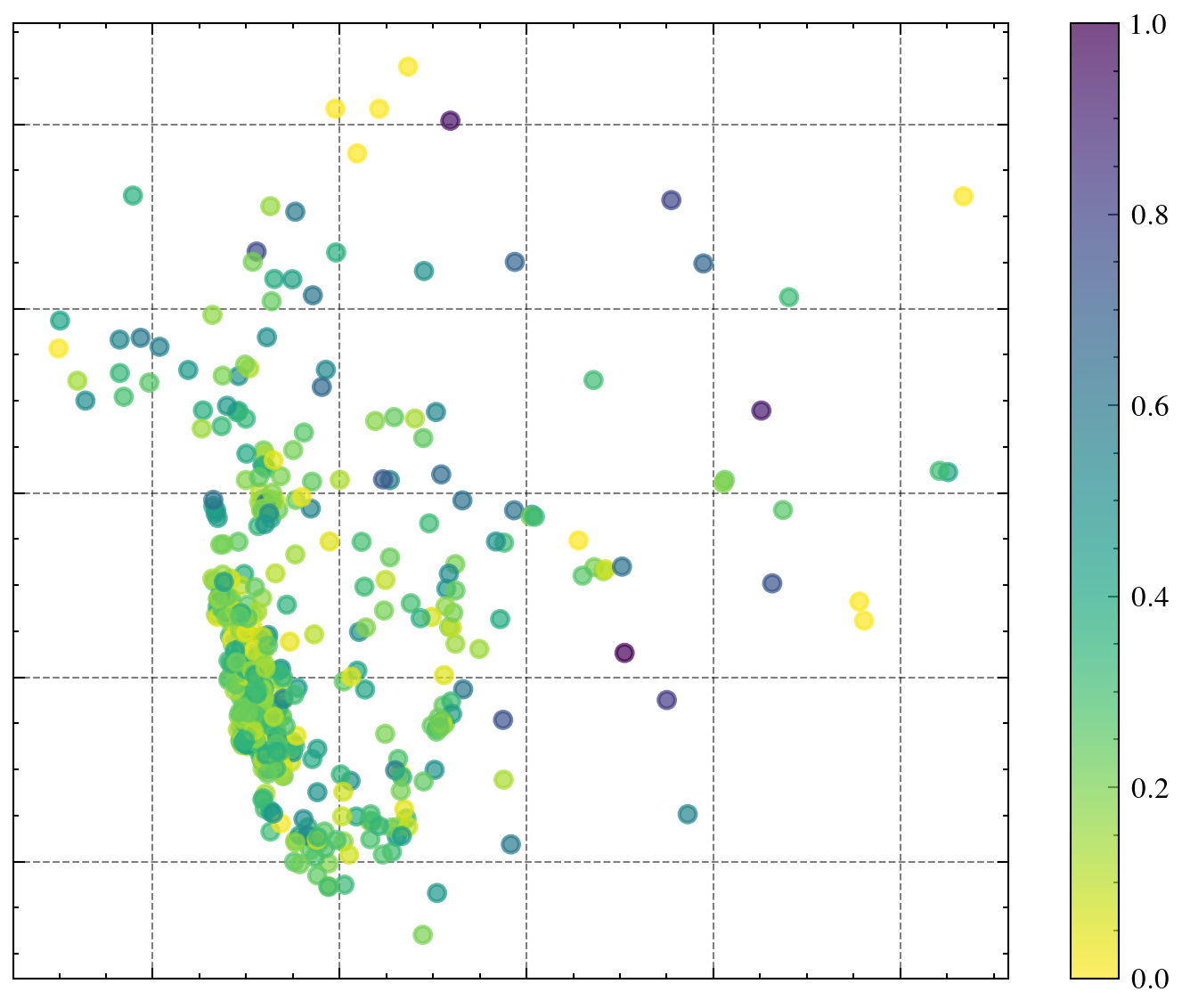}
        \caption{EfficientNet-B0}
    \end{subfigure}

    \vspace{0.01\textwidth}
    \begin{subfigure}[b]{0.2\textwidth}
        \centering
        \includegraphics[width=\textwidth]{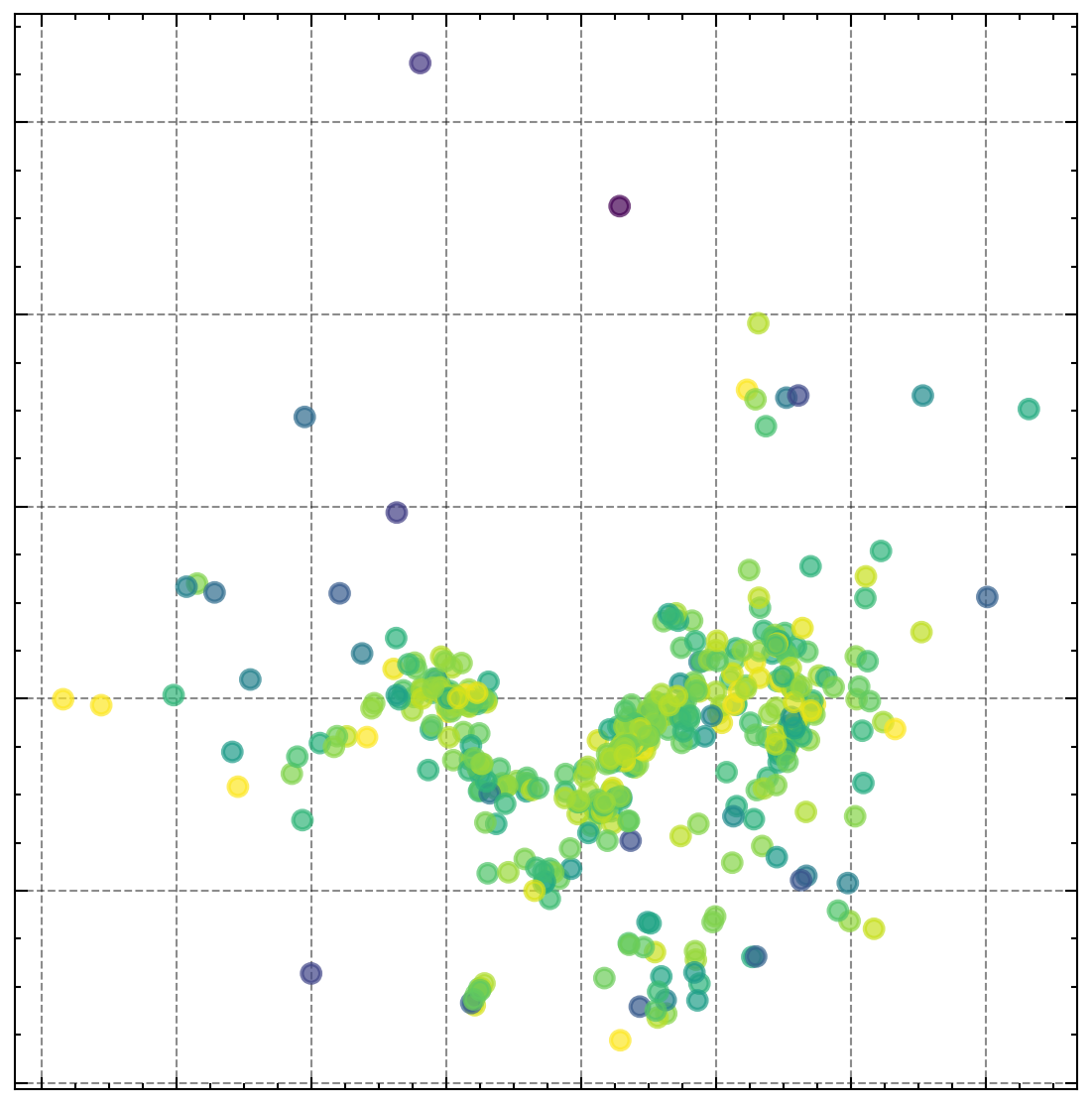}
        \caption{EffNetV2-M}
    \end{subfigure}
    \hspace{0.01\textwidth}
    \begin{subfigure}[b]{0.2\textwidth}
        \centering
        \includegraphics[width=\textwidth]{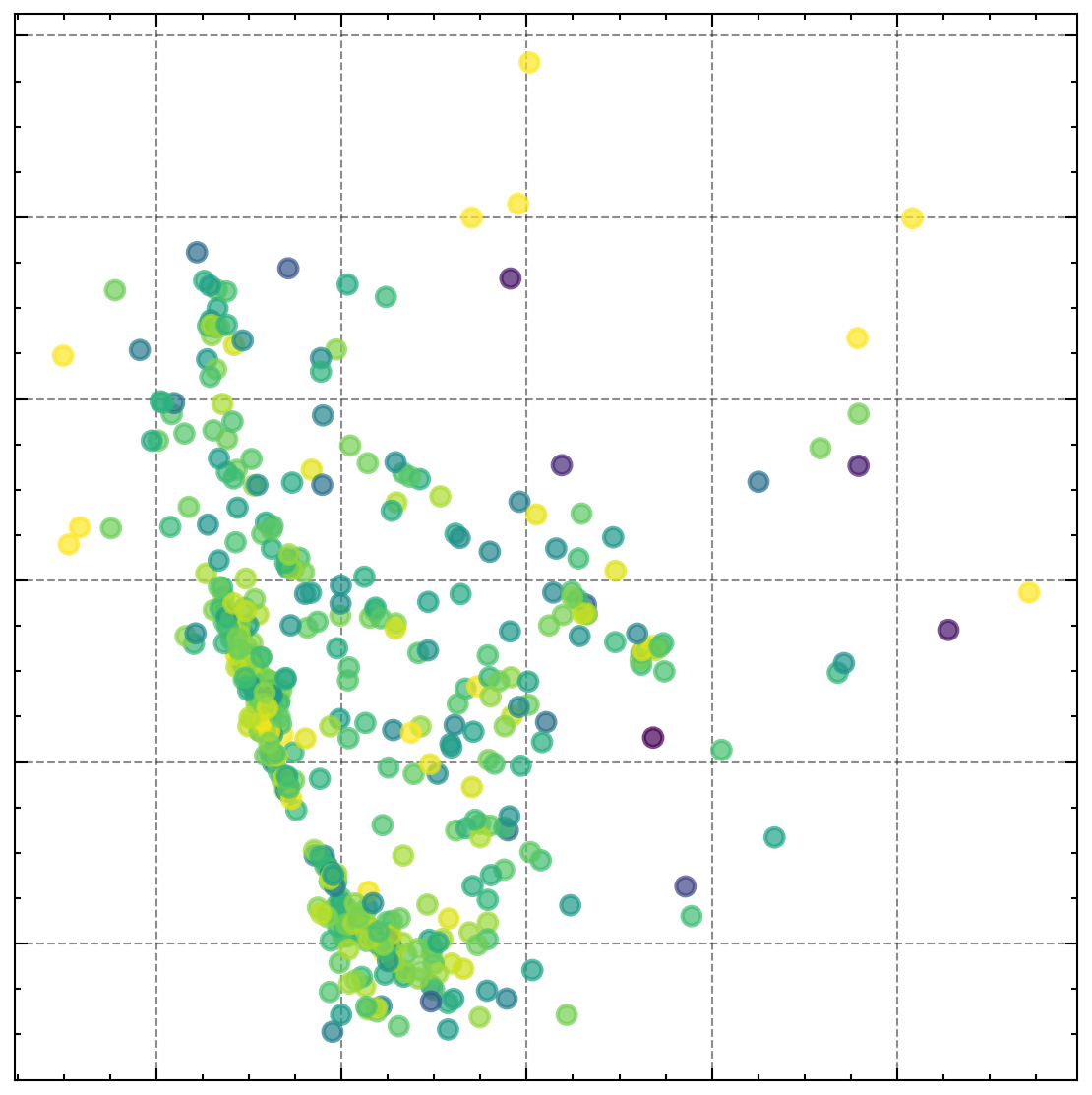}
        \caption{SwinV2-T}
    \end{subfigure}
    \hspace{0.01\textwidth}
    \begin{subfigure}[b]{0.2\textwidth}
        \centering
        \includegraphics[width=\textwidth]{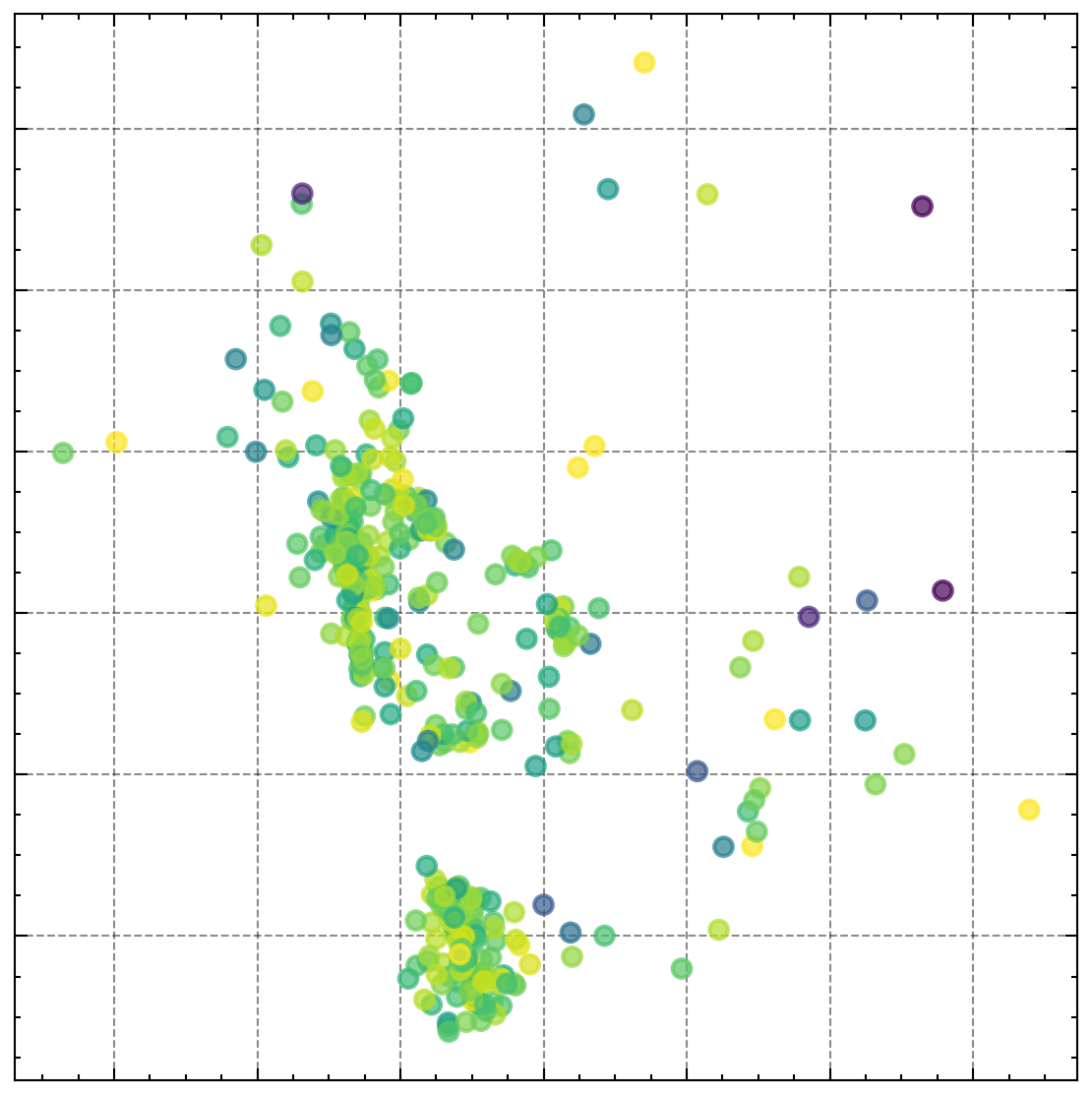}
        \caption{SwinV2-B}
    \end{subfigure}
    \hspace{0.01\textwidth}
    \begin{subfigure}[b]{0.2\textwidth}
        \centering
        \includegraphics[width=\textwidth]{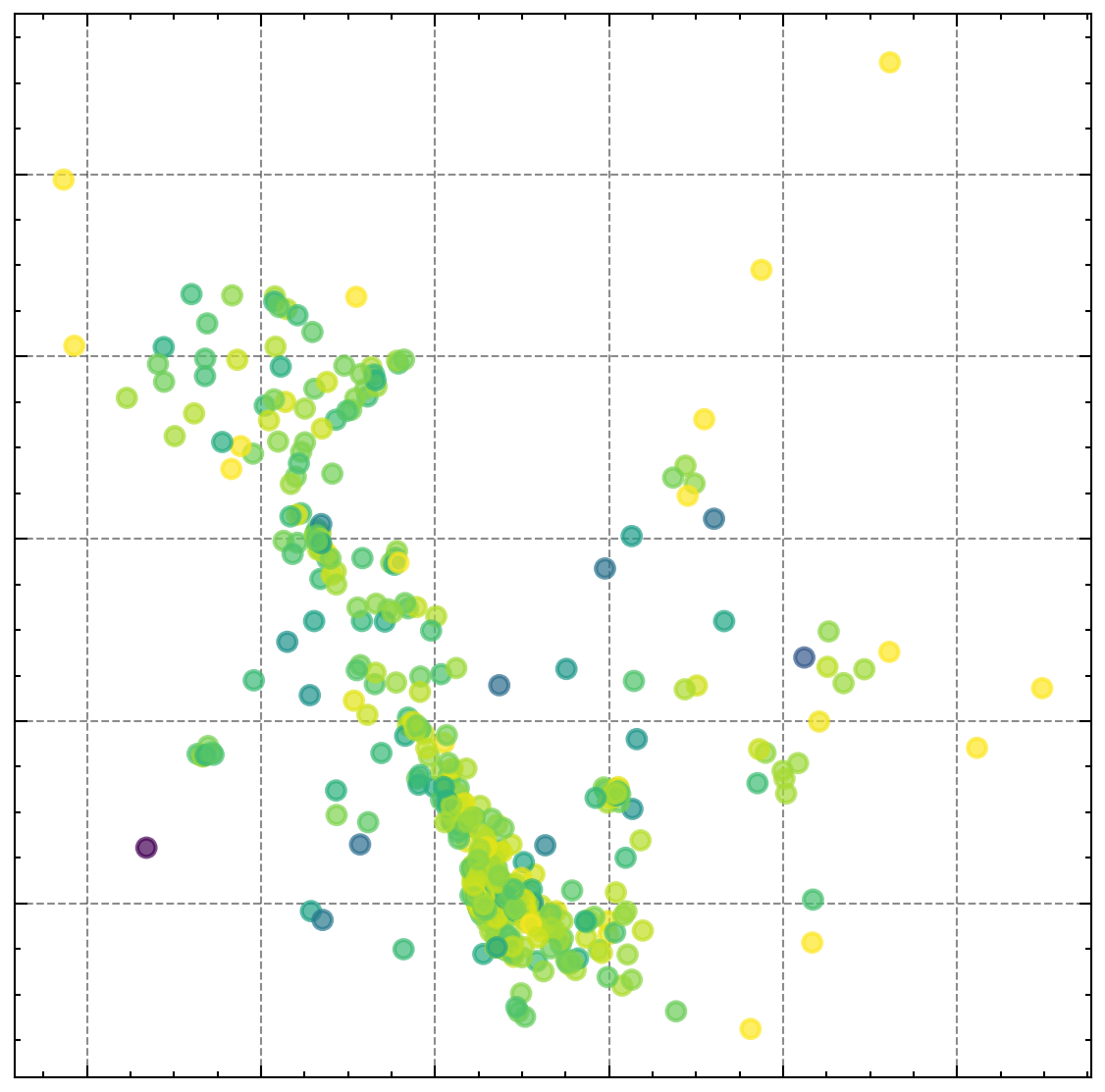}
        \caption{ViT-L-16}
    \end{subfigure}
    
    \caption{$\text{Score}_{\text{pers}}$ on Cifar-100 (butterfly class; label 14) with 6 optimization steps. Score [0, 1]: (0: Yellow = low persistence score (e.g., small optim. distance); 1: blue = higher persistence score (e.g., higher optim. distance).}
    \label{fig:viz_pers_label14}
\end{figure}

\begin{figure}[!tbp]
    \centering
    \begin{subfigure}[b]{0.2\textwidth}
        \centering
        \includegraphics[width=\textwidth]{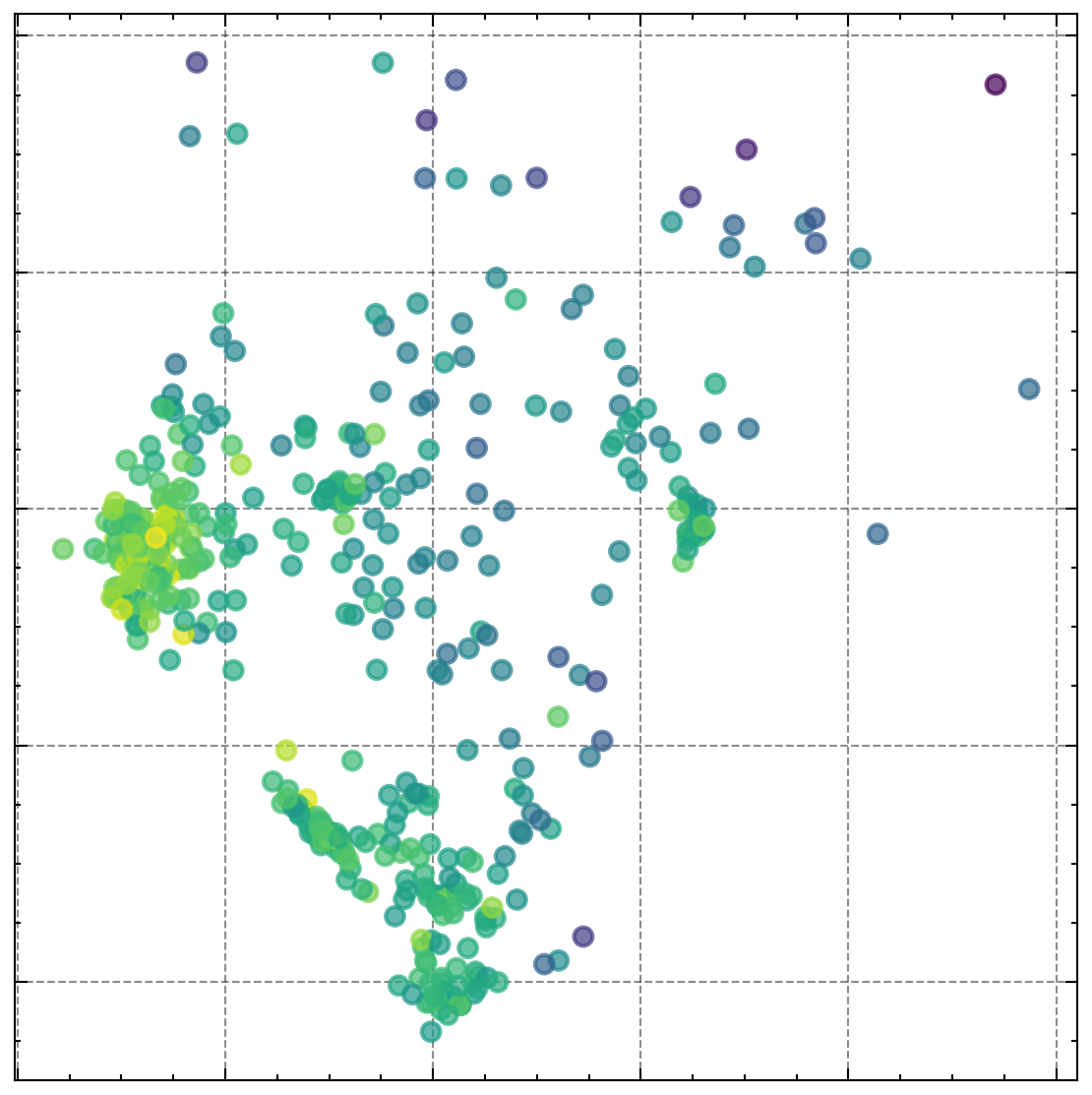}
        \caption{ResNet-18}
    \end{subfigure}
    \hspace{0.01\textwidth}
    \begin{subfigure}[b]{0.2\textwidth}
        \centering
        \includegraphics[width=\textwidth]{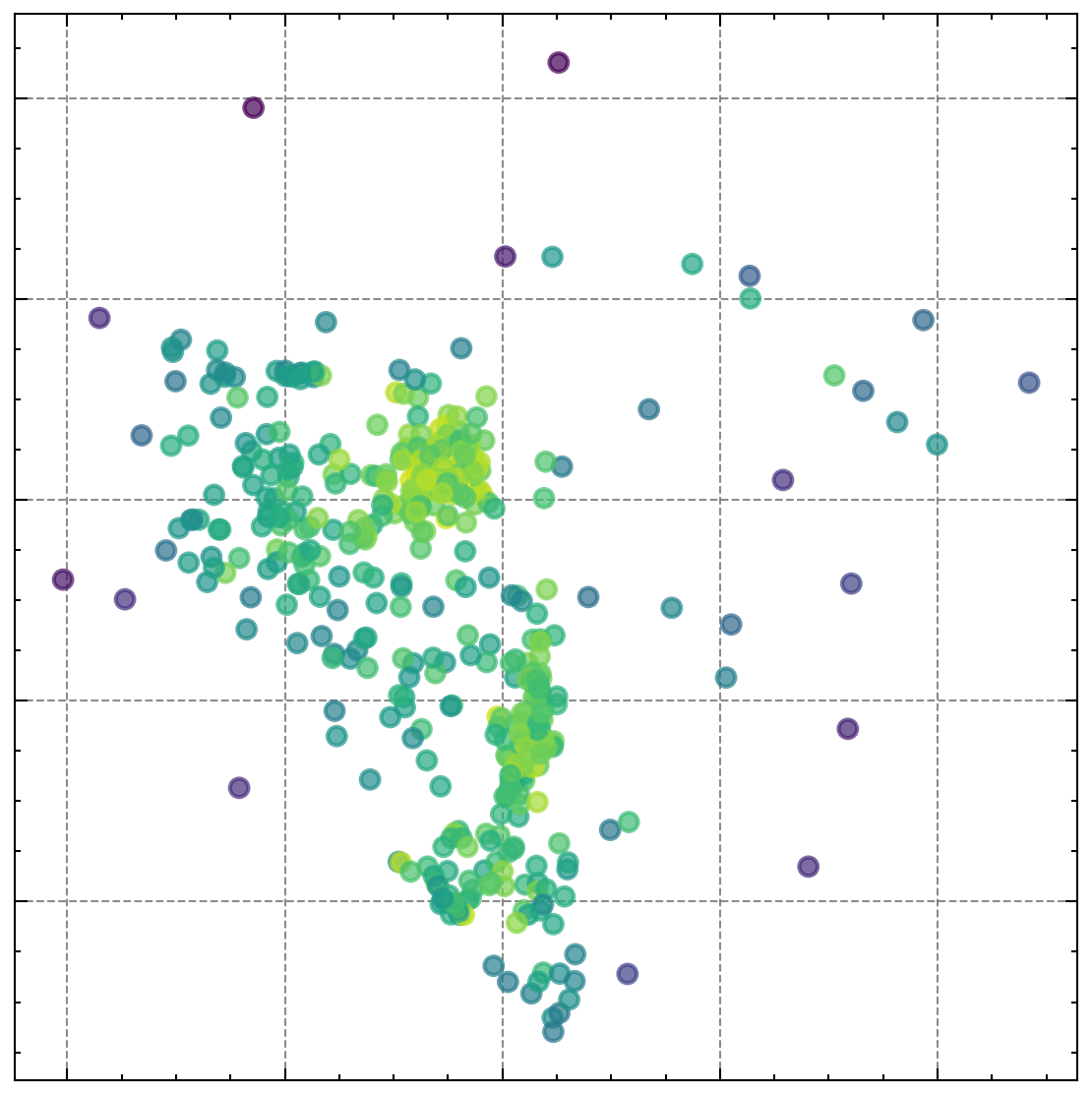}
        \caption{ResNet-50}
    \end{subfigure}
    \hspace{0.01\textwidth}
    \begin{subfigure}[b]{0.2\textwidth}
        \centering
        \includegraphics[width=\textwidth]{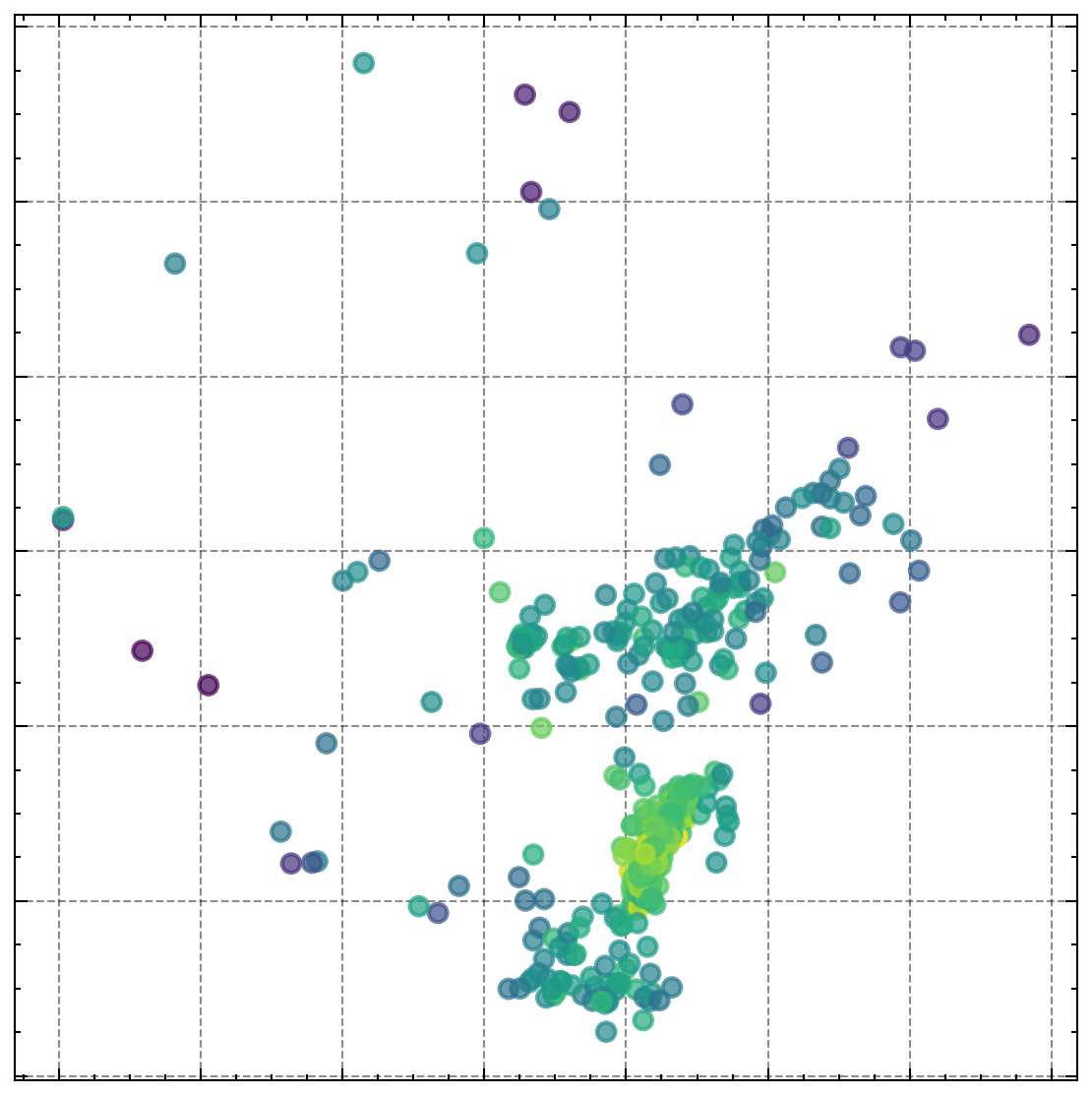}
        \caption{ResNet-101}
    \end{subfigure}
    \hspace{0.01\textwidth}
    \begin{subfigure}[b]{0.235\textwidth}
        \centering
        \includegraphics[width=\textwidth]{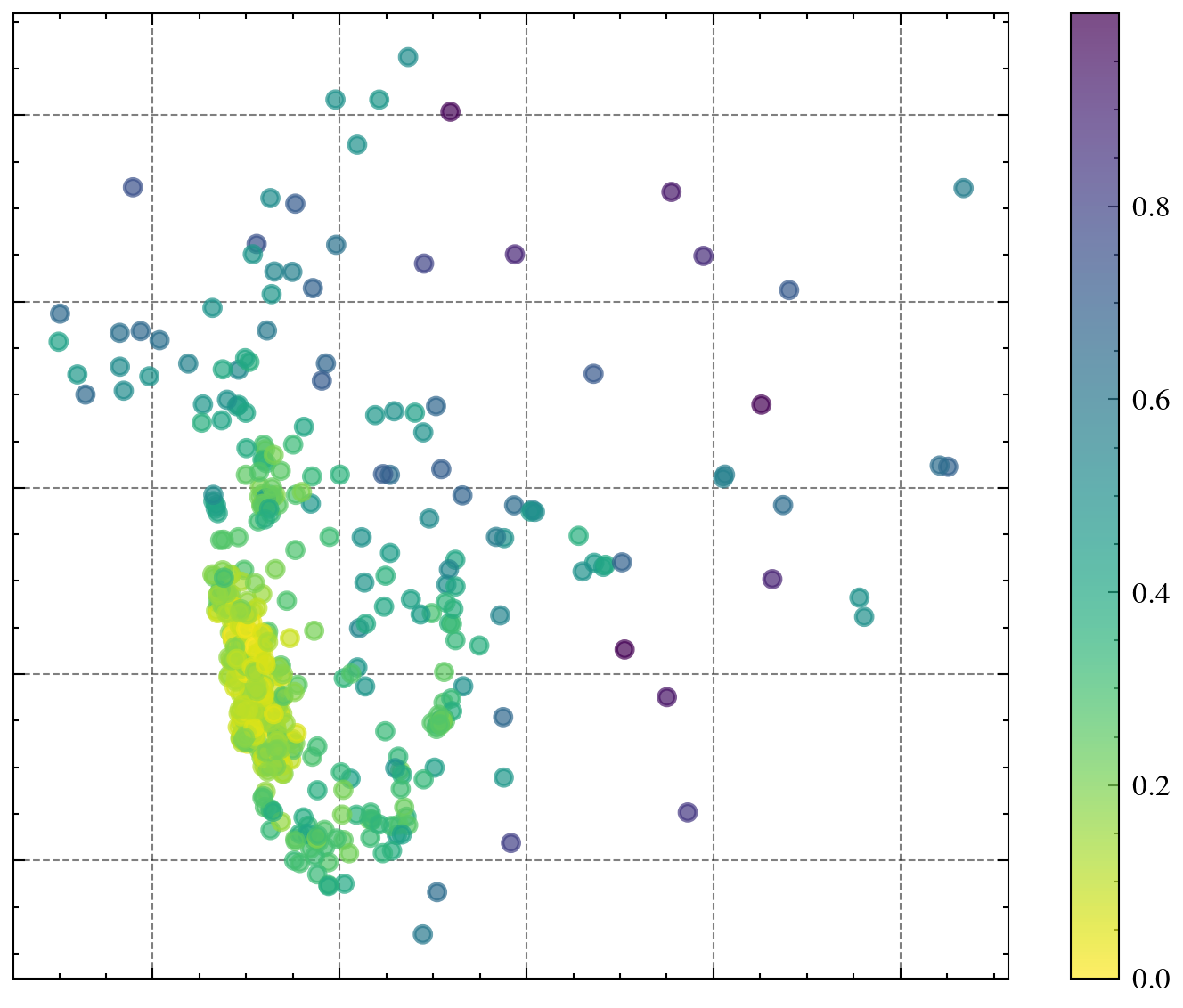}
        \caption{EfficientNet-B0}
    \end{subfigure}

    \vspace{0.01\textwidth}
    \begin{subfigure}[b]{0.2\textwidth}
        \centering
        \includegraphics[width=\textwidth]{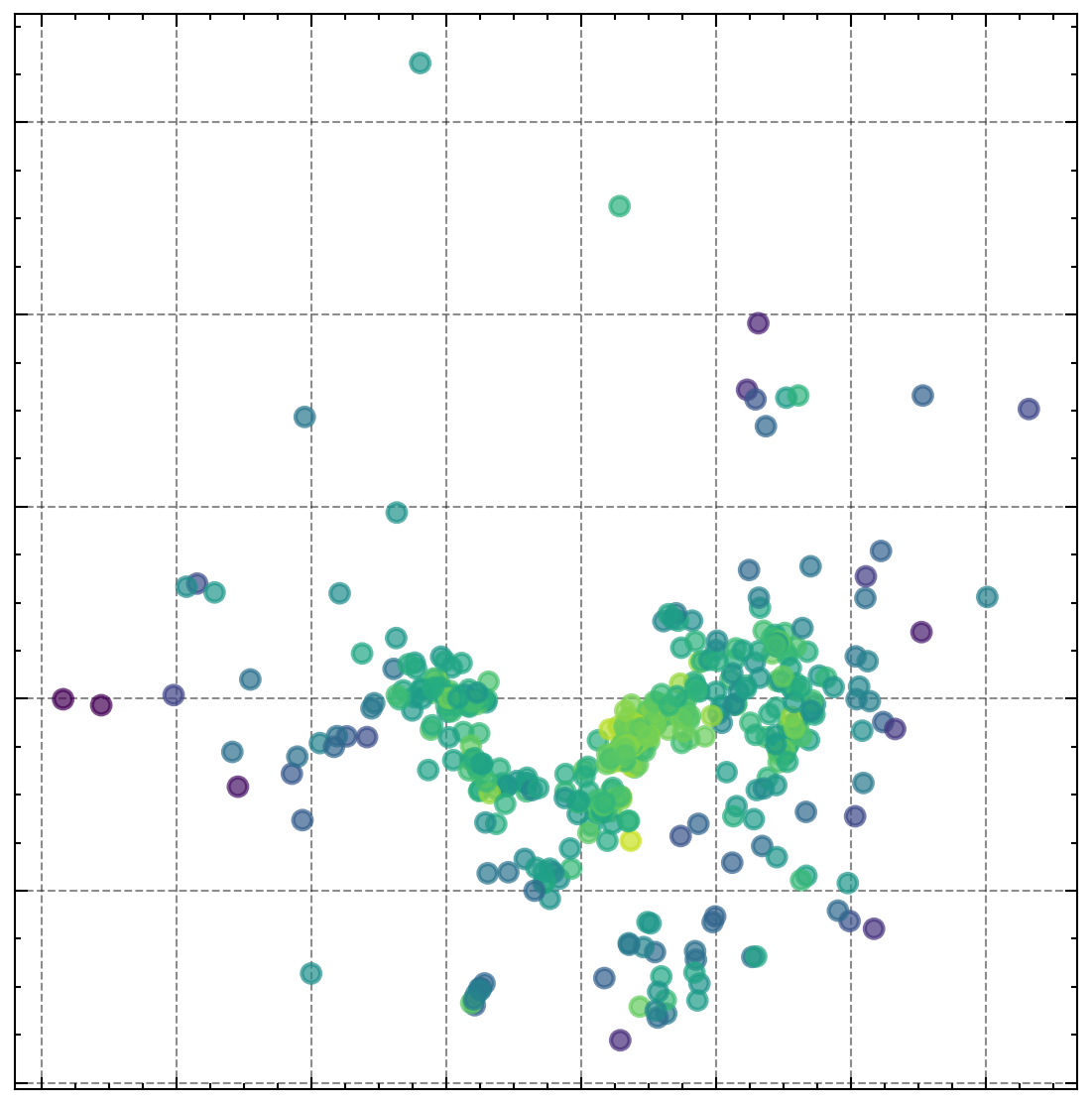}
        \caption{EffNetV2-M}
    \end{subfigure}
    \hspace{0.01\textwidth}
    \begin{subfigure}[b]{0.2\textwidth}
        \centering
        \includegraphics[width=\textwidth]{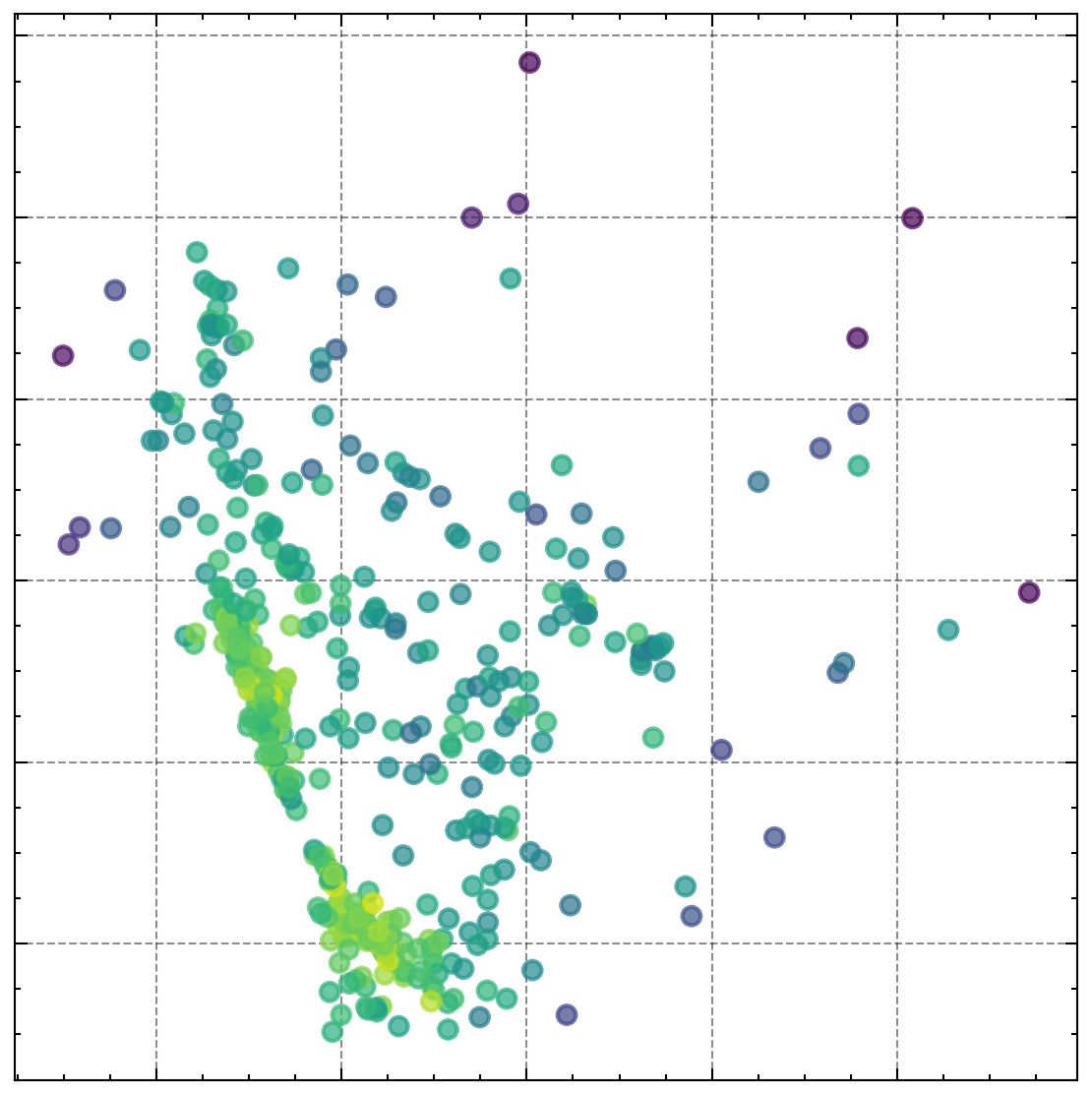}
        \caption{SwinV2-T}
    \end{subfigure}
    \hspace{0.01\textwidth}
    \begin{subfigure}[b]{0.2\textwidth}
        \centering
        \includegraphics[width=\textwidth]{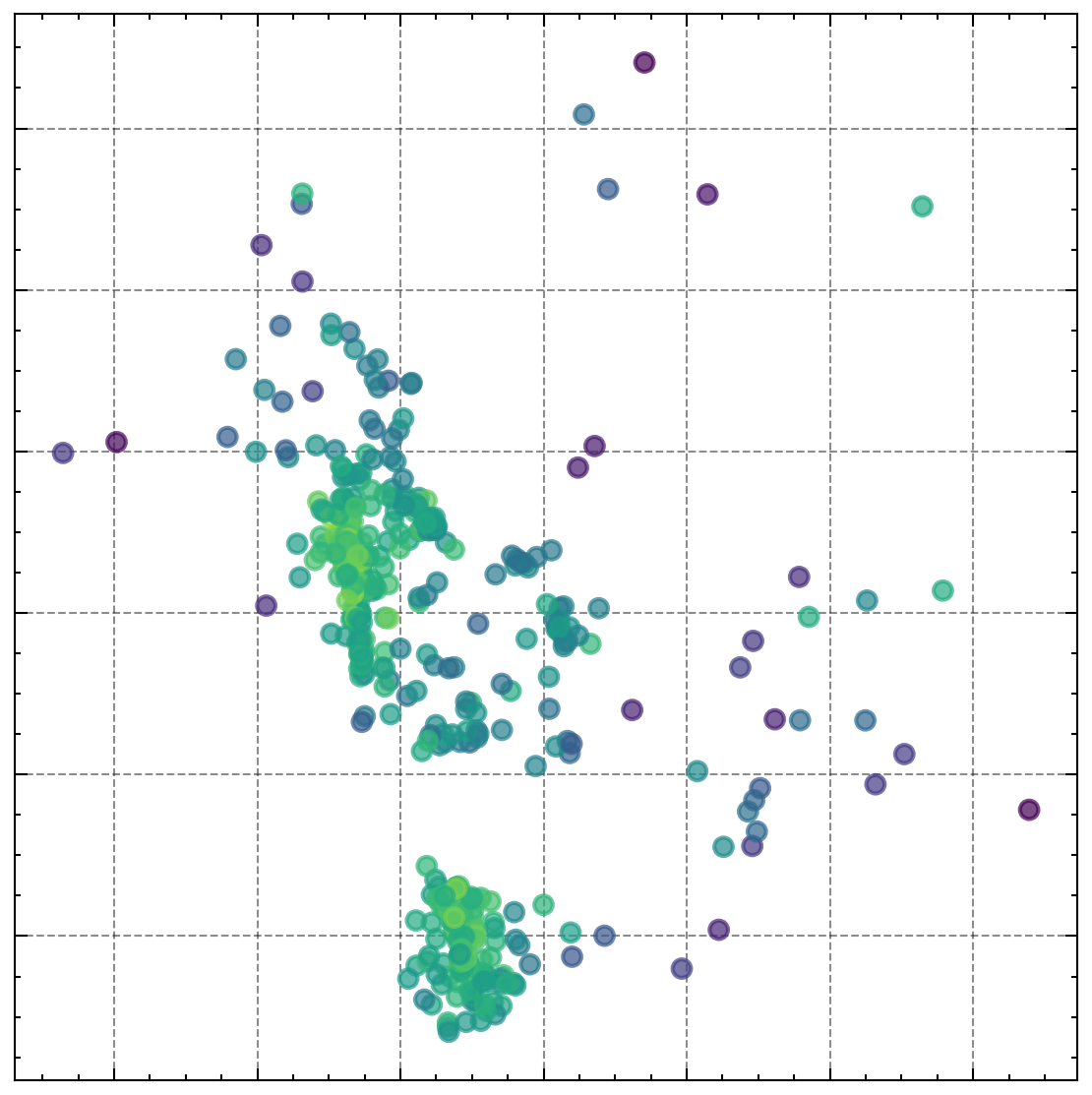}
        \caption{SwinV2-B}
    \end{subfigure}
    \hspace{0.01\textwidth}
    \begin{subfigure}[b]{0.2\textwidth}
        \centering
        \includegraphics[width=\textwidth]{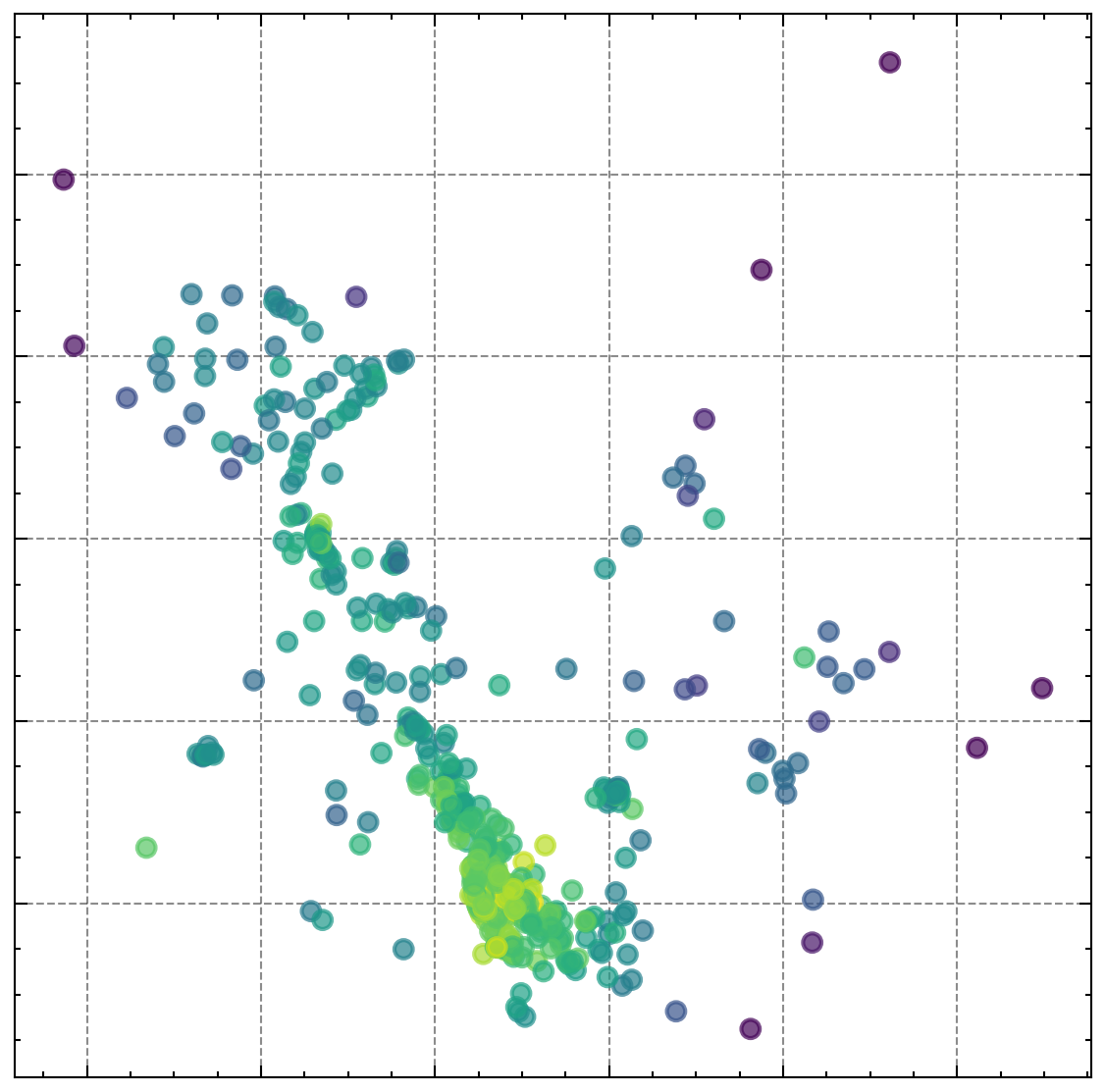}
        \caption{ViT-L-16}
    \end{subfigure}
    
    \caption{$\text{Score}_{\text{unified}}$ on Cifar-100 (butterfly class; label 14). Setting $\alpha=0.5, \beta=0.5$.}
    \label{fig:viz_unified_label14}
\end{figure}

\newpage
\subsubsection{Performance Impact of the Unified Score}
\label{sec:appendix_ablation_weights}
We investigate the performance impact of the hyperparameters $\alpha$ and $\beta$ from \cref{eq:toposcore}, which balance the influence of our global density and local persistence scores (see \cref{fig:ablation_density_persistence}). Our analysis reveals that while the coreset quality is generally stable across a range of $(\alpha, \beta)$ values, a combination of both metrics consistently yields the best performance. Although using either density or persistence alone provides a reasonable baseline, combining them is particularly crucial at high pruning rates (e.g., 90\%), where a balanced score improves accuracy by up to 5.4\% over using either metric in isolation. This demonstrates that both global and local topology are vital for optimal selection and justifies our use of a fixed and balanced configuration set at $(50/50)$ across all experiments, minimizing the need for extensive hyperparameter tuning.

\begin{figure}[H]
    \centering
    \begin{subfigure}[b]{0.44\textwidth}
        \includegraphics[width=\textwidth]{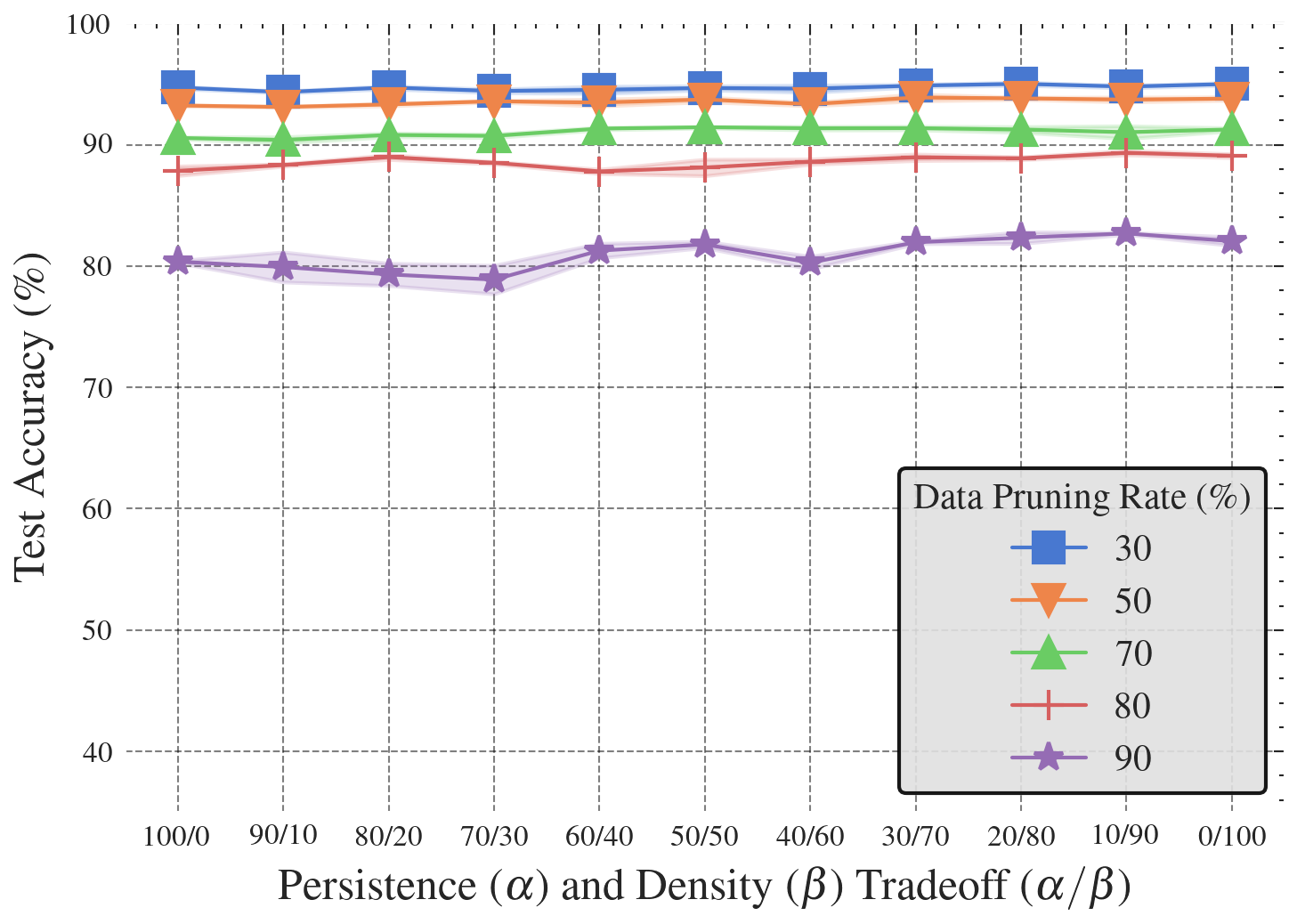}
        \caption{CIFAR-10}
    \end{subfigure}
    \hspace{1em}
    \begin{subfigure}[b]{0.44\textwidth}
        \includegraphics[width=\textwidth]{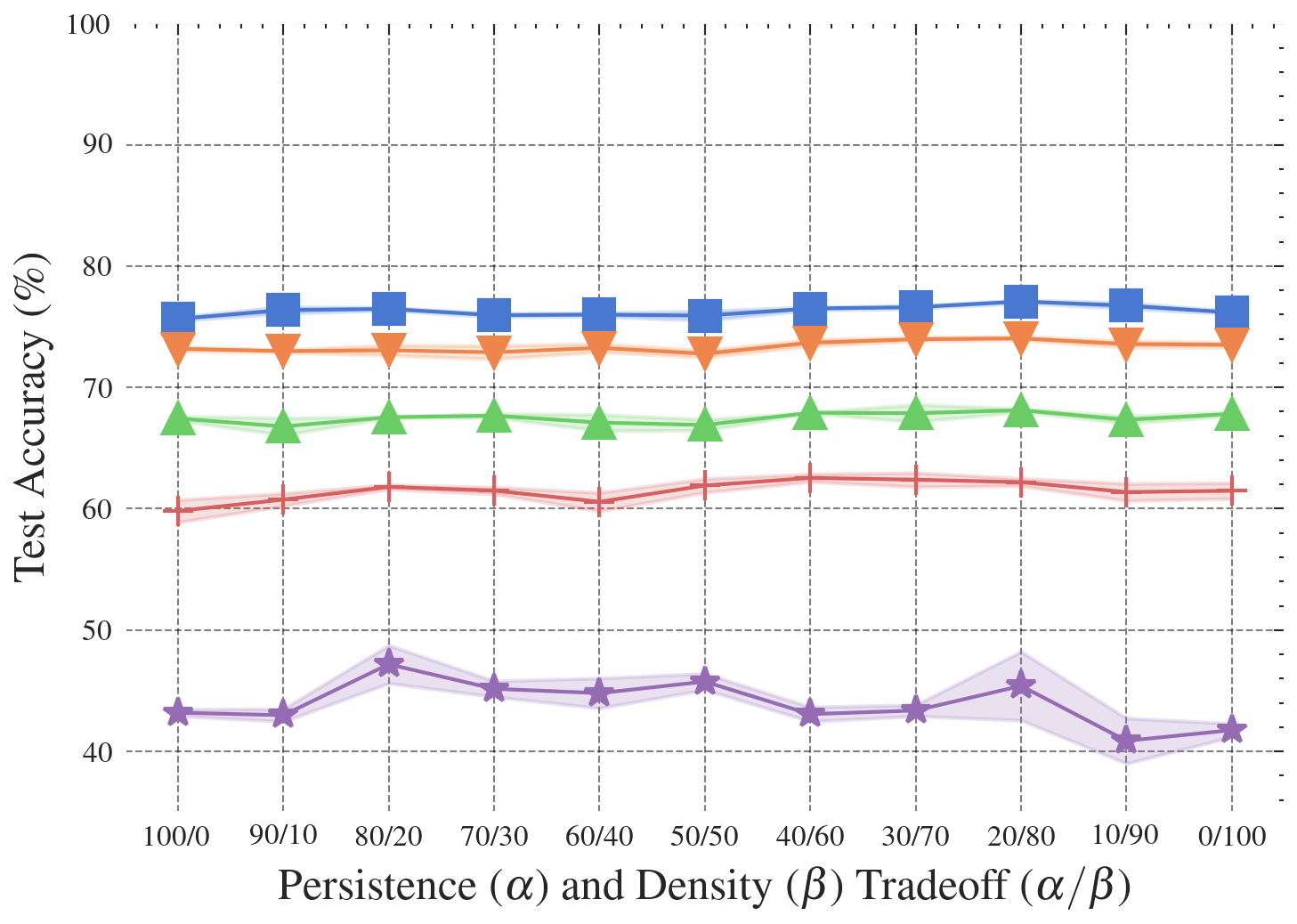}
        \caption{CIFAR-100}
    \end{subfigure}
    
    \caption{$\alpha$ and $\beta$ sweep across data pruning rates for both \textbf{(a)} CIFAR-10 and \textbf{(b)} CIFAR-100.}
    \label{fig:ablation_density_persistence}
\end{figure}

\subsubsection{Illustrative Example of Coreset Construction}
\label{Sec:appendix_illustrative}
To provide a complete picture, we visualize \sys at various rates (70\%, 50\%, 30\%, 20\% and 10\%) for the ``butterfly'' class in CIFAR-100 (see \cref{fig:illustrative}). The visualization reveals a high variance in Persistence Scores within localized regions of the class manifold, demonstrating the method's sensitivity to fine-grained local structures and its ability to distinguish between nearby samples. Despite this focus on local complexity, the final coresets remain density-preserving, with their overall distribution closely matching that of the full dataset. This illustrates how \sys successfully balances the selection of topologically critical local samples with the preservation of the global data structure.

\begin{figure}[!tp]
    \centering
    \includegraphics[width=\textwidth]{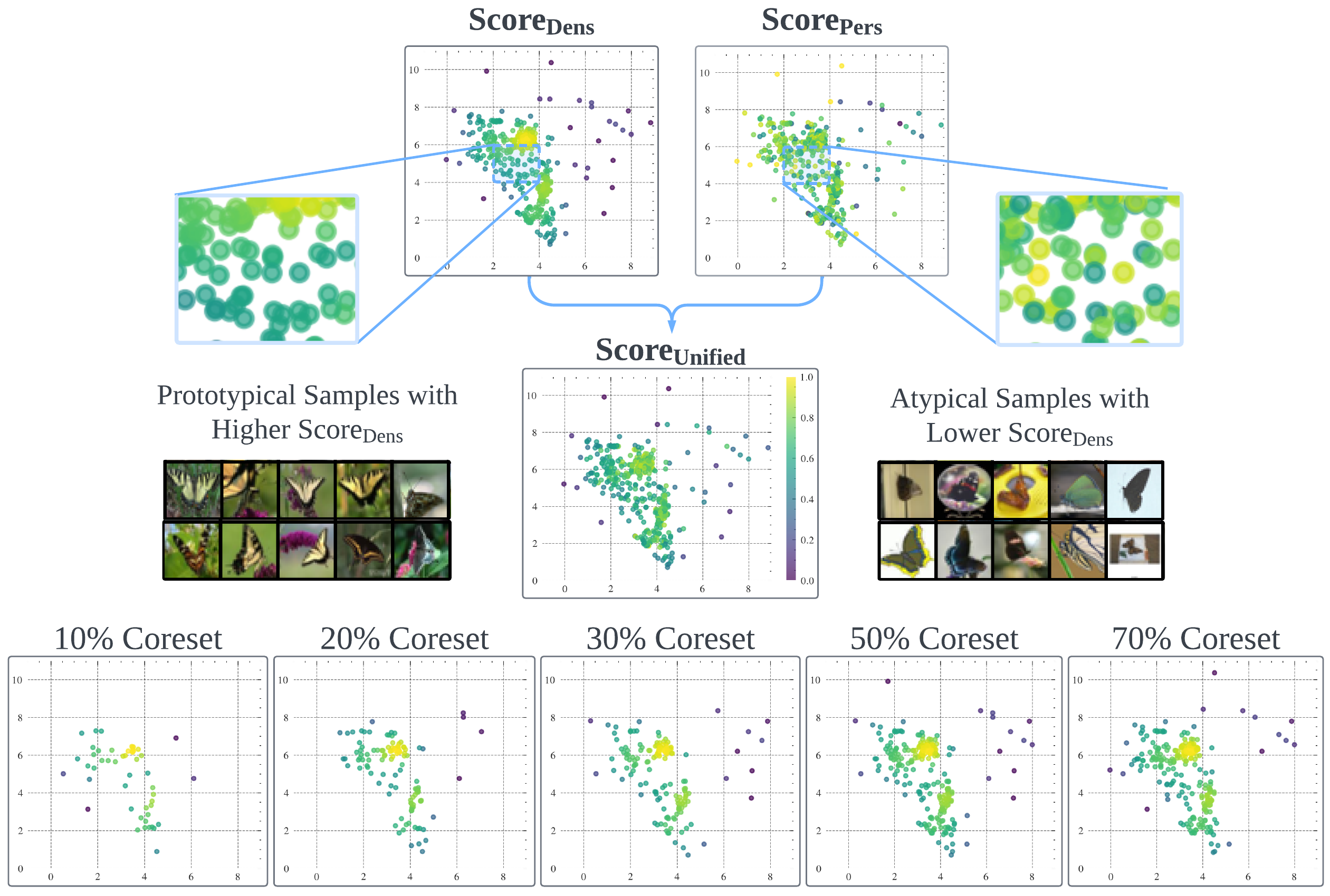}
    \caption{Qualitative analysis of \sys on the CIFAR-100 ``butterfly'' class. \textbf{(Top)} A comparison of scoring components. In highly populated regions (zoomed in view), Kernel Density Estimation saturates uniformly; however, local Persistence effectively resolves the fine-grained topological structure within the local neighborhood to isolate structural anchors. \textbf{(Middle)} The unified \sys score successfully separates prototypical samples (high score) from atypical variants (low score). \textbf{(Bottom)} The spatial distribution of selected samples across varying coreset sizes (10\% to 70\%) demonstrates how this dual-scale topological approach ensures comprehensive coverage of the underlying data manifold at extreme pruning rates.}
    \label{fig:illustrative}
\end{figure}

\subsection{Differentiable Persistence Optimization Steps}
\label{Sec:appendix_tda_steps_ablation}
We investigate the impact of the number of optimization steps for multi-parameter persistent homology (see \cref{fig:ablation_tda_steps}). The number of required persistence optimization steps is inversely correlated with the final coreset size. When selecting a large coreset (e.g., at a 30\% pruning rate), the selection process is robust, and even a few optimization steps (1-2) suffice to identify a high-quality subset. However, at high pruning rates (e.g., 90\%), the task of distinguishing the most crucial samples becomes more sensitive, necessitating a greater number of optimization steps ($\geq$6) to allow the point positions to converge and accurately reveal the most structurally important examples.

\begin{figure}[!h]
    \centering
    \begin{subfigure}[b]{0.44\textwidth}
        \includegraphics[width=\textwidth]{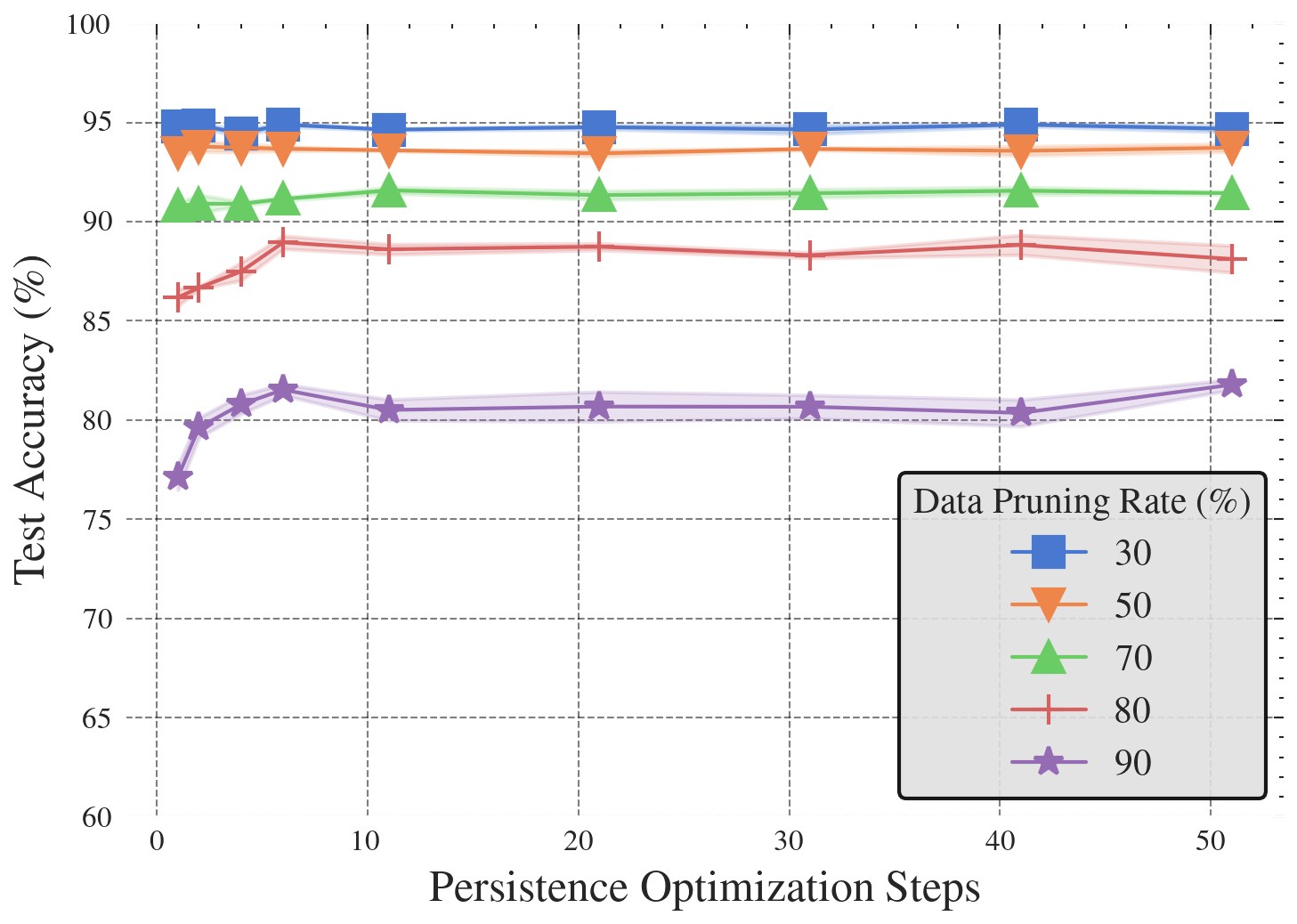}
        \caption{Accuracy on CIFAR-10}
        \label{fig:ablation_tda_steps_acc}
    \end{subfigure}
    \hspace{1em}
    \begin{subfigure}[b]{0.44\textwidth}
        \includegraphics[width=\textwidth]{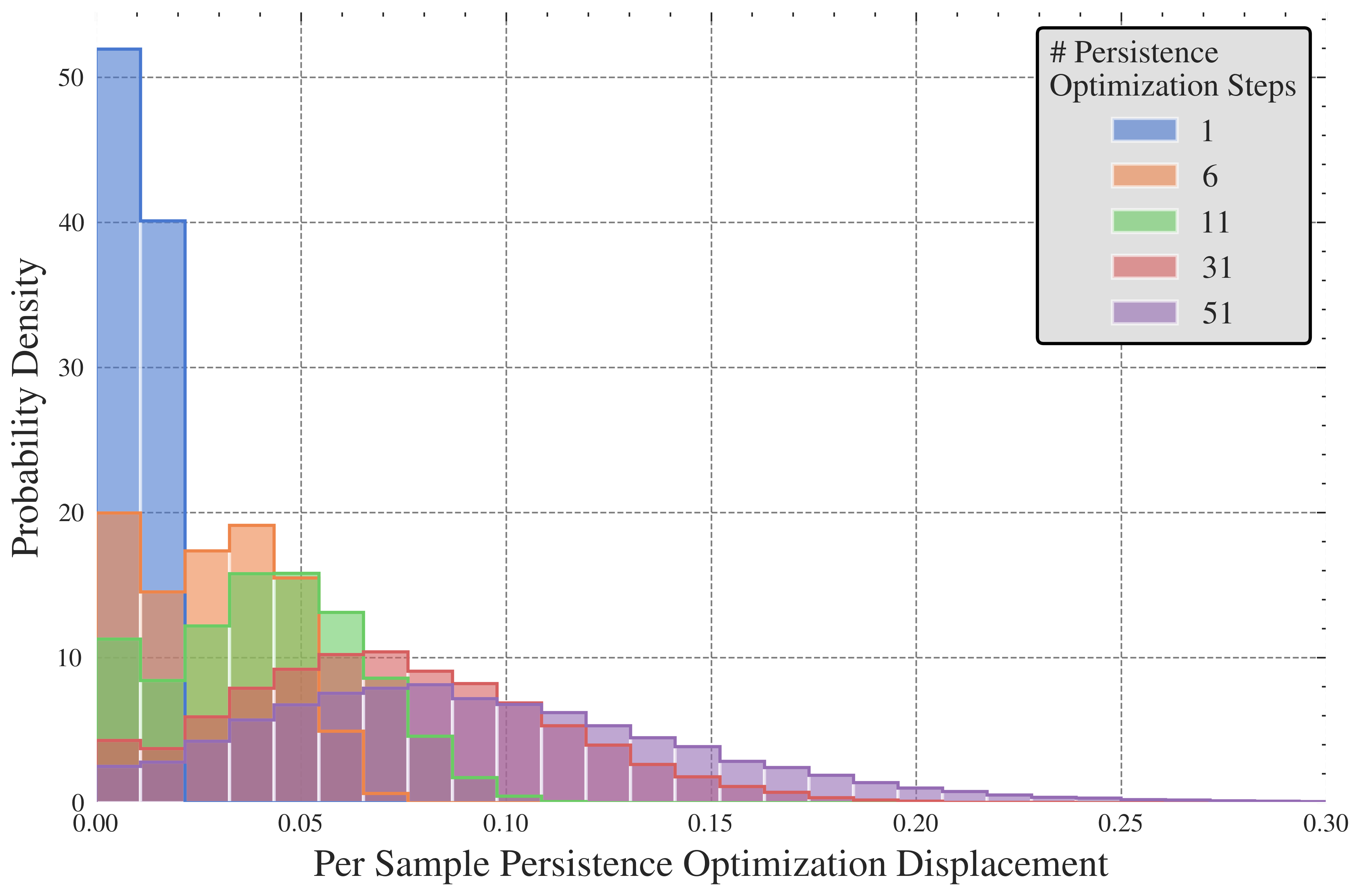}
        \caption{Distribution of Per-Sample Displacement}
        \label{fig:ablation_tda_steps_distribution}
    \end{subfigure}
    
    \caption{Higher compression rates magnify sample importance, necessitating more optimization steps for precise selection. Conversely, larger coresets are more forgiving and require fewer steps.}
    \vspace{-\baselineskip}
    \vspace{-\baselineskip}
    \label{fig:ablation_tda_steps}
\end{figure}

\subsection{Choice of Filtration and Homology Degree}
\label{Sec:appendix_filtration_homology}
We ablate the density parameter of the bifiltration and the homology degree of our persistence objective (\cref{eq:lpers}) on CIFAR-100, holding every other component fixed (\cref{tab:filtration_homology}). We fix the global embedding to $d=2$ (Appendix \cref{table:hyperparameters}) because multi-parameter persistent homology scales poorly with dimension. The Delaunay complex is therefore planar, so only $H_0$ and $H_1$ exist.

\paragraph{The density parameter is crucial.} Without it, the one-parameter $H_1$ variant (40.4$\pm$2.4 at 90\% pruning) falls \textit{below} density-only selection (41.8$\pm$0.6, \cref{tab:abl_density_persistence}), so unconstrained persistence maximization is actively harmful. Recomputing $\hat{f}$ at each epoch keeps $Y'_c$ faithful to the original density landscape. Under a distance-only filtration, points drift across density levels, and displacement conflates structural importance with movement through the density landscape.

\paragraph{$H_0$ versus $H_1$.} Degree-0 persistence of a point cloud is equivalent to single-linkage clustering \citep{carlsson2010characterization, edelsbrunner2010computational}, and over a density parameter \citep{ph17} it recovers the cluster tree of the density \citep{chazal2013persistence, chaudhuri2010rates}, which our Density Score already estimates. Accordingly, $H_0$ outperforms $H_1$ only without explicit density (42.9 vs.\ 40.4 at 90\%), and combining both degrees helps slightly (43.2). Once density is supplied, $H_0$ is redundant and trails $H_1$ (44.4 vs.\ 45.7). Variance follows the same ordering ($\pm$2.4, $\pm$1.3, $\pm$0.5), and \sys is the best configuration at every pruning rate.

\definecolor{Gray}{gray}{0.9}
\newcolumntype{g}{>{\columncolor{Gray}}c}
\begin{table}[!h]
\caption{\textbf{Ablation of filtration and homology degree} on CIFAR-100. Rows are ordered by accuracy at 90\% pruning, and \sys values are from \cref{tab:acc_std}. The explicit density parameter ($KDE$) is crucial, and $H_1$ is the best homology degree once density is present, with the tightest variance.}
\centering
\resizebox{\linewidth}{!}
{
\begin{tabular}{lcccccg}
    \toprule
    & Filtration & Homology & 50\% & 70\% & 80\% & 90\% \\
    \midrule
    One-parameter PH & $Del$ & $H_1$ & 72.1$\pm$0.2 & 66.4$\pm$0.7 & 61.1$\pm$1.0 & 40.4$\pm$2.4 \\
    One-parameter PH & $Del$ & $H_0$ & 72.8$\pm$0.2 & 67.3$\pm$0.3 & 61.6$\pm$0.5 & 42.9$\pm$0.9 \\
    $H_0 + H_1$ & $Del_{H_0} + Del_{H_1}$ & $H_0$ \& $H_1$ & 70.5$\pm$0.9 & 67.4$\pm$0.3 & 61.3$\pm$0.2 & 43.2$\pm$0.6 \\
    $H_0$ variant & $Del + KDE$ & $H_0$ & 71.5$\pm$0.5 & 67.2$\pm$0.3 & 61.5$\pm$0.7 & 44.4$\pm$1.3 \\
    \midrule
    \textbf{\sys} & $Del + KDE$ & $H_1$ & \textbf{73.8$\pm$0.3} & \textbf{68.1$\pm$0.3} & \textbf{62.3$\pm$0.4} & \textbf{45.7$\pm$0.5} \\
\bottomrule
\end{tabular}
}
\label{tab:filtration_homology}
\end{table}

\subsection{Training-free Proxies of Area Under Margin (AUM) for Mislabel Detection}
\label{Sec:appendix_mislabel_proxy_ablation}
We evaluate several training-free methods to serve as a proxy for Area Under the Margin (AUM) in \cref{tab:mislabel_ablation}. These proxies identify potentially noisy samples using different geometric criteria, ranging from culling samples based on their \textit{Distance} to the class prototype, to using an \textit{Adjacent Distance} ratio to remove points closer to another class's prototype. Other heuristics include culling samples with the lowest \textit{Density} score, or using our proposed \textit{Neighborhood Label Purity Score (NLPS)}, which identifies points in mixed-label regions by calculating the fraction of same-label nearest neighbors (from 20 nearest neighbors). Our results show that NLPS provides the highest coreset accuracy among all training-free proxies, with its advantage being most pronounced at high data pruning rates. While it does not fully match the performance of using the true AUM, NLPS serves as a simple and effective training-free proxy.

\begin{table}[!h]
\caption{\textbf{Training-free proxies for Area Under Margin (AUM)} \citep{aum} on CIFAR-100. We see that Neighborhood Label Purity Score (NLPS) performs closest to AUM.}
\centering 
{
\begin{tabular}{lccccc}
    \toprule
    Pruning Rate ($\rightarrow$) & 30\% & 50\% & 70\% & 80\% & 90\% \\
    \midrule
    Distance & 75.4\footnotesize $\pm$0.4 & 71.3\footnotesize $\pm$0.3 & 63.2\footnotesize $\pm$0.1 & 56.1\footnotesize $\pm$0.6 & 37.8\footnotesize $\pm$0.3 \\
    Adjacent Distance & 75.5\footnotesize $\pm$0.3 & 71.8\footnotesize $\pm$0.2 & 62.8\footnotesize $\pm$1.0 & 57.1\footnotesize $\pm$0.3 & 38.1\footnotesize $\pm$1.5 \\
    Density & \underline{75.6\footnotesize $\pm$0.1} & 71.4\footnotesize $\pm$0.2 & 64.3\footnotesize $\pm$0.3 & 52.4\footnotesize $\pm$0.3 & 38.0\footnotesize $\pm$0.7 \\
    \textbf{NLPS} & \textbf{75.6\footnotesize $\pm$0.2} & \textbf{71.9\footnotesize $\pm$0.2} & \textbf{65.3\footnotesize $\pm$0.4} & \textbf{56.7\footnotesize $\pm$0.4} & \textbf{41.6\footnotesize $\pm$1.0} \\
    \midrule
    \textbf{AUM} & 75.9\footnotesize $\pm$0.4 & 72.8\footnotesize $\pm$0.3 & 66.9\footnotesize $\pm$0.5 & 61.9\footnotesize $\pm$0.6 & 45.7\footnotesize $\pm$0.7 \\
\bottomrule
\end{tabular}
}
\label{tab:mislabel_ablation}
\end{table}

\newpage
\subsection{Component Isolation: The Impact of Mislabel Filtering}
\label{Sec:Appendix_filtering_ablation}

To confirm our performance gains are driven by topological scoring rather than initial data sanitization, we isolate the effect of filtering on TopoPrune and geometric baseline Moderate \citep{moderateds}. We evaluate both on CIFAR-100 (our noisiest benchmark) across three regimes: no filtering, training-free (NLPS), and training-dynamic (AUM). TopoPrune outperforms Moderate at every filtering stage (\cref{tab:misl_filt_rebuttal}). This confirms the topological construction provides a fundamental improvement independent of mislabel filtering, remaining superior even without filtering at moderate pruning rates ($50\%$--$80\%$).

Notably, the performance gap widens under explicit filtering, particularly at extreme ($90\%$) pruning. This exposes a structural limitation of Moderate: it restricts sampling to the interquartile range of the distance-to-prototype metric. This rigid cutoff acts as an implicit, naive filter that blindly discards outliers. Thus, applying explicit filters (NLPS/AUM) on top of Moderate yields diminishing returns, as it inherently ignores the regions containing most noisy samples. Conversely, TopoPrune scores the entire data manifold without arbitrary cutoffs. This allows it to fully capitalize on the cleaned distribution, driving the expanding performance gap.

\definecolor{Gray}{gray}{0.9}
\newcolumntype{g}{>{\columncolor{Gray}}c}
\begin{table}[!h]
\caption{\textbf{Component isolation of mislabel filtering.} Evaluating TopoPrune and Moderate across un-filtered, NLPS-filtered, and AUM-filtered CIFAR-100 distributions. TopoPrune consistently outperforms Moderate, with the performance gap expanding significantly at extreme compression rates ($90\%$) under explicit filtering.}
\centering
\resizebox{0.75\linewidth}{!}
{
\begin{tabular}{clcccg}
    \toprule
    & Pruning Rate ($\rightarrow$) & 50\% & 70\% & 80\% & 90\% \\
    \midrule
    \multirow{3}{*}{
    \begin{tabular}
        [c]{@{}c@{}}\small \textit{No}\\ \small \textit{Filtering}
    \end{tabular}
    } & Moderate \citep{moderateds} & 70.1$\pm$0.3 & 63.7$\pm$0.2 & 56.1$\pm$0.5 & 34.9$\pm$2.1  \\
    & \textbf{\sys}  & 71.4$\pm$0.2 & 64.6$\pm$0.3 & 56.4$\pm$0.3 & 35.5$\pm$1.0  \\
    & Acc. $\Delta$  & +1.3 & +0.9 & +0.3 & +0.6 \\

    \midrule
    \multirow{3}{*}{
    \begin{tabular}
        [c]{@{}c@{}}\small \textit{with}\\ \small \textit{NLPS}
    \end{tabular}
    } & Moderate \citep{moderateds} & 70.0$\pm$0.4 & 64.0$\pm$0.1 & 56.2$\pm$0.5 & 39.9$\pm$0.9  \\
    & \textbf{\sys}  & 71.9$\pm$0.2 & 65.3$\pm$0.4 & 56.7$\pm$0.4 & 41.6$\pm$1.0  \\
    & Acc. $\Delta$  & +1.9 & +1.3 & +0.5 & +1.7 \\

    \midrule
    \multirow{3}{*}{
    \begin{tabular}
        [c]{@{}c@{}}\small \textit{with}\\ \small \textit{AUM}
    \end{tabular}
    } & Moderate \citep{moderateds} & 72.4$\pm$0.2 & 66.7$\pm$0.3 & 60.2$\pm$0.8 & 40.0$\pm$1.2  \\
    & \textbf{\sys}  & 72.8$\pm$0.3 & 66.9$\pm$0.5 & 61.9$\pm$0.6 & 45.7$\pm$0.7  \\
    & Acc. $\Delta$  & +0.4 & +0.2 & +1.7 & +5.7 \\

\bottomrule
\end{tabular}
}
\label{tab:misl_filt_rebuttal}
\end{table}

\subsection{Sensitivity to UMAP Manifold Projection Hyperparameters}
\label{Sec:appendix_umap_sensitivity}

To establish the robustness of the global manifold embedding stage of \sys, we conducted a comprehensive hyperparameter sweep of UMAP. We specifically investigate two critical parameters: (1) \textit{n\_neighbors}, which controls the balance between local and global geometric preservation, and (2) \textit{min\_dist}, which governs how tightly samples are packed in the low-dimensional manifold. The results in \cref{tab:umap_hyper} indicate that our selection of \textit{n\_neighbors=15} and \textit{min\_dist=0.1} provides a strong balance of high accuracy and low variance.

\begin{table}[!h]
    \caption{\textbf{\sys sensitivity to UMAP hyperparameters.} \textbf{(a)} accuracy on CIFAR-100 (90\% pruning rate) and \textbf{(b)} relative change to default configuration when modulating \textit{n\_neighbors} and \textit{min\_dist} UMAP hyperparameters.}
    \label{tab:umap_hyper}
    
    \centering
    \begin{subtable}{0.63\textwidth}
        \centering
        \resizebox{\textwidth}{!}
        {
        \begin{tabular}{ccccc|c}
        \toprule
            & \multicolumn{5}{c}{\textbf{n\_neighbors}} \\
            \cmidrule(lr){2-6}
            \textbf{min\_dist} & 5 & 15 & 50 & 100 & Row Avg. \\
            \midrule
            0.05 & 43.9\footnotesize $\pm$2.2 & 43.0\footnotesize $\pm$0.9 & 44.4\footnotesize $\pm$1.7 & 44.9\footnotesize $\pm$1.6 & 44.1\footnotesize $\pm$1.6 \\
            0.1 & 43.1\footnotesize $\pm$1.8 & \textbf{45.8\footnotesize $\pm$0.7} & 45.7\footnotesize $\pm$0.8 & 43.6\footnotesize $\pm$0.7 & 44.6\footnotesize $\pm$0.9 \\
            0.5 & 46.2\footnotesize $\pm$1.0 & 43.8\footnotesize $\pm$0.7 & 44.3\footnotesize $\pm$2.2 & 43.9\footnotesize $\pm$0.4 & 44.6\footnotesize $\pm$1.0 \\
            1.0 & 44.3\footnotesize $\pm$0.9 & 45.4\footnotesize $\pm$0.6 & 45.0\footnotesize $\pm$1.1 & 43.1\footnotesize $\pm$1.3 & 44.4\footnotesize $\pm$1.0 \\
            \midrule
            Col Avg. & 44.4\footnotesize $\pm$1.5 & 44.5\footnotesize $\pm$0.7 &  44.8\footnotesize $\pm$1.4 & 43.9\footnotesize $\pm$1.0 & - \\
        \bottomrule
        \end{tabular}
        }
        
        \caption{Accuracy (over 5 runs)}
        \label{tab:umap_hyper_acc}
    \end{subtable}
    \hspace{0.01\textwidth}    
    \begin{subtable}{0.34\textwidth}
        \centering
        \resizebox{\textwidth}{!}
        {
        \begin{tabular}{ccccc}
        \toprule
            & \multicolumn{4}{c}{\textbf{n\_neighbors}} \\
            \cmidrule(lr){2-5}
            \textbf{min\_dist} & 5 & 15 & 50 & 100 \\
            \midrule
            0.05 & -1.9  & -2.8 & -1.4 & -0.9 \\
            0.1 & -2.7 & $\bigstar$ & -0.1 & -2.2 \\
            0.5 & +0.4 & -2.0 & -1.5 & -1.9 \\
            1.0 & -1.5 & -0.4 & -0.8 & -2.7 \\
        \bottomrule
        \end{tabular}
        }
        
        \caption{Accuracy $\Delta$}
        \label{tab:umap_hyper_acc_delta}
    \end{subtable}
    \vspace{-\baselineskip}
\end{table}

\section{Extended Empirical Results}
\subsection{Extended Text Dataset Pruning Results}
\label{sec:appendix_nlp_full}
We provide the per-benchmark breakdown for the instruction tuning setting of \cref{sec:results_performance} in \cref{tab:nlp_full}. Reasoning benchmark averages are close across all methods, as most of the capability is retained from pretraining when LoRA-tuning on 52K instructions, so data selection shifts instruction-following more than knowledge and these benchmarks largely saturate.

\definecolor{Gray}{gray}{0.9}
\newcolumntype{g}{>{\columncolor{Gray}}c}
\begin{table}[!h]
\caption{\textbf{Per-benchmark results for instruction tuning on Alpaca} (mean $\pm$ std over 3 seeds). Held-out perplexity (PPL), 5-shot MMLU and commonsense reasoning (ArcE/ArcC: ARC-Easy/Challenge, BQ: OpenBookQA, WG: WinoGrande). We use 25-shot ArcC, 5-shot WG and 0-shot for all others. Avg. is the mean over all accuracy columns, with its std computed over the per-seed averages.}
\label{tab:nlp_full}
\setlength{\tabcolsep}{3.1pt}
\centering
\begin{subtable}{\linewidth}
    \centering
    \begin{adjustbox}{max width=\linewidth}
    \begin{tabular}{llcccgccccc}
        \toprule
        Pruning Rate & Method & PPL $\downarrow$ & MMLU & ArcE & ArcC & PIQA & SciQ & BQ & WG & Avg. \\
        \midrule
        \multirow{4}{*}{70\%} & Random & \underline{3.029\footnotesize $\pm$0.001} & \textbf{55.1\footnotesize $\pm$0.2} & 72.3\footnotesize $\pm$1.1 & 50.7\footnotesize $\pm$0.3 & 77.3\footnotesize $\pm$0.9 & 91.1\footnotesize $\pm$0.3 & 42.9\footnotesize $\pm$2.0 & \textbf{72.3\footnotesize $\pm$0.2} & 66.0\footnotesize $\pm$0.4 \\
        & Moderate & 3.041\footnotesize $\pm$0.004 & 54.7\footnotesize $\pm$0.5 & 72.6\footnotesize $\pm$1.0 & 51.1\footnotesize $\pm$0.7 & 77.6\footnotesize $\pm$0.1 & 91.4\footnotesize $\pm$0.4 & 42.6\footnotesize $\pm$0.2 & 70.6\footnotesize $\pm$1.1 & 65.8\footnotesize $\pm$0.3 \\
        & D2 & 3.079\footnotesize $\pm$0.001 & 53.2\footnotesize $\pm$0.8 & \textbf{75.5\footnotesize $\pm$0.6} & 51.2\footnotesize $\pm$0.8 & \textbf{77.9\footnotesize $\pm$0.2} & \textbf{91.7\footnotesize $\pm$0.3} & \textbf{43.8\footnotesize $\pm$0.5} & 70.1\footnotesize $\pm$0.7 & \underline{66.2\footnotesize $\pm$0.2} \\
        & \textbf{\sys} & \textbf{3.029\footnotesize $\pm$0.001} & 53.4\footnotesize $\pm$0.5 & 73.8\footnotesize $\pm$0.3 & \textbf{53.8\footnotesize $\pm$0.5} & 77.7\footnotesize $\pm$0.1 & 91.2\footnotesize $\pm$0.6 & 42.1\footnotesize $\pm$1.0 & 71.7\footnotesize $\pm$0.3 & \textbf{66.2\footnotesize $\pm$0.1} \\
        \midrule
        \multirow{4}{*}{80\%} & Random & 3.039\footnotesize $\pm$0.002 & 54.9\footnotesize $\pm$0.6 & 74.0\footnotesize $\pm$0.6 & 51.7\footnotesize $\pm$0.1 & 77.5\footnotesize $\pm$0.3 & 91.9\footnotesize $\pm$0.6 & 42.3\footnotesize $\pm$0.1 & 71.2\footnotesize $\pm$1.0 & 66.2\footnotesize $\pm$0.2 \\
        & Moderate & 3.048\footnotesize $\pm$0.001 & \textbf{55.8\footnotesize $\pm$0.3} & 71.6\footnotesize $\pm$0.6 & 51.8\footnotesize $\pm$0.6 & \underline{78.2\footnotesize $\pm$0.3} & 90.9\footnotesize $\pm$0.1 & 41.9\footnotesize $\pm$0.4 & 70.9\footnotesize $\pm$0.5 & 65.9\footnotesize $\pm$0.1 \\
        & D2 & 3.048\footnotesize $\pm$0.006 & 54.4\footnotesize $\pm$0.6 & \textbf{75.7\footnotesize $\pm$0.4} & 51.3\footnotesize $\pm$1.0 & 77.9\footnotesize $\pm$0.2 & \textbf{92.3\footnotesize $\pm$0.3} & \underline{42.7\footnotesize $\pm$0.9} & \textbf{71.6\footnotesize $\pm$0.5} & \textbf{66.6\footnotesize $\pm$0.2} \\
        & \textbf{\sys} & \textbf{3.036\footnotesize $\pm$0.001} & 54.7\footnotesize $\pm$0.8 & 73.9\footnotesize $\pm$0.8 & \textbf{52.4\footnotesize $\pm$0.4} & \textbf{78.2\footnotesize $\pm$0.2} & 91.5\footnotesize $\pm$0.3 & \textbf{42.7\footnotesize $\pm$0.2} & 71.5\footnotesize $\pm$0.7 & 66.4\footnotesize $\pm$0.2 \\
        \midrule
        \multirow{4}{*}{90\%} & Random & 3.055\footnotesize $\pm$0.001 & 55.4\footnotesize $\pm$0.2 & 73.9\footnotesize $\pm$0.7 & 50.6\footnotesize $\pm$0.9 & 78.3\footnotesize $\pm$0.4 & 92.3\footnotesize $\pm$0.5 & 43.1\footnotesize $\pm$1.0 & 71.8\footnotesize $\pm$0.8 & 66.5\footnotesize $\pm$0.3 \\
        & Moderate & 3.060\footnotesize $\pm$0.002 & \textbf{56.0\footnotesize $\pm$0.2} & 73.7\footnotesize $\pm$1.1 & 52.1\footnotesize $\pm$1.2 & 78.0\footnotesize $\pm$0.3 & \textbf{92.6\footnotesize $\pm$0.3} & 42.2\footnotesize $\pm$1.5 & 71.9\footnotesize $\pm$0.5 & \underline{66.6\footnotesize $\pm$0.5} \\
        & D2 & \textbf{3.051\footnotesize $\pm$0.003} & 55.4\footnotesize $\pm$0.3 & \textbf{74.0\footnotesize $\pm$0.3} & 51.5\footnotesize $\pm$1.1 & 77.6\footnotesize $\pm$0.1 & 92.1\footnotesize $\pm$0.4 & 42.9\footnotesize $\pm$0.4 & 71.9\footnotesize $\pm$1.5 & 66.5\footnotesize $\pm$0.4 \\
        & \textbf{\sys} & 3.053\footnotesize $\pm$0.006 & 55.7\footnotesize $\pm$0.4 & 72.7\footnotesize $\pm$0.6 & \textbf{52.6\footnotesize $\pm$0.5} & \textbf{78.4\footnotesize $\pm$0.2} & 91.5\footnotesize $\pm$0.3 & \textbf{43.5\footnotesize $\pm$0.8} & \textbf{72.0\footnotesize $\pm$0.4} & \textbf{66.6\footnotesize $\pm$0.2} \\
    \bottomrule
    \end{tabular}
    \end{adjustbox}
    \caption{Llama-3.2-3B}
\end{subtable}

\vspace{4pt}
\begin{subtable}{\linewidth}
    \centering
    \begin{adjustbox}{max width=\linewidth}
    \begin{tabular}{llcccgccccc}
        \toprule
        Pruning Rate & Method & PPL $\downarrow$ & MMLU & ArcE & ArcC & PIQA & SciQ & BQ & WG & Avg. \\
        \midrule
        \multirow{4}{*}{70\%} & Random & \textbf{3.953\footnotesize $\pm$0.015} & 25.9\footnotesize $\pm$0.5 & 54.0\footnotesize $\pm$1.1 & 37.6\footnotesize $\pm$0.5 & \textbf{74.3\footnotesize $\pm$0.5} & 75.0\footnotesize $\pm$0.4 & 37.1\footnotesize $\pm$0.7 & \textbf{60.9\footnotesize $\pm$1.1} & 52.1\footnotesize $\pm$0.2 \\
        & Moderate & 3.960\footnotesize $\pm$0.026 & 26.2\footnotesize $\pm$0.8 & 53.9\footnotesize $\pm$1.1 & 37.7\footnotesize $\pm$0.6 & 74.1\footnotesize $\pm$0.7 & 76.2\footnotesize $\pm$1.1 & \underline{37.3\footnotesize $\pm$0.6} & 60.3\footnotesize $\pm$0.1 & \underline{52.3\footnotesize $\pm$0.5} \\
        & D2 & 3.990\footnotesize $\pm$0.013 & 26.1\footnotesize $\pm$0.5 & \textbf{54.5\footnotesize $\pm$0.8} & 37.3\footnotesize $\pm$1.0 & 74.2\footnotesize $\pm$0.6 & \textbf{76.7\footnotesize $\pm$1.1} & 37.1\footnotesize $\pm$0.1 & 60.2\footnotesize $\pm$0.6 & \underline{52.3\footnotesize $\pm$0.6} \\
        & \textbf{\sys} & 3.972\footnotesize $\pm$0.020 & \textbf{26.5\footnotesize $\pm$0.6} & 54.3\footnotesize $\pm$0.6 & \textbf{37.9\footnotesize $\pm$1.1} & 74.2\footnotesize $\pm$0.3 & 75.2\footnotesize $\pm$2.4 & \textbf{37.3\footnotesize $\pm$0.5} & 60.5\footnotesize $\pm$0.1 & \textbf{52.3\footnotesize $\pm$0.4} \\
        \midrule
        \multirow{4}{*}{80\%} & Random & 4.017\footnotesize $\pm$0.048 & 26.4\footnotesize $\pm$0.7 & 55.0\footnotesize $\pm$0.7 & 37.0\footnotesize $\pm$0.3 & 74.5\footnotesize $\pm$0.4 & 77.9\footnotesize $\pm$1.4 & 37.1\footnotesize $\pm$0.8 & 61.0\footnotesize $\pm$0.7 & 52.7\footnotesize $\pm$0.4 \\
        & Moderate & \textbf{3.977\footnotesize $\pm$0.017} & 26.3\footnotesize $\pm$0.7 & \textbf{56.0\footnotesize $\pm$1.1} & 37.6\footnotesize $\pm$0.9 & 74.6\footnotesize $\pm$0.1 & 78.3\footnotesize $\pm$0.5 & 36.6\footnotesize $\pm$0.7 & 60.9\footnotesize $\pm$0.8 & 52.9\footnotesize $\pm$0.1 \\
        & D2 & 4.018\footnotesize $\pm$0.035 & \textbf{26.5\footnotesize $\pm$1.1} & 55.8\footnotesize $\pm$1.0 & 36.9\footnotesize $\pm$0.4 & \textbf{74.9\footnotesize $\pm$0.4} & \textbf{78.5\footnotesize $\pm$1.3} & \textbf{37.7\footnotesize $\pm$0.6} & \textbf{61.1\footnotesize $\pm$0.4} & \textbf{53.1\footnotesize $\pm$0.1} \\
        & \textbf{\sys} & 3.996\footnotesize $\pm$0.015 & 26.2\footnotesize $\pm$0.5 & 55.9\footnotesize $\pm$1.5 & \textbf{38.1\footnotesize $\pm$0.5} & 74.6\footnotesize $\pm$0.2 & 77.2\footnotesize $\pm$1.0 & 36.8\footnotesize $\pm$1.2 & \underline{61.1\footnotesize $\pm$0.6} & 52.9\footnotesize $\pm$0.3 \\
        \midrule
        \multirow{4}{*}{90\%} & Random & 4.044\footnotesize $\pm$0.015 & 26.4\footnotesize $\pm$0.8 & 55.1\footnotesize $\pm$0.9 & 36.6\footnotesize $\pm$0.4 & 75.0\footnotesize $\pm$0.7 & 78.4\footnotesize $\pm$1.6 & 37.2\footnotesize $\pm$0.8 & \underline{61.2\footnotesize $\pm$1.2} & 52.8\footnotesize $\pm$0.3 \\
        & Moderate & 4.068\footnotesize $\pm$0.024 & 26.3\footnotesize $\pm$0.4 & 55.4\footnotesize $\pm$0.7 & 37.3\footnotesize $\pm$1.2 & 74.9\footnotesize $\pm$0.2 & 78.0\footnotesize $\pm$1.2 & 37.5\footnotesize $\pm$0.3 & 61.0\footnotesize $\pm$1.0 & 52.9\footnotesize $\pm$0.6 \\
        & D2 & 4.079\footnotesize $\pm$0.013 & \textbf{26.9\footnotesize $\pm$0.7} & 55.9\footnotesize $\pm$0.5 & 37.1\footnotesize $\pm$0.8 & 74.9\footnotesize $\pm$0.5 & 78.3\footnotesize $\pm$0.2 & \textbf{37.7\footnotesize $\pm$0.4} & \textbf{61.2\footnotesize $\pm$0.8} & 53.2\footnotesize $\pm$0.2 \\
        & \textbf{\sys} & \textbf{4.038\footnotesize $\pm$0.022} & 26.4\footnotesize $\pm$0.4 & \textbf{56.3\footnotesize $\pm$1.7} & \textbf{37.9\footnotesize $\pm$0.7} & \textbf{75.3\footnotesize $\pm$0.2} & \textbf{78.9\footnotesize $\pm$1.8} & 37.5\footnotesize $\pm$0.9 & 60.9\footnotesize $\pm$0.6 & \textbf{53.3\footnotesize $\pm$0.7} \\
    \bottomrule
    \end{tabular}
    \end{adjustbox}
    \caption{Pythia-2.8B}
\end{subtable}
\end{table}

\subsection{Statistical Significance of Precision Improvements}
\label{sec:appendix_statistical_significance}

\begin{table}[!h]
    \caption{\textbf{Statistical significance of variance} at high pruning rate (90\%) using \textbf{(a, b, c)} F-test shows \sys's lower variance is statistically significant for more difficult datasets. This is further validated by \textbf{(d)} \sys's tighter 95\% confidence intervals of standard deviation.}
    \label{tab:stat_sig}
    
    \centering
    \begin{subtable}{0.3\textwidth}
        \centering
        \resizebox{\textwidth}{!}
        {
        \begin{tabular}{lccc}
        \toprule
            \textbf{p\_value} & CCS & D2 & \textbf{\sys} \\
            \midrule
            CCS & 0.500  & $\sim~$ & $\sim~$ \\
            D2 & 0.128 & 0.500 & $\sim~$ \\
            \textbf{\sys} & 0.422 & 0.094 & 0.500 \\
        \bottomrule
        \end{tabular}
        }
        
        \caption{F-test on CIFAR-10}
    \end{subtable}
    \hspace{0.02\textwidth}    
    \begin{subtable}{0.3\textwidth}
        \centering
        \resizebox{\textwidth}{!}
        {
        \begin{tabular}{lcccc}
        \toprule
            \textbf{p\_value} & CCS & D2 & \textbf{\sys} \\
            \midrule
            CCS & 0.500  & $\sim~$ & $\sim~$ \\
            D2 & 0.437& 0.500 & $\sim~$ \\
            \textbf{\sys} & \textbf{0.015} & \textbf{0.011} & 0.500 \\
        \bottomrule
        \end{tabular}
        }
        
        \caption{F-test on CIFAR-100}
    \end{subtable}
    \hspace{0.02\textwidth}    
    \begin{subtable}{0.3\textwidth}
        \centering
        \resizebox{\linewidth}{!}
        {
        \begin{tabular}{lccc}
        \toprule
            \textbf{p\_value} & CCS & D2 & \textbf{\sys} \\
            \midrule
            CCS & 0.500  & $\sim~$ & $\sim~$ \\
            D2 & 0.059 & 0.500 & $\sim~$ \\
            \textbf{\sys} & \textbf{0.012} & \textbf{0.001} & 0.500 \\
        \bottomrule
        \end{tabular}
        }
        
        \caption{F-test on ImageNet-1K}
    \end{subtable}

    \vfill
    \vspace{1em}
    \begin{subtable}{\textwidth}
        \centering
        \resizebox{0.41\linewidth}{!}
        {
        \begin{tabular}{lccc}
        \toprule
            \textbf{95\% CI} & CIFAR-10 & CIFAR-100 & ImageNet-1K \\
            \midrule
            CCS & [0.05, 0.91] & [0.85, 3.10] & [0.07, 0.75] \\
            D2 & [0.38, 1.77] & [0.26, 2.69] & [0.3, 2.12] \\
            \textbf{\sys} & [0.08, 0.35] & [0.1, 0.87] & [0.04, 0.23] \\
        \bottomrule
        \end{tabular}
        }

        \caption{Bootstrapped 95\% confidence intervals of standard deviation}
    \end{subtable}

\end{table}

To validate the stability benefits of \sys (specifically, reduced variability in final model accuracy across independent runs), we conducted statistical tests focusing on high pruning rates (e.g., 90\%), where coreset selection variance is typically most pronounced. We compare \sys against top-performing baselines, D2 \citep{d2} and CCS \citep{ccs}. Our analysis employed two complementary statistical approaches: (1) a one-tailed F-test to evaluate the hypothesis that \sys exhibits lower variance than baselines ($\sigma^2_{\text{\sys}} < \sigma^2_{\text{baseline}}$), and (2) Bootstrapped 95\% Confidence Intervals (CI) for the standard deviation of final accuracies, computed using 10,000 resamples with replacement.

As detailed in \cref{tab:stat_sig}, on simpler datasets the difference in variance between methods is not statistically significant. However, for more challenging benchmarks, \sys demonstrates statistically significant $p$-values ($p < 0.05$, bolded) for lower variance compared to both D2 and CCS. Also, the bootstrapped confidence intervals for the standard deviation of \sys are strictly lower, and in some cases disjoint, than those of the baselines. These findings confirm that our topological selection mechanism offers superior precision compared to prior geometric methods.

\newpage
\subsection{Roadmap of Detailed Experimental Results}
\label{Sec:appendix_all_acc_std}

\begin{table}[!h]
\centering
\resizebox{\linewidth}{!}{
\label{tab:appendix_roadmap}
\begin{tabular}{l c c}
\toprule
\textbf{Experiment / Ablation Topic} & \textbf{Main Section} & \textbf{Full Data Table} \\
\midrule
Overall Performance Across Baselines and Datasets & \cref{sec:results_performance} & \cref{tab:acc_std} \\
Text Dataset Pruning & \cref{sec:results_performance} & \cref{tab:nlp}, \cref{tab:nlp_full} \\
Transferability of \textit{Diverse Embeddings $\rightarrow$ Fixed Target} & \cref{sec:results_arch_transfer} & \cref{tab:arch_ablation} \\
Transferability of \textit{Fixed Embedding $\rightarrow$ Diverse Targets} & \cref{sec:results_arch_transfer} & \cref{tab:arch_ablation_2} \\
Robustness to Noisy and Corrupted Representations & \cref{sec:results_noisy_features} & \cref{tab:noise_ablation}, \cref{tab:input_noise_ablation}, \cref{tab:corruption_ablation} \\
Evaluation of Deep Topological Autoencoders & \cref{sec:appendix_topo_ae} & \cref{tab:topo_ae} \\
Computational Complexity & \cref{Sec:appendix_complexity_analysis} & \cref{table:complexity}, \cref{table:wall_clock} \\
Ablation: Performance Impact of ($\alpha$ vs. $\beta$) & \cref{sec:appendix_ablation_weights} & \cref{tab:abl_density_persistence} \\
Ablation: Filtration and Homology Degree & \cref{Sec:appendix_filtration_homology} & \cref{tab:filtration_homology} \\
Ablation: Training-Free Proxies for Mislabel Detection &\cref{Sec:appendix_mislabel_proxy_ablation} & \cref{tab:mislabel_ablation} \\
Ablation: Sensitivity to UMAP Hyperparameters & \cref{Sec:appendix_umap_sensitivity} & \cref{tab:umap_hyper} \\
Implementation Hyperparameters & $\thicksim$ & \cref{table:hyperparameters} \\
\bottomrule
\end{tabular}}
\end{table}

\begin{table}[!bp]
\caption{\textbf{Hyperparameters for local persistence ($\alpha$) and global density ($\beta$).} While our fixed 50/50 split provides strong, stable performance, the results indicate that further accuracy gains are possible with task-specific tuning. We observe a trend where the optimal balance increasingly relies on the persistence score ($\alpha$) on more challenging datasets (e.g., CIFAR-100 vs. CIFAR-10) and at higher data pruning rates. Optimal results are in bold, ties are underlined.}
\setlength{\tabcolsep}{3.1pt}
\resizebox{0.49\linewidth}{!}
{  
\begin{tabular}{ccccccc}
    \toprule
    \textbf{CIFAR-10} & \multicolumn{5}{c}{Pruning Ratio (\%)} \\
    \cmidrule(lr){1-1} \cmidrule(lr){2-6}
    $\alpha/\beta$ & 30\% & 50\% & 70\% & 80\% & 90\% \\
    \midrule
    100/0 & 94.7$\pm$0.1 & 93.2$\pm$0.1 & 90.6$\pm$0.1 & 87.9$\pm$0.4 & 80.4$\pm$0.1 \\
    90/10 & 94.3$\pm$0.1 & 93.1$\pm$0.1 & 90.4$\pm$0.3 & 88.3$\pm$0.1 & 79.9$\pm$1.2 \\
    80/20 & 94.7$\pm$0.1 & 93.3$\pm$0.1 & 90.8$\pm$0.2 & 89.0$\pm$0.2 & 79.3$\pm$0.9 \\
    70/30 & 94.4$\pm$0.2 & 93.6$\pm$0.1 & 90.7$\pm$0.2 & 88.5$\pm$0.1 & 78.9$\pm$1.2 \\
    60/40 & 94.5$\pm$0.3 & 93.5$\pm$0.3 & 91.3$\pm$0.1 & 87.8$\pm$0.2 & 81.3$\pm$0.6 \\
    \midrule
    50/50 & 94.7$\pm$0.2 & 93.7$\pm$0.2 & \textbf{91.6$\pm$0.1} & 88.7$\pm$0.7 & 82.1$\pm$0.3 \\
    \midrule
    40/60 & 94.6$\pm$0.3 & 93.3$\pm$0.1 & 91.4$\pm$0.2 & 88.6$\pm$0.2 & 80.3$\pm$0.5 \\
    30/70 & 94.9$\pm$0.1 & \textbf{93.9$\pm$0.3} & 91.4$\pm$0.2 & 89.0$\pm$0.3 & 82.0$\pm$0.1 \\
    20/80 & \textbf{95.0$\pm$0.1} & 93.8$\pm$0.1 & 91.3$\pm$0.2 & 88.9$\pm$0.1 & 82.3$\pm$0.5 \\
    10/90 & 94.8$\pm$0.1 & 93.7$\pm$0.2 & 91.0$\pm$0.5 & \textbf{89.3$\pm$0.2} & \textbf{82.7$\pm$0.1} \\
    0/100 & 94.9$\pm$0.1 & 93.8$\pm$0.1 & 91.3$\pm$0.1 & 89.1$\pm$0.2 & 82.2$\pm$0.4 \\
\bottomrule
\end{tabular}
}
\quad
\resizebox{0.49\linewidth}{!}
{
\begin{tabular}{ccccccc}
    \toprule
    \textbf{CIFAR-100} & \multicolumn{5}{c}{Pruning Ratio (\%)} \\
    \cmidrule(lr){1-1} \cmidrule(lr){2-6}
    $\alpha/\beta$ & 30\% & 50\% & 70\% & 80\% & 90\% \\
    \midrule
    100/0 & 75.7$\pm$0.2 & 73.2$\pm$0.2 & 67.4$\pm$0.2 & 59.8$\pm$0.9 & 43.2$\pm$0.3 \\
    90/10 & 76.4$\pm$0.3 & 73.0$\pm$0.1 & 66.8$\pm$0.7 & 60.8$\pm$0.5 & 43.0$\pm$0.5 \\
    80/20 & 76.5$\pm$0.1 & 73.1$\pm$0.4 & 67.5$\pm$0.1 & 61.8$\pm$0.2 & \textbf{47.2$\pm$1.6} \\
    70/30 & 76.0$\pm$0.2 & 72.9$\pm$0.6 & 67.7$\pm$0.2 & 61.5$\pm$0.3 & 45.2$\pm$0.7 \\
    60/40 & 76.0$\pm$0.2 & 73.3$\pm$0.3 & 67.1$\pm$0.7 & 60.6$\pm$0.8 & 44.8$\pm$1.2 \\
    \midrule
    50/50 & 75.9$\pm$0.4 & 72.8$\pm$0.3 & 66.9$\pm$0.5 & 61.9$\pm$0.6 & 45.8$\pm$0.7 \\
    \midrule
    40/60 & 76.5$\pm$0.1 & 73.7$\pm$0.3 & 67.9$\pm$0.1 & \textbf{62.5$\pm$0.3} & 43.1$\pm$0.6 \\
    30/70 & 76.6$\pm$0.1 & \textbf{74.0$\pm$0.2} & 67.9$\pm$0.7 & 62.4$\pm$0.6 & 43.4$\pm$0.5 \\
    20/80 & \textbf{77.1$\pm$0.1} & \underline{74.0$\pm$0.2} & \textbf{68.1$\pm$0.1} & 62.2$\pm$0.3 & 45.4$\pm$2.8 \\
    10/90 & 76.7$\pm$0.2 & 73.6$\pm$0.3 & 67.3$\pm$0.3 & 61.4$\pm$0.7 & 40.9$\pm$1.9 \\
    0/100 & 76.2$\pm$0.1 & 73.5$\pm$0.2 & 67.8$\pm$0.2 & 61.5$\pm$0.7 & 41.8$\pm$0.6 \\
\bottomrule
\end{tabular}
}

\label{tab:abl_density_persistence}
\end{table}

\definecolor{Gray}{gray}{0.9}
\newcolumntype{g}{>{\columncolor{Gray}}c}
\begin{table}[!tp]
\caption{\textbf{Coreset performance under \textit{latent Gaussian noise} on CIFAR-100.}} 
\setlength{\tabcolsep}{3.5pt}
\centering
\resizebox{0.95\linewidth}{!}
{  
\begin{tabular}{lcccgcccgcccgcccg}
    \toprule
    Noise ($\rightarrow$) & \multicolumn{4}{c}{$\mathbf{\epsilon} \sim \mathcal{N}(0,\,0.25\sigma)$} & \multicolumn{4}{c}{$\mathbf{\epsilon} \sim \mathcal{N}(0,\,\sigma)$} \\
    \cmidrule(lr){2-5} \cmidrule(lr){6-9}
    Pruning Rate ($\rightarrow$) & 50\% & 70\% & 80\% & 90\% & 50\% & 70\% & 80\% & 90\% \\
    \midrule
    Moderate \citep{moderateds}  & 71.1$\pm$0.2 & 63.7$\pm$0.4 & 56.0$\pm$0.6 & 33.2$\pm$0.9   & 70.4$\pm$\underline{0.1} & 62.5$\pm$0.4 & 54.6$\pm$0.6 & 33.9$\pm$0.2  \\
    D2 \citep{d2}  & 72.8$\pm$0.1 & 68.5$\pm$0.7 & 63.0$\pm$0.5 & 44.4$\pm$1.5   & 73.0$\pm$0.1 & 67.7$\pm$1.2 & 62.3$\pm$0.8 & 40.2$\pm$2.0  \\
    \sys  & 73.5$\pm$0.3 & 67.8$\pm$0.1 & 60.5$\pm$0.2 & \textbf{45.4$\pm$0.8}   & 73.3$\pm$0.2 & 68.0$\pm$0.5 & 62.7$\pm$0.3 & \textbf{45.5$\pm$0.6} \\
\bottomrule
\end{tabular}
}
\vspace{3pt}
\resizebox{0.95\linewidth}{!}
{
\begin{tabular}{lcccgcccgcccgcccg}
    \toprule
    Noise ($\rightarrow$) & \multicolumn{4}{c}{$\mathbf{\epsilon} \sim \mathcal{N}(0,\,4\sigma)$} & \multicolumn{4}{c}{$\mathbf{\epsilon} \sim \mathcal{N}(0,\,8\sigma)$} \\
    \cmidrule(lr){2-5} \cmidrule(lr){6-9}
    Pruning Rate ($\rightarrow$) & 50\% & 70\% & 80\% & 90\% & 50\% & 70\% & 80\% & 90\% \\
    \midrule
    Moderate \citep{moderateds} & 71.0$\pm$0.4 & 62.9$\pm$0.3 & 52.3$\pm$0.6 & 32.0$\pm$1.2 &  70.9$\pm$0.2 & 63.4$\pm$0.2 & 51.6$\pm$2.2 & 32.1$\pm$1.4  \\
    D2 \citep{d2} & 72.3$\pm$\underline{0.1} & 67.8$\pm$1.1 & 62.1$\pm$1.0 & 40.5$\pm$1.7 &  71.6$\pm$0.5 & \underline{67.2}$\pm$0.8 & 60.3$\pm$2.4 & 39.8$\pm$3.2  \\
    \sys & 73.4$\pm$0.1 & 67.4$\pm$0.2 & 61.7$\pm$0.3 & \textbf{46.1$\pm$0.7} &  73.4$\pm$0.2 & 67.2$\pm$0.1 & 61.8$\pm$0.1 & \textbf{43.9$\pm$0.4}  \\
\bottomrule
\end{tabular}
}
\vspace{-10pt}
\label{tab:noise_ablation}
\end{table}
\definecolor{Gray}{gray}{0.9}
\newcolumntype{g}{>{\columncolor{Gray}}c}
\begin{table}[!tp]
\caption{\textbf{Coreset performance under \textit{image-level Gaussian noise} on CIFAR-100.} 
}
\setlength{\tabcolsep}{3.5pt}
\centering
\resizebox{0.95\linewidth}{!}
{  
\begin{tabular}{lcccgcccgcccgcccg}
    \toprule
    Noise ($\rightarrow$) & \multicolumn{4}{c}{$\mathbf{\epsilon} \sim \mathcal{N}(0,\,0.25\sigma)$} & \multicolumn{4}{c}{$\mathbf{\epsilon} \sim \mathcal{N}(0,\,\sigma)$} \\
    \cmidrule(lr){2-5} \cmidrule(lr){6-9}
    Pruning Rate ($\rightarrow$) & 50\% & 70\% & 80\% & 90\% & 50\% & 70\% & 80\% & 90\% \\
    \midrule
    Moderate \citep{moderateds}  & 70.4$\pm$0.3 & 62.8$\pm$0.2 & 53.3$\pm$0.4 & 35.0$\pm$0.3            & 70.8$\pm$0.2 & 62.8$\pm$0.3 & 53.1$\pm$0.3 & 34.0$\pm$0.6  \\
    D2 \citep{d2}                & 72.7$\pm$0.5 & 68.9$\pm$0.2 & 62.0$\pm$0.8 & 40.4$\pm$1.5            & 71.6$\pm$0.2 & 67.2$\pm$0.5 & 62.7$\pm$1.1 & 43.4$\pm$0.9  \\
    \textbf{\sys}                & 73.0$\pm$0.3 & 67.8$\pm$0.2 & 61.1$\pm$0.4 & \textbf{40.9$\pm$0.2}   & 72.7$\pm$0.2 & 67.0$\pm$0.4 & 62.3$\pm$0.1 & \textbf{43.5$\pm$0.3} \\
\bottomrule
\end{tabular}
}
\vspace{3pt}
\resizebox{0.95\linewidth}{!}
{
\begin{tabular}{lcccgcccgcccgcccg}
    \toprule
    Noise ($\rightarrow$) & \multicolumn{4}{c}{$\mathbf{\epsilon} \sim \mathcal{N}(0,\,4\sigma)$} & \multicolumn{4}{c}{$\mathbf{\epsilon} \sim \mathcal{N}(0,\,8\sigma)$} \\
    \cmidrule(lr){2-5} \cmidrule(lr){6-9}
    Pruning Rate ($\rightarrow$) & 50\% & 70\% & 80\% & 90\% & 50\% & 70\% & 80\% & 90\% \\
    \midrule
    Moderate \citep{moderateds}  & 71.5$\pm$0.3 & 63.9$\pm$0.5 & 57.5$\pm$0.4 & 35.7$\pm$1.7            & 71.6$\pm$0.2 & 64.6$\pm$0.6 & 54.7$\pm$1.0 & 34.4$\pm$1.5  \\
    D2 \citep{d2}                & 72.5$\pm$0.2 & 67.2$\pm$0.9 & 60.3$\pm$1.2 & 41.3$\pm$2.2            & 72.5$\pm$0.2 & 67.7$\pm$1.0 & 59.9$\pm$1.0 & 42.6$\pm$2.5  \\
    \textbf{\sys}                & 73.2$\pm$0.3 & 68.3$\pm$0.2 & 61.1$\pm$0.5 & \textbf{41.9$\pm$0.4}   & 73.2$\pm$0.1 & 67.1$\pm$0.3 & 60.8$\pm$0.5 & \textbf{43.3$\pm$0.7} \\
\bottomrule
\end{tabular}
}
\vspace{-10pt}
\label{tab:input_noise_ablation}
\end{table}

\definecolor{Gray}{gray}{0.9}
\newcolumntype{g}{>{\columncolor{Gray}}c}
\begin{table}[!tp]
\caption{\textbf{Coreset performance under \textit{image-level corruptions} on CIFAR-100.} We apply four representative CIFAR-C corruptions \citep{hendrycksbenchmarking} (contrast, motion blur, frost, and JPEG compression) at severity level 3.}
\setlength{\tabcolsep}{3.5pt}
\centering
\resizebox{0.95\linewidth}{!}
{  
\begin{tabular}{lcccgcccgcccgcccg}
    \toprule
    Corruption ($\rightarrow$) & \multicolumn{4}{c}{Contrast} & \multicolumn{4}{c}{Motion Blur} \\
    \cmidrule(lr){2-5} \cmidrule(lr){6-9}
    Pruning Rate ($\rightarrow$) & 50\% & 70\% & 80\% & 90\% & 50\% & 70\% & 80\% & 90\% \\
    \midrule
    Moderate \citep{moderateds}  & 70.9$\pm$0.3 & 63.4$\pm$0.4 & 53.7$\pm$0.3 & 30.9$\pm$0.7   & 71.2$\pm$0.1 & 64.1$\pm$0.2 & 56.4$\pm$0.5 & 34.9$\pm$0.8  \\
    D2 \citep{d2}                & 74.3$\pm$0.6 & 64.0$\pm$1.4 & 62.6$\pm$1.0 & 43.2$\pm$2.2   & 73.7$\pm$1.0 & 65.7$\pm$0.9 & 62.8$\pm$1.0 & 41.3$\pm$1.2  \\
    \textbf{\sys}                & 73.1$\pm$0.3 & 67.9$\pm$0.6 & 62.0$\pm$0.5 & \textbf{45.0$\pm$1.2}   & 73.1$\pm$0.1 & 67.7$\pm$0.1 & 62.4$\pm$0.8 & \textbf{42.7$\pm$0.5} \\
\bottomrule
\end{tabular}
}
\vspace{3pt}
\resizebox{0.95\linewidth}{!}
{
\begin{tabular}{lcccgcccgcccgcccg}
    \toprule
    Corruption ($\rightarrow$) & \multicolumn{4}{c}{Frost} & \multicolumn{4}{c}{JPEG Compression} \\
    \cmidrule(lr){2-5} \cmidrule(lr){6-9}
    Pruning Rate ($\rightarrow$) & 50\% & 70\% & 80\% & 90\% & 50\% & 70\% & 80\% & 90\% \\
    \midrule
    Moderate \citep{moderateds}  & 70.4$\pm$0.4 & 63.5$\pm$0.4 & 55.1$\pm$0.8 & 36.1$\pm$0.3   & 70.8$\pm$0.2 & 63.8$\pm$0.2 & 55.9$\pm$0.9 & 33.9$\pm$0.6  \\
    D2 \citep{d2}                & 73.3$\pm$0.2 & 65.1$\pm$0.7 & 61.9$\pm$1.2 & 43.2$\pm$1.4   & 73.8$\pm$0.3 & 64.7$\pm$1.5 & 62.8$\pm$1.0 & 40.9$\pm$1.1  \\
    \textbf{\sys}                & 73.0$\pm$0.3 & 65.9$\pm$0.1 & 61.7$\pm$0.3 & \textbf{43.2$\pm$0.2}   & 73.0$\pm$0.1 & 67.9$\pm$0.5 & 62.8$\pm$0.2 & \textbf{43.5$\pm$0.5} \\
\bottomrule
\end{tabular}
}
\vspace{-10pt}
\label{tab:corruption_ablation}
\end{table}
\begin{table}[!tp]
\caption{\textbf{Transferability of \textit{Diverse (larger) Embeddings $\rightarrow$ Fixed (smaller) Target}}. Features across many architectures to train a ResNet-18 model on CIFAR-100. Models taken from \texttt{torchvision} are finetuned on ImageNet-1K (all except OpenCLIP). We also look at the transferability of features from OpenCLIP foundational models trained on the LAION-2b dataset \citep{laion}.}
\setlength{\tabcolsep}{3.1pt}
\centering
\resizebox{0.8\linewidth}{!}
{  
\begin{tabular}{lcccccc}
    \toprule
    Pruning Rate ($\rightarrow$) & \multicolumn{3}{c}{50\%} & \multicolumn{3}{c}{70\%}\\
    \cmidrule(lr){2-4} \cmidrule(lr){5-7}
    & Moderate & D2 & \textbf{\sys} & Moderate & D2 & \textbf{\sys} \\
    \midrule
    ResNet-18 \citep{resnet} & 70.9$\pm$0.4 & 73.0$\pm$0.8 & 73.6$\pm$0.2  & 62.9$\pm$0.2 & 67.9$\pm$0.3 & 68.1$\pm$0.2  \\
    ResNet-50 \citep{resnet} & 71.1$\pm$0.1 & 73.0$\pm$0.2 & 73.7$\pm$0.2  & 63.3$\pm$0.4 & 67.7$\pm$0.4 & 68.0$\pm$0.1  \\
    ResNet-101 \citep{resnet} & 70.0$\pm$0.5 & 73.2$\pm$0.2 & 73.5$\pm$0.2  & 62.8$\pm$0.4 & 66.9$\pm$0.9 & 68.0$\pm$0.3  \\
    EfficientNet-B0 \citep{efficientnet} & 71.6$\pm$0.3 & 73.2$\pm$0.2 & 72.9$\pm$0.2  & 62.9$\pm$0.1 & 67.5$\pm$0.4 & 67.5$\pm$0.2  \\
    EfficientNetV2-M \citep{efficientnetv2} & 69.6$\pm$0.3 & 72.8$\pm$0.7 & 73.4$\pm$0.2  & 61.0$\pm$0.2 & 67.1$\pm$0.8 & 67.1$\pm$0.1  \\
    SwinV2-T \citep{swinv2} & 70.5$\pm$0.1 & 73.6$\pm$0.1 & 73.6$\pm$0.2  & 60.4$\pm$0.6 & 66.9$\pm$1.0 & 67.6$\pm$0.1  \\
    SwinV2-B \citep{swinv2} & 69.9$\pm$0.3 & 73.5$\pm$0.4 & 73.6$\pm$0.2   & 61.9$\pm$0.4 & 67.1$\pm$0.4 & 68.1$\pm$0.2  \\
    ViT-L-16 \citep{vit} & 69.8$\pm$0.3 & 73.3$\pm$0.3 & 73.6$\pm$0.1  & 61.5$\pm$0.3 & 67.2$\pm$0.6 & 68.0$\pm$0.1  \\
    \multirow{2}{*}{
        \begin{tabular}
            [l]{@{}l@{}}OpenCLIP ViT-L-14 \\ \citep{clip, laion} \\
        \end{tabular}} & 
        \multirow{2}{*}{
        \begin{tabular}
            [c]{@{}c@{}}71.2$\pm$0.4
        \end{tabular}} & 
        \multirow{2}{*}{
        \begin{tabular}
            [c]{@{}c@{}}73.2$\pm$0.2
        \end{tabular}} & 
        \multirow{2}{*}{
        \begin{tabular}
            [c]{@{}c@{}}73.3$\pm$0.1
        \end{tabular}} & 
                \multirow{3}{*}{
        \begin{tabular}
            [c]{@{}c@{}}63.5$\pm$0.4
        \end{tabular}} & 
        \multirow{2}{*}{
        \begin{tabular}
            [c]{@{}c@{}}66.8$\pm$0.3
        \end{tabular}} & 
        \multirow{2}{*}{
        \begin{tabular}
            [c]{@{}c@{}}67.9$\pm$0.3
        \end{tabular}} \\ \\
    \multirow{2}{*}{
        \begin{tabular}
            [l]{@{}l@{}}OpenCLIP ViT-H-14 \\ \citep{clip, laion} \\
        \end{tabular}} & 
        \multirow{2}{*}{
        \begin{tabular}
            [c]{@{}c@{}}70.9$\pm$0.3
        \end{tabular}} & 
        \multirow{2}{*}{
        \begin{tabular}
            [c]{@{}c@{}}73.0$\pm$0.2
        \end{tabular}} & 
        \multirow{2}{*}{
        \begin{tabular}
            [c]{@{}c@{}}73.1$\pm$0.4
        \end{tabular}} &
        \multirow{2}{*}{
        \begin{tabular}
            [c]{@{}c@{}}62.4$\pm$0.3
        \end{tabular}} & 
        \multirow{2}{*}{
        \begin{tabular}
            [c]{@{}c@{}}66.6$\pm$1.3
        \end{tabular}} & 
        \multirow{2}{*}{
        \begin{tabular}
            [c]{@{}c@{}}67.7$\pm$0.2
        \end{tabular}} \\ \\
        \midrule
        \textbf{Overall Average} & 70.6$\pm$0.3 & 73.2$\pm$0.3 & \textbf{73.4$\pm$0.2} & 62.3$\pm$0.3 & 67.2$\pm$0.7 & \textbf{67.8$\pm$0.2} \\
\bottomrule
\end{tabular}
}

\vspace{2pt}
\centering
\resizebox{0.8\linewidth}{!}
{
\begin{tabular}{lcccccc}
    \toprule
    Pruning Rate ($\rightarrow$) & \multicolumn{3}{c}{80\%} & \multicolumn{3}{c}{90\%}\\
    \cmidrule(lr){2-4} \cmidrule(lr){5-7}
    & Moderate & D2 & \textbf{\sys} & Moderate & D2 & \textbf{\sys} \\
    \midrule
    ResNet-18 \citep{resnet} & 54.8$\pm$0.2 & 60.3$\pm$1.9 & 60.2$\pm$0.2  & 33.8$\pm$0.8 & 42.5$\pm$1.9 & 43.4$\pm$0.4  \\
    ResNet-50 \citep{resnet} & 55.9$\pm$0.6 & 60.8$\pm$1.0 & 61.3$\pm$0.7  & 31.8$\pm$1.6 & 44.5$\pm$1.7 & 47.4$\pm$0.5  \\
    ResNet-101 \citep{resnet} & 54.5$\pm$0.4 & 60.4$\pm$0.2 & 60.0$\pm$0.6  & 35.9$\pm$0.6 & 41.7$\pm$2.6 & 43.2$\pm$1.3  \\
    EfficientNet-B0 \citep{efficientnet} & 54.8$\pm$1.1 & 60.5$\pm$0.8 & 62.0$\pm$0.4 & 29.8$\pm$1.3 & 42.0$\pm$2.2 & 42.2$\pm$0.3  \\
    EfficientNetV2-M \citep{efficientnetv2} & 53.1$\pm$0.5 & 60.2$\pm$0.4 & 60.1$\pm$1.6  & 33.3$\pm$0.8 & 41.3$\pm$1.9 & 44.4$\pm$1.7 \\
    SwinV2-T \citep{swinv2} & 53.3$\pm$1.2 & 59.3$\pm$2.0 & 61.8$\pm$0.4  & 32.5$\pm$1.5 & 41.4$\pm$2.2 & 43.4$\pm$0.8  \\
    SwinV2-B \citep{swinv2} & 53.5$\pm$0.2 & 60.2$\pm$1.5 & 61.1$\pm$0.6  & 35.7$\pm$0.3 & 42.8$\pm$2.9 & 42.7$\pm$1.6  \\
    ViT-L-16 \citep{vit}& 53.9$\pm$0.9 & 59.1$\pm$1.1 & 61.1$\pm$0.2  & 30.6$\pm$1.2 & 40.9$\pm$2.6 & 44.1$\pm$1.3  \\
    \multirow{2}{*}{
        \begin{tabular}
            [l]{@{}l@{}}OpenCLIP ViT-L-14 \\ \citep{clip, laion} \\
        \end{tabular}} & 
        \multirow{2}{*}{
        \begin{tabular}
            [c]{@{}c@{}}53.9$\pm$0.8
        \end{tabular}} & 
        \multirow{3}{*}{
        \begin{tabular}
            [c]{@{}c@{}}60.4$\pm$1.0
        \end{tabular}} & 
        \multirow{2}{*}{
        \begin{tabular}
            [c]{@{}c@{}}61.7$\pm$0.8
        \end{tabular}} & 
        \multirow{2}{*}{
        \begin{tabular}
            [c]{@{}c@{}}35.0$\pm$1.5
        \end{tabular}} & 
        \multirow{2}{*}{
        \begin{tabular}
            [c]{@{}c@{}}40.7$\pm$2.9
        \end{tabular}} & 
        \multirow{2}{*}{
        \begin{tabular}
            [c]{@{}c@{}}42.5$\pm$1.6
        \end{tabular}} \\ \\
    \multirow{2}{*}{
        \begin{tabular}
            [l]{@{}l@{}}OpenCLIP ViT-H-14 \\ \citep{clip, laion} \\
        \end{tabular}} & 
        \multirow{2}{*}{
        \begin{tabular}
            [c]{@{}c@{}}54.0$\pm$0.8
        \end{tabular}} & 
        \multirow{2}{*}{
        \begin{tabular}
            [c]{@{}c@{}}61.1$\pm$0.4
        \end{tabular}} & 
        \multirow{2}{*}{
        \begin{tabular}
            [c]{@{}c@{}}61.8$\pm$0.8
        \end{tabular}} & 
        \multirow{2}{*}{
        \begin{tabular}
            [c]{@{}c@{}}35.4$\pm$1.4
        \end{tabular}} & 
        \multirow{2}{*}{
        \begin{tabular}
            [c]{@{}c@{}}42.3$\pm$2.1
        \end{tabular}} & 
        \multirow{2}{*}{
        \begin{tabular}
            [c]{@{}c@{}}42.9$\pm$1.0
        \end{tabular}} \\ \\
        \midrule
        \textbf{Overall Average} & 54.2$\pm$0.7 & 60.2$\pm$1.0 & \textbf{61.1$\pm$0.6} & 33.4$\pm$1.1 & 42.0$\pm$2.2 & \textbf{43.6$\pm$1.1} \\
\bottomrule
\end{tabular}
}
\label{tab:arch_ablation}
\end{table}

\begin{table}[!tbp]
    \caption{\textbf{(a)} Training and topological hyperparameters. \textbf{(b)} Dataset mislabel ratios. Similar to those in \citet{ccs}.}
    \centering
    \begin{subtable}[c]{0.65\textwidth}
        \centering

        \resizebox{\textwidth}{!}
        {
        \begin{tabular}{cccc}\toprule
        \multirow{2}{*}{
        \begin{tabular}[c]{@{}c@{}}Section\end{tabular}} 
        & \multirow{2}{*}{
        \begin{tabular}[c]{@{}c@{}}Hyperparameter\end{tabular}} 
        & \multirow{2}{*}{
            \begin{tabular}
                [c]{@{}c@{}}CIFAR-10\\CIFAR-100
            \end{tabular}} 
        & \multirow{2}{*}{
        \begin{tabular}[c]{@{}c@{}}ImageNet\end{tabular}} \\ \\
        
        \midrule
        \multirow{7}{*}{
        \begin{tabular}[c]{@{}c@{}}Training\\(\texttt{DeepCore})\end{tabular}}
        & Epochs & 200 & 60 \\
        & Batch Size & 256 & 128 \\
        & Optimizer & SGD & SGD \\
        & Momentum & 0.9 & 0.9 \\
        & Learning Rate & 1e-1 & 1e-1 \\
        & Weight Decay & 5e-4 & 5e-4 \\
        & Scheduler & CosineAnnealing & CosineAnnealing \\
        
        \midrule
        \multirow{4}{*}{
        \begin{tabular}[c]{@{}c@{}}Global Manifold\\Projection (\texttt{UMAP})\end{tabular}}
        & Number Neighbors & 15 & 15 \\
        & Minimum Distance & 0.1 & 0.1 \\
        & Metric & Cosine & Cosine \\
        & Dimensions & 2 & 2 \\
        
        \midrule
        \multirow{3}{*}{
        \begin{tabular}[c]{@{}c@{}}Kernel Density\\Estimation\\(\texttt{sklearn})\end{tabular}} & 
        \multirow{3}{*}{
        \begin{tabular}[c]{@{}c@{}}Bandwidth\end{tabular}} & 
        \multirow{3}{*}{
        \begin{tabular}[c]{@{}c@{}}0.4\end{tabular}} & 
        \multirow{3}{*}{
        \begin{tabular}[c]{@{}c@{}}0.4\end{tabular}} \\ \\ \\
        
        \midrule
        \multirow{5}{*}{
        \begin{tabular}[c]{@{}c@{}}Local Persistent\\Homology\\(\texttt{multipers})\end{tabular}}
        & Theta (Dens. Band.) & 0.4 & 0.4 \\
        & Function/Kernel & Gaussian & Gaussian \\
        & Complex & Weak-Delaunay & Weak-Delaunay \\
        & Homology Degree & 1 & 1 \\
        & Optimization Steps & 6 & 6 \\
        
        \midrule
        \multirow{2}{*}{
        \begin{tabular}[c]{@{}c@{}}Topology Score\end{tabular}}
        & Global Density ($\alpha$) & 0.5 & 0.5 \\
        & Local Persistence ($\beta$) & 0.5 & 0.5 \\

        \midrule
        \multirow{2}{*}{
        \begin{tabular}[c]{@{}c@{}}NLPS\\(KNN \texttt{sklearn})\end{tabular}} & 
        \multirow{2}{*}{
        \begin{tabular}[c]{@{}c@{}}Number Neighbors\end{tabular}} & 
        \multirow{2}{*}{
        \begin{tabular}[c]{@{}c@{}}20\end{tabular}} & 
        \multirow{2}{*}{
        \begin{tabular}[c]{@{}c@{}}20\end{tabular}} \\ \\
        
        \bottomrule
        \end{tabular}
        }
        
        \caption{}
        \label{tab:hyperparams_a}
    \end{subtable}
    \hspace{0.01\textwidth}
    \begin{subtable}[c]{0.3\textwidth}
        \centering

        \resizebox{\textwidth}{!}
        {
        \begin{tabular}{cccc}\toprule
        & \multicolumn{3}{c}{Mislabel Ratio (\%)} \\
        \cmidrule(lr){2-4}
        Pruning & C-10 & C-100 & ImageNet \\
        
        \midrule
        30\% & 0\% & 10\% & 0\% \\
        50\% & 0\% & 20\% & 10\% \\
        70\% & 10\% & 20\% & 20\% \\
        80\% & 10\% & 40\% & 20\% \\
        90\% & 30\% & 50\% & 30\% \\
        
        \bottomrule
        \end{tabular}
        }
        
        \caption{}
        \label{tab:hyperparams_b}
    \end{subtable}
    \label{table:hyperparameters}
\end{table}

\clearpage
\subsection{Additional Persistence Optimization Visualizations}
\label{sec:appendix_pers_optim_viz}
\begin{figure*}[!h]
    \centering
    \includegraphics[width=0.99\textwidth]{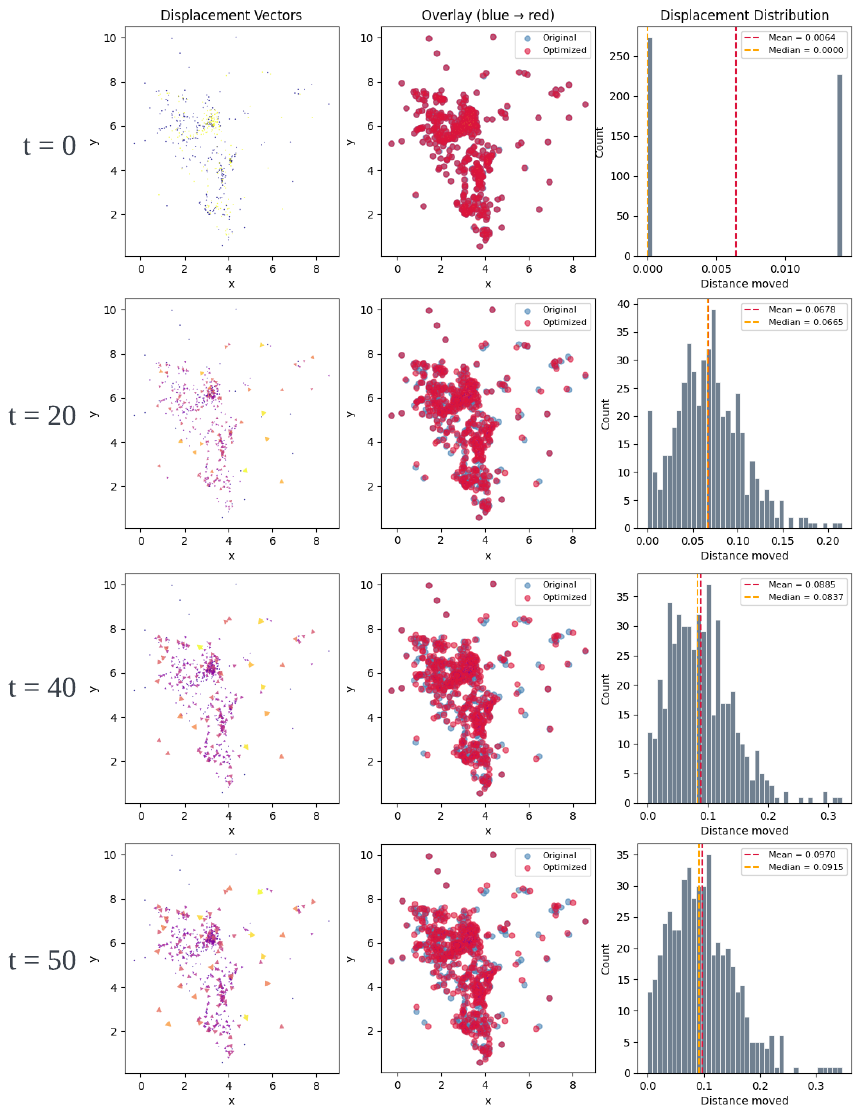}
    \caption{Persistence change across optimization (Label \#14 C-100; ResNet-50)}
\end{figure*}

\begin{figure*}[!tbp]
    \centering
    \includegraphics[width=0.99\textwidth]{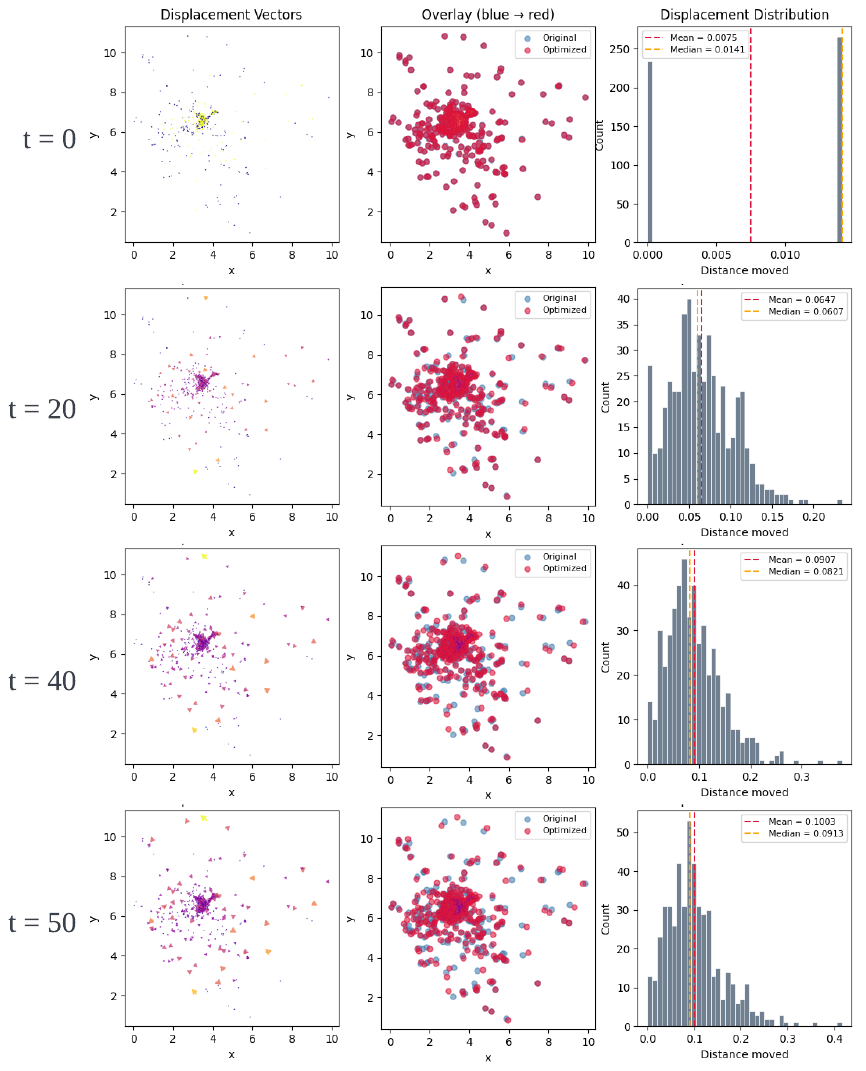}
    \caption{Persistence change across optimization (Label \#18 C-100; ResNet-50)}
\end{figure*}

\begin{figure*}[!tbp]
    \centering
    \includegraphics[width=0.99\textwidth]{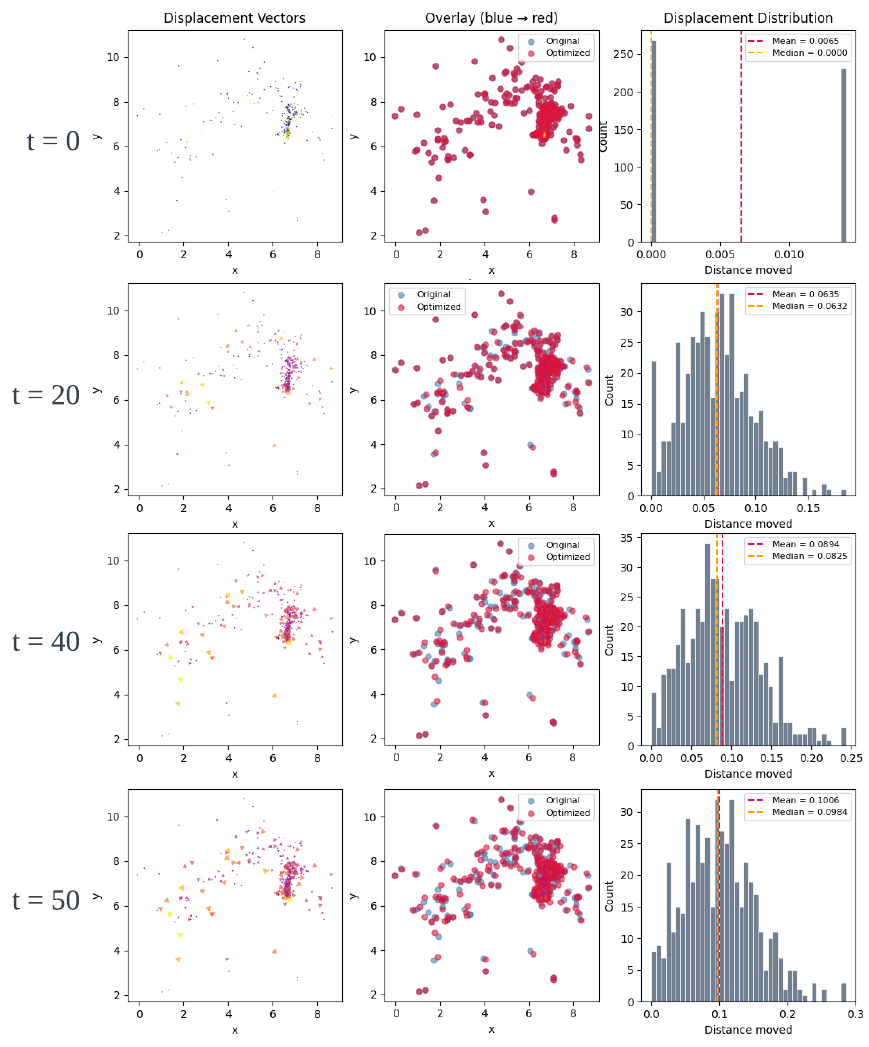}
    \caption{Persistence change across optimization (Label \#8 C-100; ResNet-50)}
\end{figure*}

\end{document}